\documentclass[11pt]{article}

\PassOptionsToPackage{hyperfootnotes=false}{hyperref}

\usepackage[preprint]{acl}

\usepackage{times}
\usepackage{latexsym}
\usepackage[T1]{fontenc}
\usepackage[utf8]{inputenc}
\usepackage{colortbl}
\usepackage{microtype}
\usepackage{makecell}
\usepackage{graphicx}
\usepackage{float}
\usepackage{placeins}
\usepackage{amsmath}
\usepackage{threeparttable}
\usepackage{amssymb}
\usepackage{amsthm}
\usepackage{enumitem}
\usepackage{tabularx}
\usepackage{siunitx}
\usepackage{xcolor}
\usepackage{booktabs}
\definecolor{highlightcolor}{RGB}{240, 245, 255}
\definecolor{scdorange}{RGB}{255, 127, 14}
\definecolor{dcdblue}{RGB}{31, 119, 180}

\usepackage{multirow}
\usepackage{subcaption}

\usepackage{pgfplots}
\pgfplotsset{compat=1.18}

\usepackage[ruled,vlined,linesnumbered,algo2e]{algorithm2e}

\newtheorem{theorem}{Theorem}[section]
\newtheorem{lemma}[theorem]{Lemma}

\newtheorem{definition}[theorem]{Definition}
\newtheorem{remark}[theorem]{Remark}

\usepackage{cleveref}

\title{Decoupled Contrastive Decoding via Expert-Aligned Drafting}

\author{
  Zhixuan Liu$^{1,2}$\thanks{\raggedright Corresponding author:\\
  \texttt{lzx993124494@sjtu.edu.cn}} \quad
  Zhichen Dong$^{1,2}$ \quad
  Yuanfu Wang$^{2}$ \quad
  Chao Yang$^{2}$ \\
  $^{1}$Shanghai Jiao Tong University \\
  $^{2}$Shanghai Artificial Intelligence Laboratory
}

\begin{document}

\maketitle

\begin{abstract}
Contrastive Decoding (CD) improves generation quality, but its amateur-model pass makes decoding expensive. Accelerating CD with speculative decoding raises a proposal-alignment question: should the contrastive signal shape the drafter, or should it remain only in verification? We study this question in the lightweight feature-level drafter regime. Two controlled diagnostics, matched Cross-$\alpha$ training and an Approximate Dual-Drafter decomposition, give the same diagnosis: contrastive-aware drafting does not consistently improve over expert-aligned drafting because the contrastive correction is usually weaker than drafter error, and reconstruction can amplify that error. We introduce \textbf{Decoupled Contrastive Decoding (DCD)}, which drafts with an expert-aligned lightweight proposer and applies the amateur only in unchanged CD verification. Standard speculative verification preserves the vanilla-CD output distribution. Across the main 8B settings, EAGLE3-based DCD achieves average greedy speedups of 1.65 to 1.95$\times$ over vanilla CD and reduces MMLU proposal-path latency by about 5 to 12$\times$ relative to amateur-coupled proposal paths.

\end{abstract}

\section{Introduction}
Large Language Models (LLMs) still hallucinate and make reasoning errors~\citep{DBLP:journals/tois/HuangYMZFWCPFQL25, DBLP:journals/corr/abs-2407-11511}. Contrastive Decoding (CD)~\citep{li2023contrastive,obrien2023contrastive} improves generation by correcting a strong \textit{expert} with a weaker \textit{amateur}. It also underlies lightweight tuning schemes such as Emulator Fine-Tuning~\citep{mitchell2024emulator} and Proxy Tuning~\citep{liu2024tuning}, but each token requires both models.

Speculative decoding can reduce this cost, but CD adds a proposal-alignment choice absent from single-model decoding: the serial proposal path can use the amateur signal, imitate the full contrastive distribution, or leave contrastive scoring to verification. We study three proposal routes: amateur-coupled proposals, contrastive-aware lightweight proposals, and expert-aligned proposals with contrastive verification. Figure~\ref{fig:overview} illustrates the system-level paths in the main comparison: vanilla CD, amateur-coupled SCD, and DCD with expert-aligned proposals plus contrastive verification.

\begin{figure}[t]
    \centering
    \includegraphics[width=\columnwidth, trim={0.5cm 0cm 14cm 0cm}, clip]{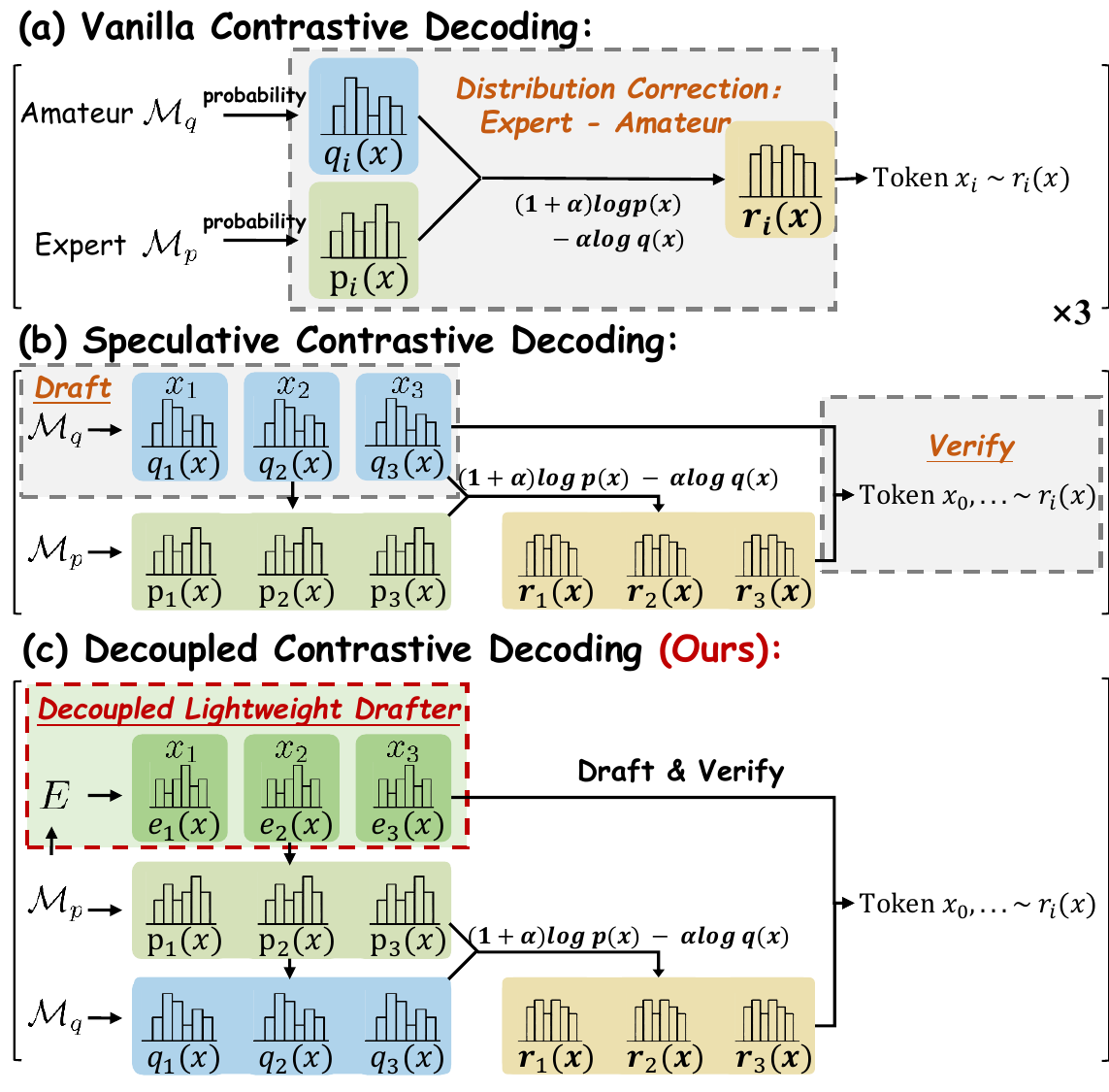}
    \vspace{-20pt}
    \caption{DCD keeps the CD target in verification while avoiding an amateur-coupled proposal route. The expert $\mathcal{M}_p$ and amateur $\mathcal{M}_q$ still define $\pi_{CD}$; the lightweight proposer $E$ changes only the serial draft path.}
    \label{fig:overview}
    \vspace{-10pt}
\end{figure}

We first diagnose proposal-side contrastive alignment under matched lightweight-drafter conditions. In the Cross-$\alpha$ diagnostic, a dual-input EAGLE variant receives concatenated expert and amateur hidden states; within that variant, varying train-$\alpha$ changes only the target distribution, while the architecture, data, and training recipe remain fixed. An Approximate Dual-Drafter separately approximates expert and amateur distributions and recombines them at inference. Both diagnostics show the same failure mode: the contrastive signal is usually weaker than the drafter error it must overcome, and reconstruction can amplify that error. Across 24 configurations, 81.1\% of positions have a contrastive signal below $1.0$, while 48.7\% have expert-side proposal $D_{KL}\geq2.0$; positive $\Delta$Top-1 values concentrate in high-signal, low-error cells. Neither contrastive-aware variant yields consistent accepted-length gains over expert-aligned drafting (Section~\ref{sec:why_not_cd}).

DCD validates the expert-aligned route at system level. It reuses a lightweight expert-aligned drafter for proposals and retains the amateur only in verification, so EAGLE-style drafting can accelerate CD without retraining for each amateur. SCD and CoS keep the amateur in the serial proposal loop.

Instantiated with EAGLE3~\citep{li2024eagle2,li2024eagle,li2025eagle3}, DCD$_{\textsc{eagle3}}$ achieves average greedy speedups of 1.65 to 1.95$\times$ over vanilla CD and reduces MMLU per-step proposal latency by about 5 to 12$\times$ relative to amateur-coupled proposal paths. DCD$_{\textsc{eagle3}}$ remains close to $2\times$ in a greedy-only 70B extension, and a DCD variant with a matched N-gram proposer also gives average speedups above vanilla CD. Because DCD changes only the proposal path, its lossless guarantee follows from standard speculative verification.

Our contributions are threefold: (1) we formulate proposal alignment as the central design choice in speculative contrastive decoding, separating amateur-coupled, contrastive-aware, and expert-aligned proposal routes; (2) we provide controlled diagnostics showing that contrastive-aware lightweight drafting does not reliably improve accepted length over expert-aligned drafting, because the contrastive signal is often smaller than proposal error and reconstruction can amplify that error; and (3) we instantiate this diagnosis as Decoupled Contrastive Decoding (DCD), a lossless CD accelerator that keeps contrastive scoring in verification and achieves deployment-level speedups with both EAGLE3 and matched N-gram proposers.

Our code is available at \url{https://github.com/chadlzx/dcd}.

\section{Methodology}
\label{sec:methodology}

This section formalizes proposal alignment and the DCD verification recipe.

\subsection{Preliminaries}
\label{sec:preliminaries}

\paragraph{Contrastive Decoding.}
Given history $h$, expert $\mathcal{M}_p$, and amateur $\mathcal{M}_q$, Contrastive Decoding~\citep{li2023contrastive,obrien2023contrastive} favors tokens high under the expert and low under the amateur:
\[
\begin{aligned}
\log \pi_{CD}(x\mid h) = {} & (1+\alpha) \log \pi_p(x\mid h) \\
& {} - \alpha \log \pi_q(x\mid h) - \log Z_{CD}(h),
\end{aligned}
\]
where $\alpha \geq 0$ controls contrastive strength. Equivalently,
\begin{equation}
\label{eq:contrastive_target}
\pi_{CD}(x\mid h) = \frac{1}{Z_{CD}(h)} \pi_p(x\mid h)^{1+\alpha} \pi_q(x\mid h)^{-\alpha},
\end{equation}
where $Z_{CD}(h)$ normalizes over the vocabulary.

Rewriting Eq.~\ref{eq:contrastive_target} as $\pi_{CD}(x\mid h) \propto \pi_p(x\mid h) \left(\frac{\pi_p(x\mid h)}{\pi_q(x\mid h)}\right)^{\alpha}$ separates the dominant expert term from the contrastive factor that a contrastive-aware lightweight drafter must also model.

The original CD plausibility constraint~\citep{li2023contrastive} is also compatible with DCD: masking Eq.~\ref{eq:contrastive_target} to $\mathcal{P}_{\beta}(h)=\{x:\pi_p(x\mid h)\geq \beta \max_{x'}\pi_p(x'\mid h)\}$ and renormalizing still yields a valid speculative-verification target. Proposals outside the mask simply receive zero target probability, so losslessness is unchanged after renormalization. Our experiments use the untruncated target to isolate proposal alignment without adding a threshold hyperparameter.

\paragraph{Speculative Decoding.}
Speculative Decoding (SD)~\citep{leviathan2023fast,chen2023accelerating} drafts $\gamma$ candidate tokens and verifies them in parallel with the target model. Its speed depends on draft-target alignment; we use $D_{KL}(\pi_{\text{draft}} \| \pi_{\text{target}})$ as a diagnostic proxy and report acceptance directly as mean accepted length $L$.

\subsection{Diagnosing Proposal-Side Alternatives in the Lightweight-Drafter Regime}
\label{sec:dilemma}

We compare three routes under lightweight drafting. They differ only in where the amateur signal enters the speculative decoding pipeline: directly in proposal generation, indirectly through a contrastive-aware lightweight proposer, or only during verification.

\paragraph{Route 1: Direct amateur drafting.}
This route uses $\pi_q$ for proposals even though the unnormalized CD score assigns coefficient $-\alpha$ to $\log \pi_q(x\mid h)$. It therefore has both a distributional mismatch and a systems cost: serial proposal cost remains tied to the amateur, roughly $\gamma \cdot t_q$ before method-specific bonus or delayed-drafting terms.

\paragraph{Route 2: Contrastive-aware lightweight drafting.}
A lightweight drafter must approximate $\pi_p$ and the extra contrastive factor $\left(\frac{\pi_p}{\pi_q}\right)^{\alpha}$. Training on the contrastive target increases modeling burden without removing baseline proposal error, and inference-time reconstruction adds another error source. Section~\ref{sec:why_not_cd} shows that this added signal is usually weaker than the proposal error it must overcome.

\paragraph{Route 3: Expert-aligned proposals with contrastive verification.}
\label{sec:dcd}

\textbf{Decoupled Contrastive Decoding (DCD)} proposes $\tilde{y}_{1:\gamma}$ without using $\mathcal{M}_q$ in the serial proposal path. The expert and amateur are coupled only during verification, where proposals are evaluated against the unchanged $\pi_{CD}$.
Algorithm~\ref{alg:dcd_round} summarizes one DCD round. Appendix~\ref{sec:proof_lossless} proves that replacing the proposer does not change the output distribution as long as speculative verification targets the normalized CD distribution $\pi_{CD}$; the proposer only changes acceptance rate and efficiency. We instantiate the proposer $E$ with \textbf{EAGLE3}~\citep{li2025eagle3} for the main deployment experiments because it gives the strongest average operating point among our tested proposers. This EAGLE3 instantiation is expert-aligned; the same DCD verification rule can also wrap other amateur-independent proposers, such as the N-gram proposer.

\begin{algorithm2e}[t]
\DontPrintSemicolon
\SetAlgoNlRelativeSize{0}
\caption{Decoupled Contrastive Decoding (One Round)}
\label{alg:dcd_round}
\small
\KwData{$\mathbf{y} = [\cdots, y_0]$, $\mathbf{h}_{draft}^{(0)}$, $\mathcal{M}_p$, $\mathcal{M}_q$, $E$, $\gamma$, $\alpha$}
\KwResult{updated prefix after one DCD round}
$\tilde{y}_0 \leftarrow y_0$\;
\tcp*[l]{Draft generation}
\For{$i \leftarrow 1$ \KwTo $\gamma$}{
    $\pi_e^{(i)}, \mathbf{h}_{draft}^{(i)} \leftarrow E(\mathbf{h}_{draft}^{(i-1)}, \tilde{y}_{i-1})$\;
    $\tilde{y}_i \sim \pi_e^{(i)}$\;
}
$\tilde{\mathbf{y}} \leftarrow [\tilde{y}_1, \ldots, \tilde{y}_\gamma]$\;
\tcp*[l]{Parallel model evaluation}
$[\pi_p^{(1)}, \ldots, \pi_p^{(\gamma+1)}] \leftarrow \mathcal{M}_p(\mathbf{y} \oplus \tilde{\mathbf{y}})$\;
$[\pi_q^{(1)}, \ldots, \pi_q^{(\gamma+1)}] \leftarrow \mathcal{M}_q(\mathbf{y} \oplus \tilde{\mathbf{y}})$\;
\tcp*[l]{Speculative accept or fallback}
\For{$i \leftarrow 1$ \KwTo $\gamma$}{
    compute $\pi_{CD}^{(i)}$ from Eq.~\eqref{eq:contrastive_target} using $\pi_p^{(i)}$ and $\pi_q^{(i)}$, and draw $r_i \sim \mathcal{U}(0, 1)$\;
}
$k \leftarrow \min\left(\left\{i \mid r_i > \frac{\pi_{CD}^{(i)}(\tilde{y}_i)}{\pi_e^{(i)}(\tilde{y}_i)}\right\} \cup \{\gamma + 1\}\right)$\;
\eIf{$k \leq \gamma$}{
    accept $\tilde{y}_{1:k-1}$\;
    $P_r(x) \leftarrow \mathrm{Normalize}(\max(0, \pi_{CD}^{(k)}(x) - \pi_e^{(k)}(x)))$\;
    sample $y_k \sim P_r$ and return $\mathbf{y} \oplus \tilde{y}_{1:k-1} \oplus y_k$\;
}{
    compute $\pi_{CD}^{(\gamma+1)}$ from Eq.~\eqref{eq:contrastive_target} using $\pi_p^{(\gamma+1)}$ and $\pi_q^{(\gamma+1)}$\;
    sample $y_{\gamma+1} \sim \pi_{CD}^{(\gamma+1)}$ and return $\mathbf{y} \oplus \tilde{y}_{1:\gamma} \oplus y_{\gamma+1}$\;
}
\end{algorithm2e}

DCD separates distributional correctness from proposal efficiency. Because verification still targets $\pi_{CD}$, changing the proposer does not change the output distribution of vanilla CD under standard speculative sampling. The change falls on the serial proposal path: DCD replaces the amateur cost of roughly $\gamma \cdot t_q$ with the cost of the chosen proposer, namely $\gamma \cdot t_e$ for EAGLE3 with $t_e \ll t_q$, or near-zero measured proposal latency for the N-gram variant.

\subsection{Alpha-Robustness Intuition}
\label{sec:theory}
Effective Draft Alignment means that the proposer places more mass than the amateur on tokens where the expert has a larger log-likelihood advantage. Under this condition, the amateur proposal's draft-target KL grows faster with $\alpha$ than an expert-aligned drafter's KL. This slope comparison suggests that expert-aligned DCD can lose accepted length more slowly even when SCD or CoS start with a higher acceptance baseline. Appendix~\ref{sec:appendix_theory} proves the KL-slope inequality and connects it to the observed robustness trend.

\section{Experiments}
\label{sec:experiments}
The experiments follow the design question from Section~\ref{sec:dilemma}. We first ask whether lightweight drafters benefit from absorbing the contrastive signal, then evaluate the expert-aligned DCD route at deployment scale and explain its speed source.

\paragraph{Shared setting.} Unless noted otherwise, we use 200 samples per dataset (or the full dataset if smaller), greedy decoding ($T{=}0$), draft length $\gamma{=}5$, and chain-style, greedy drafting for EAGLE3 proposers. Main deployment tables report three-run averages; diagnostics and ablation studies use single-run traces. Runtime flags are listed in Appendix~\ref{sec:implementation_details}.

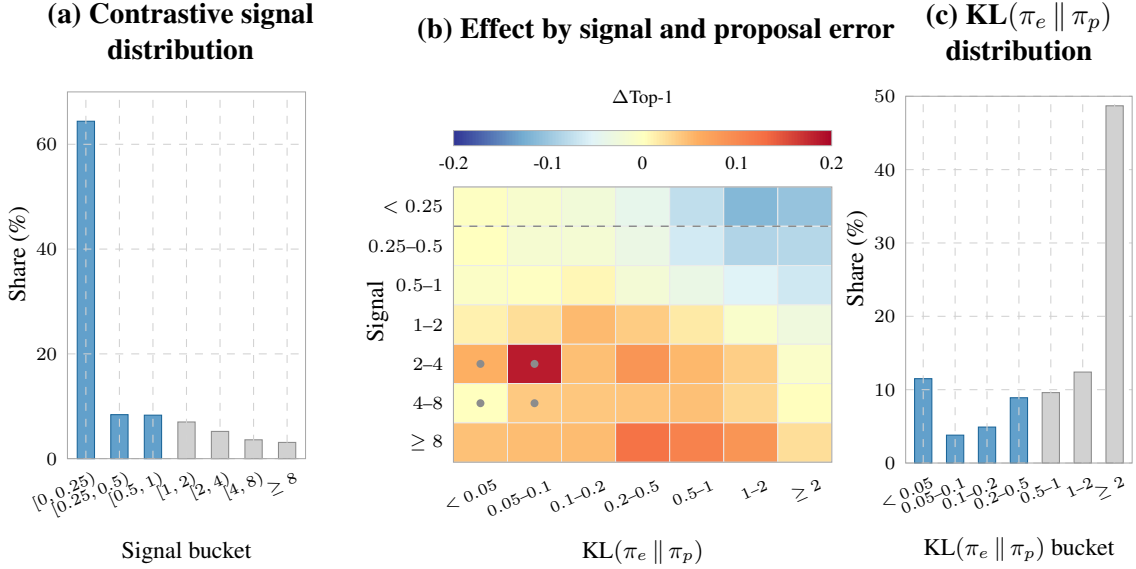
\begin{figure*}[!t]
\centering
\hspace*{-1.3cm}
\begin{subfigure}[t]{0.24\textwidth}
\centering
\parbox[t][2.7em][c]{\linewidth}{\centering
\makebox[0pt][c]{\hspace*{2.5em}\textbf{(a) Contrastive signal}}\\
\makebox[0pt][c]{\hspace*{2.5em}\textbf{distribution}}}\par\vspace{0.25em}
\begin{tikzpicture}
\begin{axis}[
    ybar,
    bar width=6.5pt,
    bar shift=0pt,
    width=0.82\linewidth,
    height=4.85cm,
    xmin=-0.55, xmax=6.55,
    ymin=0, ymax=70,
    xtick={0,1,2,3,4,5,6},
    xticklabels={{$[0,0.25)$},{$[0.25,0.5)$},{$[0.5,1)$},{$[1,2)$},{$[2,4)$},{$[4,8)$},{$\geq 8$}},
    ytick={0,20,40,60},
    xlabel={Signal bucket},
    ylabel={Share (\%)},
    xlabel style={font=\small, yshift=0pt},
    ylabel style={font=\small, xshift=8pt, yshift=-5pt},
    xticklabel style={font=\tiny, rotate=24, anchor=east, xshift=6pt, yshift=-4pt},
    yticklabel style={font=\scriptsize},
    tick style={draw=none},
    axis line style={draw=black!35, line width=0.35pt},
    grid=major,
    grid style={dashed, gray!35},
    scale only axis,
    axis on top,
    clip=false,
]
\addplot[draw=dcdblue!85!black, fill=dcdblue!70] coordinates { (0,64.4) (1,8.4) (2,8.3) };
\addplot[draw=black!45, fill=black!18] coordinates { (3,7.0) (4,5.2) (5,3.6) (6,3.1) };
\end{axis}
\end{tikzpicture}
\end{subfigure}
\hspace{0.3cm}
\begin{subfigure}[t]{0.46\textwidth}
\centering
\parbox[t][2.7em][c]{\linewidth}{\centering
\makebox[0pt][c]{\hspace*{3.5em}\textbf{(b) Effect by signal and proposal error}}}\par\vspace{0.25em}
\begin{tikzpicture}
\begin{axis}[
    width=0.68\linewidth,
    height=3.64cm,
    xmin=-0.5, xmax=6.5,
    ymin=-0.5, ymax=6.5,
    y dir=reverse,
    xtick={0,1,2,3,4,5,6},
    ytick={0,1,2,3,4,5,6},
    xticklabels={{$<0.05$},{$0.05$--$0.1$},{$0.1$--$0.2$},{$0.2$--$0.5$},{$0.5$--$1$},{$1$--$2$},{$\geq 2$}},
    yticklabels={{$<0.25$},{$0.25$--$0.5$},{$0.5$--$1$},{$1$--$2$},{$2$--$4$},{$4$--$8$},{$\geq 8$}},
    xticklabel style={font=\tiny, rotate=24, anchor=east, xshift=6pt, yshift=-10pt},
    yticklabel style={font=\scriptsize},
    tick style={draw=none},
    axis line style={draw=black!35, line width=0.35pt},
    xlabel={KL$(\pi_e \,\|\, \pi_p)$},
    ylabel={Signal},
    xlabel style={font=\small, yshift=0pt},
    ylabel style={font=\small, xshift=5pt, yshift=-15pt},
    enlargelimits=false,
    scale only axis,
    axis on top,
    point meta min=-0.20,
    point meta max=0.20,
    colormap={paperdiv}{
        rgb255(0cm)=(49,54,149);
        rgb255(0.18cm)=(116,173,209);
        rgb255(0.36cm)=(224,243,248);
        rgb255(0.50cm)=(255,255,191);
        rgb255(0.64cm)=(253,174,97);
        rgb255(0.82cm)=(244,109,67);
        rgb255(1cm)=(165,0,38)
    },
    colorbar horizontal,
    colorbar style={
        at={(0.5,1.15)},
        anchor=south,
        title={$\Delta$Top-1},
        title style={font=\scriptsize, yshift=-1pt},
        xtick={-0.20,-0.10,0.00,0.10,0.20},
        xticklabels={{-0.2},{-0.1},{0},{0.1},{0.2}},
        ytick=\empty,
        tick label style={font=\scriptsize},
        width=5cm,
        height=0.2cm,
        major tick length=1.3pt,
        axis line style={draw=black!30, line width=0.3pt}
    }
]
\addplot[
    matrix plot*,
    mesh/cols=7,
    point meta=explicit,
    draw=black!8,
    line width=0.12pt,
] coordinates {
    (0,0) [-0.0032] (1,0) [-0.0140] (2,0) [-0.0247] (3,0) [-0.0432] (4,0) [-0.0774] (5,0) [-0.1171] (6,0) [-0.1055] (0,1) [-0.0000] (1,1) [-0.0174] (2,1) [-0.0198] (3,1) [-0.0367] (4,1) [-0.0655] (5,1) [-0.0875] (6,1) [-0.0849] (0,2) [-0.0070] (1,2) [-0.0025] (2,2) [0.0052] (3,2) [-0.0220] (4,2) [-0.0378] (5,2) [-0.0553] (6,2) [-0.0672] (0,3) [0.0095] (1,3) [0.0230] (2,3) [0.0478] (3,3) [0.0355] (4,3) [0.0148] (5,3) [-0.0114] (6,3) [-0.0296] (0,4) [0.0556] (1,4) [0.1819] (2,4) [0.0445] (3,4) [0.0842] (4,4) [0.0492] (5,4) [0.0338] (6,4) [-0.0094] (0,5) [0.0000] (1,5) [0.0372] (2,5) [0.0385] (3,5) [0.0402] (4,5) [0.0434] (5,5) [0.0276] (6,5) [0.0006] (0,6) [0.0431] (1,6) [0.0458] (2,6) [0.0464] (3,6) [0.1213] (4,6) [0.1045] (5,6) [0.0842] (6,6) [0.0226]
};

\fill[black!45] (axis cs:0,4) circle (1.5pt);
\fill[black!45] (axis cs:1,4) circle (1.5pt);
\fill[black!45] (axis cs:0,5) circle (1.5pt);
\fill[black!45] (axis cs:1,5) circle (1.5pt);
\draw[black!45, dashed, line width=0.55pt] (axis cs:-0.5,0.5) -- (axis cs:6.5,0.5);
\end{axis}
\end{tikzpicture}
\end{subfigure}
\hspace{-0.8cm}
\begin{subfigure}[t]{0.24\textwidth}
\centering
\parbox[t][2.7em][c]{\linewidth}{\centering
\makebox[0pt][c]{\hspace*{2.5em}\textbf{(c) KL$(\pi_e \,\|\, \pi_p)$}}\\
\makebox[0pt][c]{\hspace*{2.5em}\textbf{distribution}}}\par\vspace{0.25em}
\begin{tikzpicture}
\begin{axis}[
    ybar,
    bar width=6.5pt,
    bar shift=0pt,
    width=0.78\linewidth,
    height=4.85cm,
    xmin=-0.55, xmax=6.55,
    ymin=0, ymax=50,
    xtick={0,1,2,3,4,5,6},
    xticklabels={{$<0.05$},{$0.05$--$0.1$},{$0.1$--$0.2$},{$0.2$--$0.5$},{$0.5$--$1$},{$1$--$2$},{$\geq 2$}},
    ytick={0,10,20,30,40,50},
    xlabel={KL$(\pi_e \,\|\, \pi_p)$ bucket},
    ylabel={Share (\%)},
    xlabel style={font=\small, yshift=0pt},
    ylabel style={font=\small, xshift=8pt, yshift=-5pt},
    xticklabel style={font=\tiny, rotate=24, anchor=east, xshift=6pt, yshift=-4pt},
    yticklabel style={font=\scriptsize},
    tick style={draw=none},
    axis line style={draw=black!35, line width=0.35pt},
    grid=major,
    grid style={dashed, gray!35},
    scale only axis,
    axis on top,
    clip=false,
]
\addplot[draw=dcdblue!85!black, fill=dcdblue!70] coordinates { (0,11.5) (1,3.8) (2,4.9) (3,8.9) };
\addplot[draw=black!45, fill=black!18] coordinates { (4,9.6) (5,12.4) (6,48.7) };
\end{axis}
\end{tikzpicture}
\end{subfigure}

\caption{Contrastive-aware lightweight drafting has only a small favorable region to exploit. Most positions have weak signal (81.1\% below 1.0), gains concentrate in high-signal/low-error cells, and 48.7\% have KL$(\pi_e\|\pi_p)\geq2.0$. Black dots mark cells with fewer than 100 positions. Signal is $|\log(\pi_p(x^\star\mid h)/\pi_q(x^\star\mid h))|$ at $x^\star=\arg\max_x\pi_{CD}(x\mid h)$; $\Delta$Top-1 compares contrastive-aligned and expert-aligned Top-1 agreement with $x^\star$.}
\label{fig:signal_kl_ep_heatmap}
\vspace{-4pt}
\end{figure*}

\begin{figure}[t]
\centering
\begin{subfigure}[t]{\columnwidth}
\centering
\textbf{(a) Effect by signal and proposal error}\par\vspace{0.75em}
\begin{tikzpicture}
\begin{axis}[
    width=0.60\linewidth,
    height=3.35cm,
    xmin=-0.5, xmax=7.5,
    ymin=-0.5, ymax=6.5,
    y dir=reverse,
    xtick={0,1,2,3,4,5,6,7},
    ytick={0,1,2,3,4,5,6},
    xticklabels={{$<0.01$},{$0.01$--$0.05$},{$0.05$--$0.1$},{$0.1$--$0.2$},{$0.2$--$0.5$},{$0.5$--$1$},{$1$--$2$},{$\geq 2$}},
    yticklabels={{$<0.25$},{$0.25$--$0.5$},{$0.5$--$1$},{$1$--$2$},{$2$--$4$},{$4$--$8$},{$\geq 8$}},
    xticklabel style={font=\tiny, rotate=24, anchor=east, xshift=6pt, yshift=-10pt},
    yticklabel style={font=\scriptsize},
    tick style={draw=none},
    axis line style={draw=black!35, line width=0.35pt},
    xlabel={KL$(\pi_r \,\|\, \pi_p)$},
    ylabel={Signal},
    xlabel style={font=\small, yshift=0pt},
    ylabel style={font=\small, xshift=-4pt, yshift=-8pt},
    enlargelimits=false,
    scale only axis,
    axis on top,
    point meta min=-0.40,
    point meta max=0.40,
    colormap={paperdiv}{
        rgb255(0cm)=(49,54,149);
        rgb255(0.18cm)=(116,173,209);
        rgb255(0.36cm)=(224,243,248);
        rgb255(0.50cm)=(255,255,191);
        rgb255(0.64cm)=(253,174,97);
        rgb255(0.82cm)=(244,109,67);
        rgb255(1cm)=(165,0,38)
    },
    colorbar horizontal,
    colorbar style={
        at={(0.5,1.12)},
        anchor=south,
        title={$\Delta$Top-1},
        title style={font=\scriptsize, yshift=-1pt},
        xtick={-0.40,-0.20,0.00,0.20,0.40},
        xticklabels={{-0.4},{-0.2},{0},{0.2},{0.4}},
        ytick=\empty,
        tick label style={font=\scriptsize},
        width=3.8cm,
        height=0.18cm,
        major tick length=1.2pt,
        axis line style={draw=black!30, line width=0.3pt}
    }
]
\addplot[
    matrix plot*,
    mesh/cols=8,
    point meta=explicit,
    draw=black!8,
    line width=0.12pt,
] coordinates {
    (0,0) [-0.0002] (1,0) [-0.0082] (2,0) [-0.0241] (3,0) [-0.0487] (4,0) [-0.0864] (5,0) [-0.1562] (6,0) [-0.2600] (7,0) [-0.3607] (0,1) [0.0044] (1,1) [-0.0023] (2,1) [-0.0078] (3,1) [-0.0201] (4,1) [-0.0584] (5,1) [-0.1087] (6,1) [-0.1403] (7,1) [-0.2276] (0,2) [0.0412] (1,2) [0.0320] (2,2) [0.0159] (3,2) [0.0015] (4,2) [-0.0272] (5,2) [-0.0599] (6,2) [-0.0796] (7,2) [-0.1256] (0,3) [0.1031] (1,3) [0.0960] (2,3) [0.0832] (3,3) [0.0704] (4,3) [0.0345] (5,3) [-0.0164] (6,3) [-0.0261] (7,3) [-0.0335] (0,4) [0.1695] (1,4) [0.2093] (2,4) [0.2230] (3,4) [0.1769] (4,4) [0.1109] (5,4) [0.0458] (6,4) [0.0254] (7,4) [0.0086] (0,5) [0.1818] (1,5) [0.2576] (2,5) [0.2738] (3,5) [0.2051] (4,5) [0.1858] (5,5) [0.1131] (6,5) [0.0677] (7,5) [0.0287] (0,6) [0.0000] (1,6) [0.1333] (2,6) [0.1304] (3,6) [0.2963] (4,6) [0.2444] (5,6) [0.1250] (6,6) [0.1236] (7,6) [0.0714]
};
\fill[black!45] (axis cs:0,4) circle (1.5pt);
\fill[black!45] (axis cs:0,5) circle (1.5pt);
\fill[black!45] (axis cs:1,5) circle (1.5pt);
\fill[black!45] (axis cs:2,5) circle (1.5pt);
\fill[black!45] (axis cs:0,6) circle (1.5pt);
\fill[black!45] (axis cs:1,6) circle (1.5pt);
\fill[black!45] (axis cs:2,6) circle (1.5pt);
\fill[black!45] (axis cs:3,6) circle (1.5pt);
\fill[black!45] (axis cs:4,6) circle (1.5pt);
\fill[black!45] (axis cs:6,6) circle (1.5pt);
\fill[black!45] (axis cs:7,6) circle (1.5pt);
\draw[black!45, dashed, line width=0.55pt] (axis cs:-0.5,0.5) -- (axis cs:7.5,0.5);
\end{axis}
\end{tikzpicture}
\end{subfigure}
\vspace{0.2em}
\begin{subfigure}[t]{\columnwidth}
\centering
\textbf{(b) Proposer-side KL marginal}\par\vspace{0.2em}
\begin{tikzpicture}
\begin{axis}[
    ybar,
    bar width=5.2pt,
    bar shift=0pt,
    width=0.60\linewidth,
    height=2.35cm,
    xmin=-0.55, xmax=7.55,
    ymin=0, ymax=50,
    xtick={0,1,2,3,4,5,6,7},
    xticklabels={{$<0.01$},{$0.01$--$0.05$},{$0.05$--$0.1$},{$0.1$--$0.2$},{$0.2$--$0.5$},{$0.5$--$1$},{$1$--$2$},{$\geq 2$}},
    ytick={0,10,20,30,40,50},
    xlabel={KL$(\pi_r \,\|\, \pi_p)$ bucket},
    ylabel={Share (\%)},
    xlabel style={font=\small, yshift=0pt},
    ylabel style={font=\small, xshift=0pt, yshift=0pt},
    xticklabel style={font=\tiny, rotate=24, anchor=east, xshift=6pt, yshift=-4pt},
    yticklabel style={font=\scriptsize},
    tick style={draw=none},
    axis line style={draw=black!35, line width=0.35pt},
    grid=major,
    grid style={dashed, gray!35},
    scale only axis,
    axis on top,
    clip=false,
]
\addplot[draw=dcdblue!85!black, fill=dcdblue!70] coordinates { (0,47.9) (1,12.2) (2,7.3) (3,9.0) };
\addplot[draw=black!45, fill=black!18] coordinates { (4,12.4) (5,6.9) (6,3.3) (7,1.1) };
\end{axis}
\end{tikzpicture}
\end{subfigure}

\caption{Even a stronger same-family proposer leaves gains confined to high-signal, low-KL regions, with negative overall effect. $\Delta$Top-1 is the CD-aware reconstruction's Top-1 agreement gain over the unmodified 3B proposer.}
\label{fig:stronger_proposer_signal_kl_summary}
\vspace{-6pt}
\end{figure}

\subsection{Why Not Contrastive-Aware Drafting?}
\label{sec:why_not_cd}

This subsection examines Route~2 from Section~\ref{sec:dilemma}: moving the contrastive signal into lightweight proposal generation before verification. We first characterize where such alignment could help, then test whether lightweight drafters can exploit that region under matched controls, and finally examine whether the negative result persists with a substantially stronger proposer.

\paragraph{Signal--Proposal-Error Diagnostic.}
Before testing contrastive-aware drafting, we first ask whether the contrastive signal occupies a sufficiently favorable region relative to baseline proposal error. At each position, let $x^\star=\arg\max_x\pi_{CD}(x\mid h)$. We measure the contrastive signal magnitude as $\left|\log\frac{\pi_p(x^\star\mid h)}{\pi_q(x^\star\mid h)}\right|$, baseline expert-side proposal error as $D_{KL}(\pi_e\|\pi_p)$ for the expert-aligned drafter $\pi_e$, and $\Delta$Top-1 as the difference between contrastive-aligned and expert-aligned Top-1 agreement with $x^\star$. Figure~\ref{fig:signal_kl_ep_heatmap} aggregates these quantities, count-weighted, over 24 configurations: \textsc{Llama-3}, \textsc{Llama-EFT}, and \textsc{Qwen3}, each evaluated on GSM8K and MMLU at $\alpha\in\{0.1,0.3,0.5,1.0\}$.

Across all pooled positions in this count-weighted diagnostic, 81.1\% have signal below $1.0$ and 48.7\% have $D_{KL}(\pi_e\|\pi_p)\geq2.0$; positive $\Delta$Top-1 values concentrate in high-signal, low-error cells.
Thus, the tested regime leaves only a relatively small favorable region for contrastive-aware proposal alignment, while substantial baseline proposal error is widespread. We use $\Delta$Top-1 only as a position-level diagnostic; whether these local gains translate into accepted-length improvements is tested directly in the matched experiments below.

\paragraph{Matched Lightweight Drafting Experiments.}
We next test whether a lightweight drafter can exploit this signal in a controlled setting. Both variants use the matched \textsc{Llama-3}/\textsc{Qwen3} pairs on GSM8K/MMLU, with the dual-input EAGLE architecture, ShareGPT training data, and training recipe fixed. The direct variant learns the contrastive target itself; the decomposed variant separately learns expert- and amateur-aligned distributions and recombines them at inference time.

\paragraph{Direct Contrastive-Aligned Drafter.}
The most direct realization trains a dual-input EAGLE drafter to model the contrastive residual in Eq.~\ref{eq:contrastive_target}:
\[
\begin{aligned}
\frac{\pi_{CD}(x\mid h)}{\pi_p(x\mid h)}
&\propto
\left(\frac{\pi_p(x\mid h)}{\pi_q(x\mid h)}\right)^\alpha
= e^{\alpha s(x\mid h)}, \\
s(x\mid h) &= \log\frac{\pi_p(x\mid h)}{\pi_q(x\mid h)}.
\end{aligned}
\]
We vary both train-$\alpha$ and infer-$\alpha$ over $\{0.0,0.1,0.3,0.5\}$ while holding the remaining setup fixed; train-$\alpha=0$ recovers the expert-aligned training target. Table~\ref{tab:cross_alpha} shows no consistent gain from non-zero train-$\alpha$: only Llama-3/GSM8K has a small gain averaged over the four infer-$\alpha$ values ($+0.015$ in $\overline{L}$, or $0.77\%$), whereas train-$\alpha=0$ is best in the other three settings. Thus, in this matched control, directly learning the contrastive residual does not reliably translate into accepted-length gains. The full matrix is in Appendix~\ref{sec:appendix_cross_alpha}.

\begin{table}[t]
\centering
\scriptsize
\setlength{\tabcolsep}{4pt}
\renewcommand{\arraystretch}{1.05}
\begin{tabular*}{\columnwidth}{@{\extracolsep{\fill}}lcccc}
\toprule
\textbf{Model/Data} & \textbf{\shortstack{Best\\train-$\alpha$}} & \textbf{\shortstack{Avg $\Delta \overline{L}$\\vs $\alpha{=}0$}} & \textbf{\shortstack{Avg \% gain\\vs $\alpha{=}0$}} & \textbf{\shortstack{$\alpha{=}0$\\wins}} \\
\midrule
Llama-3/GSM8K & 0.1 & +0.015 & +0.77\% & 0/4 \\
Llama-3/MMLU & \textbf{0.0} & +0.000 & +0.00\% & 3/4 \\
Qwen3/GSM8K & \textbf{0.0} & +0.000 & +0.00\% & 3/4 \\
Qwen3/MMLU & \textbf{0.0} & +0.000 & +0.00\% & 4/4 \\
\bottomrule
\end{tabular*}
\caption{Cross-$\alpha$ control: contrastive-target training does not consistently improve $\overline{L}$. Best train-$\alpha$ maximizes average $\overline{L}$; only Llama-3/GSM8K shows a small non-zero gain. Full matrix in Appendix~\ref{sec:appendix_cross_alpha}.}
\label{tab:cross_alpha}
\vspace{-4pt}
\end{table}

\paragraph{Decomposed Dual-Drafter.}
Directly modeling $\pi_{CD}$ may be unnecessarily difficult, so we also test a more favorable decomposition that separates expert- and amateur-side approximation before recombining the two analytically at inference time,
\[
\hat{\pi}_{CD}(x\mid h) \propto
\pi_e(x\mid h)^{1+\alpha}\pi_f(x\mid h)^{-\alpha},
\]
where $\pi_e$ and $\pi_f$ are the expert- and amateur-aligned draft distributions. This separates expert- and amateur-side approximation, but it amplifies expert-side log-approximation error by $1+\alpha$; Appendix~\ref{sec:cd_aligned_analysis} provides the reconstruction derivation.

Even under this more favorable formulation, Approximate Dual-Drafter underperforms expert-aligned drafting at every reported $\alpha$ in mean accepted length $L$, draft-to-target $D_{KL}(\cdot\|\pi_{CD})$, and Top-1 Match, with gaps widening as $\alpha$ increases (Appendix Table~\ref{tab:dual_drafting}). At the position level, $\hat{\pi}_{CD}$ is closer to $\pi_{CD}$ than $\pi_e$ at only 31.9--39.7\% of positions.

\paragraph{Stress Test with a Stronger Full-LM Proposer.}
The preceding negative results could reflect limitations of the lightweight proposer rather than the contrastive objective itself. We therefore repeat the position-level diagnostic with \textsc{Llama-3.2-3B-Instruct} as a stronger Llama-family proposer, while fixing the \textsc{Llama-3.1-8B-Instruct}/\textsc{Llama-3.2-1B-Instruct} expert--amateur pair. On shared response histories, we apply contrastive correction using true amateur logits. This is an optimistic offline check on GSM8K and MMLU at $\alpha\in\{0.1,0.5\}$, rather than an online rollout evaluation. Figure~\ref{fig:stronger_proposer_signal_kl_summary} shows that the stronger proposer enlarges the favorable high-signal, low-error pocket, but the setting-level mean $\Delta$Top-1 remains negative in all four settings (see also Appendix~\ref{sec:appendix_stronger_proposer}). The stronger proposer therefore enlarges the favorable region in this check without reversing the aggregate disadvantage.

\paragraph{Takeaway.}
Neither direct nor decomposed contrastive-aware lightweight drafting consistently outperforms expert-aligned drafting in the tested regime. The stronger full-LM stress test enlarges the favorable region but does not reverse the position-level disadvantage. We therefore keep proposals expert-aligned and amateur-independent in the system experiments, applying contrastive scoring only during verification.

\subsection{Deployment-Level Setup}
\paragraph{Datasets and model pairs.} Following \citet{fu2025speculative}, we evaluate HumanEval~\citep{chen2021codex}, GSM8K~\citep{cobbe2021gsm8k}, MMLU~\citep{hendryckstest2021}, and CNN/DM~\citep{see-etal-2017-get} under three 8B-class configurations: \textsc{Llama-3}, \textsc{Qwen3}, and \textsc{Llama-EFT}. Their expert and amateur pairs are \textsc{Llama-3.1-8B-Instruct}/\textsc{Llama-3.2-1B-Instruct}, \textsc{Qwen3-8B}/\textsc{Qwen3-1.7B}, and \textsc{Llama-3.1-8B-Instruct}/\textsc{Llama-3.1-8B}. Baselines are vanilla CD, SCD, and CoS~\citep{fu2025speculative}. DCD$_{\textsc{eagle3}}$ uses community-released EAGLE3 drafters and requires no new drafter training. DCD$_{\textsc{ngram}}$ is reported as an alternative DCD instantiation with a simple prefix-matching N-gram proposer. Appendix~\ref{sec:appendix_models} lists the checkpoint identifiers.

\paragraph{Fairness protocol.}
All deployment methods use the same SGLang stack, single NVIDIA H200 GPU, prompts, maximum generation lengths, and expert and amateur pair within each operating point. The main speed tables use 200 samples per dataset with maximum generation length 256 and report three-run averages; the task-level accuracy check uses 1000 samples with maximum generation length 1024. Appendix Table~\ref{tab:gamma_tuning_summary} gives the draft-length tuning check.

\paragraph{Operating points and metrics.}
The main greedy and $T=1$ comparisons use $\alpha \in \{0.1, 0.5\}$; the latency decomposition uses $\alpha=0.1$; and the ablations sweep $\alpha$ and $\gamma$. We report tokens/sec and speedup over vanilla CD. When diagnosing speed sources, we also report accepted length and latency components. Appendix~\ref{sec:implementation_details} gives runtime details.

\subsection{System Validation of the Diagnostic: DCD}
\label{sec:main_results}

\begin{table}[t]
\centering
\scriptsize
\renewcommand{\arraystretch}{1.00}
\setlength{\tabcolsep}{2.5pt}
\begin{tabular*}{\columnwidth}{@{\extracolsep{\fill}}lccccc}
\toprule
\textbf{Method} & \textbf{HE} & \textbf{GSM} & \textbf{MMLU} & \textbf{CNN/DM} & \textbf{Avg.} \\
\midrule
\multicolumn{6}{c}{\textsc{Llama-3}} \\
\midrule
\multicolumn{6}{c}{$\alpha = 0.1$} \\
\cmidrule(lr){1-6}
SCD & 1.52$\times$ & 1.53$\times$ & 1.26$\times$ & 1.40$\times$ & 1.43$\times$ \\
CoS & 1.71$\times$ & 1.77$\times$ & 1.29$\times$ & 1.37$\times$ & 1.54$\times$ \\
DCD$_{\textsc{ngram}}$ & 1.48$\times$ & 1.43$\times$ & 1.29$\times$ & 2.07$\times$ & 1.56$\times$ \\
\rowcolor{highlightcolor} \textbf{DCD$_{\textsc{eagle3}}$} & \textbf{2.07$\times$} & \textbf{1.84$\times$} & \textbf{1.60$\times$} & \textbf{1.80$\times$} & \textbf{1.83$\times$} \\
\midrule
\multicolumn{6}{c}{$\alpha = 0.5$} \\
\cmidrule(lr){1-6}
SCD & 1.59$\times$ & 1.43$\times$ & 1.21$\times$ & 1.12$\times$ & 1.34$\times$ \\
CoS & 1.67$\times$ & 1.56$\times$ & 1.20$\times$ & 1.13$\times$ & 1.40$\times$ \\
DCD$_{\textsc{ngram}}$ & 1.39$\times$ & 1.34$\times$ & 1.55$\times$ & 2.10$\times$ & 1.59$\times$ \\
\rowcolor{highlightcolor} \textbf{DCD$_{\textsc{eagle3}}$} & \textbf{1.89$\times$} & \textbf{1.66$\times$} & \textbf{1.44$\times$} & \textbf{1.59$\times$} & \textbf{1.65$\times$} \\
\midrule
\multicolumn{6}{c}{\textsc{Qwen3}} \\
\midrule
\multicolumn{6}{c}{$\alpha = 0.1$} \\
\cmidrule(lr){1-6}
SCD & 1.38$\times$ & 1.35$\times$ & 0.99$\times$ & 1.05$\times$ & 1.19$\times$ \\
CoS & 1.45$\times$ & 1.33$\times$ & 1.17$\times$ & 1.14$\times$ & 1.27$\times$ \\
DCD$_{\textsc{ngram}}$ & 1.45$\times$ & 1.35$\times$ & 1.17$\times$ & 1.14$\times$ & 1.28$\times$ \\
\rowcolor{highlightcolor} \textbf{DCD$_{\textsc{eagle3}}$} & \textbf{2.30$\times$} & \textbf{2.01$\times$} & \textbf{1.66$\times$} & \textbf{1.85$\times$} & \textbf{1.95$\times$} \\
\midrule
\multicolumn{6}{c}{$\alpha = 0.5$} \\
\cmidrule(lr){1-6}
SCD & 1.23$\times$ & 1.24$\times$ & 0.88$\times$ & 0.82$\times$ & 1.05$\times$ \\
CoS & 1.25$\times$ & 1.31$\times$ & 0.97$\times$ & 0.79$\times$ & 1.08$\times$ \\
DCD$_{\textsc{ngram}}$ & 1.32$\times$ & 1.28$\times$ & 1.17$\times$ & 1.00$\times$ & 1.19$\times$ \\
\rowcolor{highlightcolor} \textbf{DCD$_{\textsc{eagle3}}$} & \textbf{1.98$\times$} & \textbf{1.85$\times$} & \textbf{1.60$\times$} & \textbf{1.39$\times$} & \textbf{1.71$\times$} \\
\midrule
\multicolumn{6}{c}{\textsc{Llama-EFT}} \\
\midrule
\multicolumn{6}{c}{$\alpha = 0.1$} \\
\cmidrule(lr){1-6}
SCD & 1.20$\times$ & 1.16$\times$ & 1.06$\times$ & 1.03$\times$ & 1.12$\times$ \\
CoS & 1.29$\times$ & 1.22$\times$ & 1.01$\times$ & 1.06$\times$ & 1.15$\times$ \\
DCD$_{\textsc{ngram}}$ & 1.43$\times$ & 1.31$\times$ & 1.32$\times$ & 2.06$\times$ & 1.52$\times$ \\
\rowcolor{highlightcolor} \textbf{DCD$_{\textsc{eagle3}}$} & \textbf{2.03$\times$} & \textbf{1.74$\times$} & \textbf{1.55$\times$} & \textbf{1.74$\times$} & \textbf{1.77$\times$} \\
\midrule
\multicolumn{6}{c}{$\alpha = 0.5$} \\
\cmidrule(lr){1-6}
SCD & 1.19$\times$ & 1.14$\times$ & 0.96$\times$ & 0.92$\times$ & 1.05$\times$ \\
CoS & 1.31$\times$ & 1.15$\times$ & 0.98$\times$ & 0.97$\times$ & 1.10$\times$ \\
DCD$_{\textsc{ngram}}$ & 1.55$\times$ & 1.35$\times$ & 1.31$\times$ & 2.11$\times$ & 1.58$\times$ \\
\rowcolor{highlightcolor} \textbf{DCD$_{\textsc{eagle3}}$} & \textbf{2.06$\times$} & \textbf{1.62$\times$} & \textbf{1.61$\times$} & \textbf{1.61$\times$} & \textbf{1.72$\times$} \\
\bottomrule
\end{tabular*}
\caption{DCD$_{\textsc{eagle3}}$ is strongest on average in every model family; DCD$_{\textsc{ngram}}$ also stays above vanilla CD on average. Avg. averages datasets; highlights mark DCD$_{\textsc{eagle3}}$.}
\label{tab:exp5_main_results}
\vspace{-4pt}
\end{table}

\paragraph{Main comparison.} Table~\ref{tab:exp5_main_results} shows DCD$_{\textsc{eagle3}}$ has the highest average speedup for every family and both $\alpha$ values. At the harder \textsc{Qwen3}, $\alpha=0.5$ setting, SCD and CoS drop below 1.0$\times$ on MMLU/CNN-DM, while DCD$_{\textsc{eagle3}}$ reaches 1.60$\times$/1.39$\times$. DCD$_{\textsc{ngram}}$ is included as a proposer-agnostic sanity check: it is weaker than EAGLE3 but remains above vanilla CD on average. Appendix~\ref{sec:appendix_exp5} gives the $T{=}1$ breakdown, Appendix~\ref{sec:appendix_multi_request} reports multi-request serving, and Appendix~\ref{sec:appendix_memory_kv} gives the KV-cache result.

This pattern validates the proposal-path diagnostic at system level. All methods use the same contrastive target and the same expert and amateur models, but DCD alone moves the expensive amateur out of the serial proposal path. The next two subsections separate the two remaining concerns: Figure~\ref{fig:performance_t0} checks that the verified CD target is preserved in task accuracy, and the latency decomposition in Section~\ref{sec:trade_off} shows why the cheaper proposal path can dominate even when SCD or CoS accept more draft tokens.

\begin{table}[t]
\centering
\scriptsize
\renewcommand{\arraystretch}{1.02}
\setlength{\tabcolsep}{2.8pt}
\begin{tabular*}{\columnwidth}{@{\extracolsep{\fill}}lcccc}
\toprule
\textbf{Setting} & \textbf{AR} & \textbf{Expert SD} & \textbf{DCD$_{\textsc{ngram}}$} & \textbf{DCD$_{\textsc{eagle3}}$} \\
\midrule
GSM8K & 1.27$\times$ & 2.55$\times$ & 1.43$\times$ & \textbf{1.75$\times$} \\
MMLU & 1.31$\times$ & 1.88$\times$ & 1.34$\times$ & \textbf{1.45$\times$} \\
\bottomrule
\end{tabular*}
\caption{At $\alpha{=}0$, DCD$_{\textsc{eagle3}}$ remains faster than vanilla CD and AR on Llama-3. Expert SD uses the same EAGLE3 drafter.}
\label{tab:alpha_zero_positioning}
\vspace{-4pt}
\end{table}

\paragraph{$\alpha=0$ boundary.}
At $\alpha=0$, CD reduces to the expert distribution. Table~\ref{tab:alpha_zero_positioning} positions DCD against autoregressive decoding and expert-only speculative decoding with the same EAGLE3 drafter. DCD$_{\textsc{eagle3}}$ remains above vanilla CD and autoregressive decoding at 1.75$\times$/1.45$\times$ on GSM8K/MMLU, while the expert-only row marks the expected upper bound when contrastive verification is removed.

\paragraph{70B-class validation.} In the greedy-only 70B extension, Table~\ref{tab:large_model_summary} uses the same $\alpha \in \{0.1,0.5\}$ operating points as the main comparison. DCD$_{\textsc{eagle3}}$ remains fastest on average, reaching 2.02$\times$ at $\alpha=0.1$ and 1.96$\times$ at $\alpha=0.5$. Appendix~\ref{sec:appendix_large_model} gives the per-dataset results.

\begin{table}[t]
\centering
\renewcommand{\arraystretch}{1.10}
\setlength{\tabcolsep}{4pt}
\small
\begin{tabular}{lcc}
\toprule
& \multicolumn{2}{c}{\textbf{Average speedup over vanilla CD}} \\
\cmidrule(lr){2-3}
\textbf{Method} & \textbf{$\alpha=0.1$} & \textbf{$\alpha=0.5$} \\
\midrule
SCD & 1.79$\times$ & 1.69$\times$ \\
CoS & 1.89$\times$ & 1.75$\times$ \\
DCD$_{\textsc{ngram}}$ & 1.46$\times$ & 1.44$\times$ \\
\rowcolor{highlightcolor} \textbf{DCD$_{\textsc{eagle3}}$} & \textbf{2.02$\times$} & \textbf{1.96$\times$} \\
\bottomrule
\end{tabular}
\caption{DCD$_{\textsc{eagle3}}$ stays near $2\times$ average speedup in the greedy 70B extension. Values average four datasets; full results are in Appendix~\ref{sec:appendix_large_model}.}
\label{tab:large_model_summary}
\vspace{-4pt}
\end{table}

\subsection{Task-Level Consistency Check}
\label{sec:performance_accuracy}
DCD changes only proposals, so the efficiency gains in Table~\ref{tab:exp5_main_results} should not be interpreted as a new decoding objective. Verification still targets $\pi_{CD}$; Appendix~\ref{sec:proof_lossless} proves losslessness. Figure~\ref{fig:performance_t0} shows the GSM8K greedy sweep for \textsc{Llama-3} with $\gamma{=}5$ and a matched DCD$_{\textsc{ngram}}$ comparison; Appendix~\ref{sec:appendix_performance_details} gives the full GSM8K/MMLU sweep under both $T{=}0$ and $T{=}1$.

\begin{figure}[t]
\centering
\begin{tikzpicture}
\begin{axis}[
    width=0.82\columnwidth,
    height=3.2cm,
    xlabel={Contrastive Strength ($\alpha$)},
    ylabel={Accuracy (\%)},
    title={GSM8K, $T{=}0$},
    title style={font=\small, yshift=-1ex},
    xmin=-0.05, xmax=1.05,
    ymin=75, ymax=90,
    xtick={0,0.2,0.4,0.6,0.8,1.0},
    ytick={75,78,81,84,87,90},
    grid=major,
    grid style={dashed, gray!40},
    label style={font=\small},
    tick label style={font=\scriptsize},
    scale only axis,
    trim axis left,
    trim axis right,
]
\draw[thick, dashed, gray!60!black, line width=1pt] (axis cs:-0.05,85.20) -- (axis cs:1.05,85.20);

\addplot[color=blue, mark=*, mark size=1.5pt, thick, line width=0.7pt, error bars/.cd, y dir=both, y explicit] coordinates {
    (0.0,85.20) +- (0,2.20) (0.1,87.00) +- (0,2.08) (0.2,86.90) +- (0,2.09) (0.3,86.30) +- (0,2.13) (0.4,85.10) +- (0,2.21) (0.5,84.70) +- (0,2.23) (0.6,83.50) +- (0,2.30) (0.7,83.50) +- (0,2.30) (0.8,82.00) +- (0,2.38) (0.9,80.40) +- (0,2.46) (1.0,79.80) +- (0,2.49)
};
\addplot[color=magenta!75!black, mark=star, mark size=1.5pt, thick, line width=0.7pt, error bars/.cd, y dir=both, y explicit] coordinates {
    (0.0,86.10) +- (0,2.14) (0.1,87.60) +- (0,2.04) (0.2,86.70) +- (0,2.10) (0.3,85.40) +- (0,2.19) (0.4,84.80) +- (0,2.23) (0.5,84.40) +- (0,2.25) (0.6,82.90) +- (0,2.33) (0.7,84.60) +- (0,2.24) (0.8,81.70) +- (0,2.40) (0.9,82.40) +- (0,2.36) (1.0,78.10) +- (0,2.56)
};
\addplot[color=red, mark=square*, mark size=1.5pt, thick, line width=0.7pt, error bars/.cd, y dir=both, y explicit] coordinates {
    (0.0,85.10) +- (0,2.21) (0.1,86.40) +- (0,2.12) (0.2,86.30) +- (0,2.13) (0.3,85.40) +- (0,2.19) (0.4,84.70) +- (0,2.23) (0.5,84.30) +- (0,2.25) (0.6,84.60) +- (0,2.24) (0.7,83.80) +- (0,2.28) (0.8,81.90) +- (0,2.39) (0.9,80.90) +- (0,2.44) (1.0,79.20) +- (0,2.52)
};
\addplot[color=green!60!black, mark=triangle*, mark size=1.5pt, thick, line width=0.7pt, error bars/.cd, y dir=both, y explicit] coordinates {
    (0.0,85.10) +- (0,2.21) (0.1,86.40) +- (0,2.12) (0.2,86.50) +- (0,2.12) (0.3,85.40) +- (0,2.19) (0.4,84.90) +- (0,2.22) (0.5,84.10) +- (0,2.27) (0.6,83.30) +- (0,2.31) (0.7,83.80) +- (0,2.28) (0.8,82.00) +- (0,2.38) (0.9,78.70) +- (0,2.54) (1.0,79.80) +- (0,2.49)
};
\addplot[color=orange, mark=diamond*, mark size=1.5pt, thick, line width=0.7pt, error bars/.cd, y dir=both, y explicit] coordinates {
    (0.0,85.60) +- (0,2.18) (0.1,86.90) +- (0,2.09) (0.2,86.60) +- (0,2.11) (0.3,85.50) +- (0,2.18) (0.4,85.10) +- (0,2.21) (0.5,84.90) +- (0,2.22) (0.6,83.60) +- (0,2.29) (0.7,83.90) +- (0,2.28) (0.8,81.90) +- (0,2.39) (0.9,80.20) +- (0,2.47) (1.0,79.10) +- (0,2.52)
};
\end{axis}
\end{tikzpicture}

\vspace{-0.3em}

\begin{tikzpicture}
\begin{axis}[
    hide axis,
    xmin=0, xmax=1,
    ymin=0, ymax=1,
    legend style={
        at={(0.5, 0.5)},
        anchor=center,
        legend columns=3,
        draw=none,
        fill=none,
        font=\small,
        /tikz/every even column/.append style={column sep=0.8em}
    },
]
\addlegendimage{color=blue, mark=*, mark size=1.5pt, thick}
\addlegendentry{DCD$_{\textsc{eagle3}}$}
\addlegendimage{color=magenta!75!black, mark=star, mark size=1.5pt, thick}
\addlegendentry{DCD$_{\textsc{ngram}}$}
\addlegendimage{color=red, mark=square*, mark size=1.5pt, thick}
\addlegendentry{SCD}
\addlegendimage{color=green!60!black, mark=triangle*, mark size=1.5pt, thick}
\addlegendentry{CoS}
\addlegendimage{color=orange, mark=diamond*, mark size=1.5pt, thick}
\addlegendentry{Vanilla CD}
\addlegendimage{thick, dashed, gray!60!black, line width=1pt}
\addlegendentry{AR}
\end{axis}
\end{tikzpicture}

\vspace{-0.5em}
\vspace{-5pt}

\caption{DCD$_{\textsc{eagle3}}$ tracks vanilla CD accuracy across $\alpha$ on Llama-3 GSM8K ($\gamma{=}5$). Error bars are 95\% confidence intervals; the dashed line is AR.}
\label{fig:performance_t0}
\vspace{-6pt}

\end{figure}
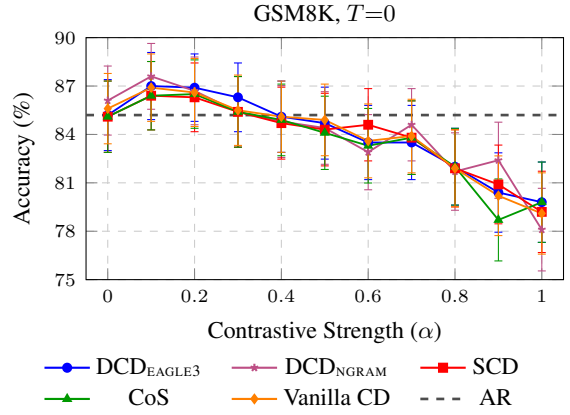

\paragraph{Consistency.}
Figure~\ref{fig:performance_t0} shows that DCD$_{\textsc{eagle3}}$ tracks vanilla CD across $\alpha$ on GSM8K under greedy decoding, with error bars largely overlapping across the operating range. It indicates that DCD preserves the contrastive target while changing only the proposal mechanism.

\subsection{Source of the DCD Speedup}
\label{sec:trade_off}
\begin{table}[t]
\centering
\small
\renewcommand{\arraystretch}{1.15}
\setlength{\tabcolsep}{3pt}
\begin{tabular*}{\columnwidth}{@{\extracolsep{\fill}}lcccccc}
    \toprule
    \textbf{Method} & $\boldsymbol{L}$ & $\boldsymbol{t_e}$ & $\boldsymbol{t_p}$ & $\boldsymbol{t_q}$ & \textbf{Speed} & \textbf{Theo.} \\
    \midrule
    \multicolumn{7}{l}{\textit{\textbf{Llama-3} (8B-Ins / 1B-Ins)}} \\
    \midrule
    CD & 0.00 & - & 6.95 & 3.83 & 1.00x & 1.00x \\
    SCD & 2.45 & - & 6.84 & 3.47 & 1.34x & 1.48x \\
    CoS & 2.45 & - & 7.08 & 3.59 & 1.34x & 1.48x \\
    DCD$_{\textsc{ngram}}$ & 0.48 & 0.00 & 7.04 & 4.06 & 1.36x & 1.43x \\
    \rowcolor{highlightcolor} \textbf{DCD$_{\textsc{eagle3}}$} & 1.44 & 0.63 & 7.98 & 4.38 & \textbf{1.50x} & \textbf{1.70x} \\
    \midrule
    \multicolumn{7}{l}{\textit{\textbf{Qwen3} (8B / 1.7B)}} \\
    \midrule
    CD & 0.00 & - & 8.47 & 6.52 & 1.00x & 1.00x \\
    SCD & 2.49 & - & 9.06 & 6.53 & 1.14x & 1.20x \\
    CoS & 2.49 & - & 9.15 & 6.63 & 1.17x & 1.23x \\
    DCD$_{\textsc{ngram}}$ & 0.41 & 0.00 & 9.07 & 7.23 & 1.22x & 1.30x \\
    \rowcolor{highlightcolor} \textbf{DCD$_{\textsc{eagle3}}$} & 1.48 & 0.61 & 10.15 & 7.93 & \textbf{1.61x} & \textbf{1.76x} \\
    \midrule
    \multicolumn{7}{l}{\textit{\textbf{Llama-EFT} (8B-Ins / 8B)}} \\
    \midrule
    CD & 0.00 & - & 7.19 & 7.25 & 1.00x & 1.00x \\
    SCD & 2.64 & - & 6.77 & 6.47 & 1.20x & 1.28x \\
    CoS & 2.64 & - & 7.67 & 7.38 & 1.11x & 1.17x \\
    DCD$_{\textsc{ngram}}$ & 0.49 & 0.00 & 6.98 & 7.12 & 1.45x & 1.52x \\
    \rowcolor{highlightcolor} \textbf{DCD$_{\textsc{eagle3}}$} & 1.46 & 0.60 & 7.42 & 7.50 & \textbf{1.79x} & \textbf{1.98x} \\
    \bottomrule
\end{tabular*}
    \caption{DCD$_{\textsc{eagle3}}$ beats SCD/CoS by lowering MMLU serial proposal latency despite smaller $L$. Results use $T{=}0$, $\alpha{=}0.1$, $\gamma{=}5$; $t_e$, $t_p$, and $t_q$ are per-step ms.}
\label{tab:accept_rate_latency}
\vspace{-4pt}
\end{table}

Table~\ref{tab:accept_rate_latency} decomposes MMLU speedup using $S \approx (L+1)/T_{\text{iter}}$ and reports accepted length plus drafter, expert, and amateur latencies. The ``Theo.'' column evaluates this coarse serial model with the measured components; measured speed remains the deployment metric.

Serial proposal cost explains the speedup. Across the three MMLU settings, SCD/CoS spend 3.5 to 7.4ms per proposal step, while DCD$_{\textsc{eagle3}}$ uses about 0.6ms. On \textsc{Llama-3}, SCD reaches $L=2.45$ and DCD$_{\textsc{eagle3}}$ reaches $L=1.44$, yet DCD$_{\textsc{eagle3}}$ is faster because cheaper proposals offset the shorter accepted length. Appendix~\ref{sec:theory_speed_calc} gives the multiplicative breakdown into acceptance and serial-cost factors.

\definecolor{dcdblue}{RGB}{31,119,180}
\definecolor{scdorange}{RGB}{255,127,14}
\definecolor{cosgreen}{RGB}{44,160,44}

\pgfplotsset{
    every axis/.append style={
        label style={font=\scriptsize},
        tick label style={font=\scriptsize, inner sep=0.5pt},
        title style={font=\scriptsize, align=center, at={(0.5, 1.01)}, anchor=south},
        ylabel style={xshift=10pt, yshift=-0.08cm},
        legend style={font=\scriptsize},
        grid style={dashed, gray!40},
        scale only axis,
        trim axis left,
        trim axis right,
    }
}
\begin{figure*}[t]
\centering
\begin{subfigure}[b]{0.22\textwidth}
\centering
\begin{tikzpicture}
\begin{axis}[
    width=0.84\textwidth,
    height=2.0cm,
    xlabel={Contrastive Strength ($\alpha$)},
    ylabel={},
    xmin=-0.02, xmax=0.52,
    ymin=60, ymax=105,
    xtick={0,0.1,0.2,0.3,0.4,0.5},
    ytick={60,70,80,90,100},
    grid=major,
    title={\shortstack{\textsc{Llama-3}\\\textsc{MMLU}}}
]
\addplot[color=dcdblue, mark=*, mark size=1.5pt, thick] coordinates {
    (0.0, 100.0) (0.1, 94.5) (0.2, 91.2) (0.3, 87.9) (0.4, 85.3) (0.5, 79.5)
};
\addplot[color=scdorange, mark=square*, mark size=1.5pt, thick] coordinates {
    (0.0, 100.0) (0.1, 94.3) (0.2, 89.5) (0.3, 85.5) (0.4, 80.8) (0.5, 74.0)
};
\addplot[color=cosgreen, mark=triangle*, mark size=1.5pt, thick] coordinates {
    (0.0, 100.0) (0.1, 94.3) (0.2, 89.6) (0.3, 85.8) (0.4, 81.6) (0.5, 74.2)
};
\end{axis}
\end{tikzpicture}
\end{subfigure}
\hfill
\begin{subfigure}[b]{0.22\textwidth}
\centering
\begin{tikzpicture}
\begin{axis}[
    width=0.84\textwidth,
    height=2.0cm,
    xlabel={Contrastive Strength ($\alpha$)},
    ylabel={},
    xmin=-0.02, xmax=0.52,
    ymin=60, ymax=105,
    xtick={0,0.1,0.2,0.3,0.4,0.5},
    ytick={60,70,80,90,100},
    grid=major,
    yticklabels={,,},
    title={\shortstack{\textsc{Qwen3}\\\textsc{MMLU}}}
]
\addplot[color=dcdblue, mark=*, mark size=1.5pt, thick] coordinates {
    (0.0, 100.0) (0.1, 97.9) (0.2, 93.9) (0.3, 90.3) (0.4, 82.7) (0.5, 77.4)
};
\addplot[color=scdorange, mark=square*, mark size=1.5pt, thick] coordinates {
    (0.0, 100.0) (0.1, 93.7) (0.2, 86.1) (0.3, 79.1) (0.4, 71.0) (0.5, 62.9)
};
\addplot[color=cosgreen, mark=triangle*, mark size=1.5pt, thick] coordinates {
    (0.0, 100.0) (0.1, 93.7) (0.2, 86.1) (0.3, 79.1) (0.4, 71.0) (0.5, 62.9)
};
\end{axis}
\end{tikzpicture}
\end{subfigure}
\hfill
\begin{subfigure}[b]{0.22\textwidth}
\centering
\begin{tikzpicture}
\begin{axis}[
    width=0.84\textwidth,
    height=2.0cm,
    xlabel={Contrastive Strength ($\alpha$)},
    ylabel={},
    xmin=-0.02, xmax=0.52,
    ymin=60, ymax=105,
    xtick={0,0.1,0.2,0.3,0.4,0.5},
    ytick={60,70,80,90,100},
    grid=major,
    yticklabels={,,},
    title={\shortstack{\textsc{Llama-3}\\\textsc{GSM8K}}}
]
\addplot[color=dcdblue, mark=*, mark size=1.5pt, thick] coordinates {
    (0.0, 100.0) (0.1, 98.8) (0.2, 97.4) (0.3, 99.5) (0.4, 97.5) (0.5, 95.1)
};
\addplot[color=scdorange, mark=square*, mark size=1.5pt, thick] coordinates {
    (0.0, 100.0) (0.1, 96.2) (0.2, 94.3) (0.3, 92.6) (0.4, 90.5) (0.5, 87.5)
};
\addplot[color=cosgreen, mark=triangle*, mark size=1.5pt, thick] coordinates {
    (0.0, 100.0) (0.1, 96.1) (0.2, 94.5) (0.3, 92.9) (0.4, 90.5) (0.5, 87.8)
};
\end{axis}
\end{tikzpicture}
\end{subfigure}
\hfill
\begin{subfigure}[b]{0.22\textwidth}
\centering
\begin{tikzpicture}
\begin{axis}[
    width=0.84\textwidth,
    height=2.0cm,
    xlabel={Contrastive Strength ($\alpha$)},
    ylabel={},
    xmin=-0.02, xmax=0.52,
    ymin=60, ymax=105,
    xtick={0,0.1,0.2,0.3,0.4,0.5},
    ytick={60,70,80,90,100},
    grid=major,
    yticklabels={,,},
    title={\shortstack{\textsc{Qwen3}\\\textsc{GSM8K}}}
]
\addplot[color=dcdblue, mark=*, mark size=1.5pt, thick] coordinates {
    (0.0, 100.0) (0.1, 98.3) (0.2, 96.9) (0.3, 93.7) (0.4, 89.2) (0.5, 81.9)
};
\addplot[color=scdorange, mark=square*, mark size=1.5pt, thick] coordinates {
    (0.0, 100.0) (0.1, 95.2) (0.2, 89.9) (0.3, 83.3) (0.4, 74.9) (0.5, 65.8)
};
\addplot[color=cosgreen, mark=triangle*, mark size=1.5pt, thick] coordinates {
    (0.0, 100.0) (0.1, 95.2) (0.2, 89.9) (0.3, 83.2) (0.4, 75.0) (0.5, 65.8)
};
\end{axis}
\end{tikzpicture}
\end{subfigure}

\vspace{-0.5em}
\begin{center}
\begin{tikzpicture}
\begin{axis}[
    hide axis,
    xmin=0, xmax=1,
    ymin=0, ymax=1,
    legend style={
        draw=none,
        fill=none,
        legend columns=-1,
        /tikz/every even column/.append style={column sep=0.5cm},
        font=\small
    },
    legend pos=north west
]
\addlegendimage{color=dcdblue, mark=*, mark size=2pt, thick}
\addlegendentry{DCD$_{\textsc{eagle3}}$}
\addlegendimage{color=scdorange, mark=square*, mark size=2pt, thick}
\addlegendentry{SCD}
\addlegendimage{color=cosgreen, mark=triangle*, mark size=2pt, thick}
\addlegendentry{CoS}
\end{axis}
\end{tikzpicture}
\end{center}
\vspace{-1.5em}
\caption{DCD$_{\textsc{eagle3}}$ loses accepted length more mildly as contrastive strength increases. Relative mean accepted length is normalized to $\alpha{=}0$ in each panel.}
\label{fig:alpha_ablation}
\vspace{-8pt}
\end{figure*}
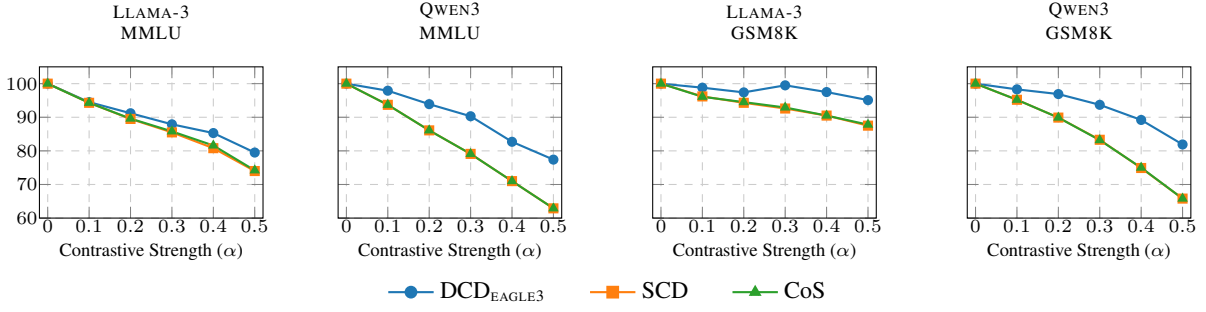

\subsection{Ablation Studies}
\label{sec:ablation_studies}
We test the $\alpha$-robustness intuition behind DCD.

Figure~\ref{fig:alpha_ablation} normalizes $L$ to 100\% at $\alpha{=}0$. This view separates robustness from absolute acceptance rate: SCD or CoS can start from a higher $L$ at $\alpha{=}0$, while the diagnostic tracks how quickly each proposal distribution drifts away from the contrastive target as the amateur penalty strengthens.

DCD$_{\textsc{eagle3}}$ degrades more gradually than amateur-coupled baselines in all four panels. On \textsc{Llama-3} MMLU, the gap is moderate: at $\alpha{=}0.5$, DCD$_{\textsc{eagle3}}$ retains 79.5\% of its $\alpha{=}0$ accepted length, while SCD and CoS retain about 74\%. The separation becomes larger on \textsc{Qwen3} MMLU, where SCD retains 62.9\% and DCD$_{\textsc{eagle3}}$ retains 77.4\%. This empirical pattern is consistent with the alignment intuition: as the contrastive target moves away from the amateur distribution, an expert-aligned lightweight drafter can lose acceptance more slowly than a proposal path tied to the amateur.

The GSM8K panels show the same trend but with different severity. \textsc{Llama-3} GSM8K is relatively stable for all methods, yet DCD$_{\textsc{eagle3}}$ still remains closest to its $\alpha{=}0$ acceptance level at $\alpha{=}0.5$ (95.1\% versus about 88\% for SCD/CoS). \textsc{Qwen3} GSM8K is harder: the amateur-coupled curves fall to 65.8\%, while DCD$_{\textsc{eagle3}}$ stays at 81.9\%. The cross-task consistency matters because the main speedups in Table~\ref{tab:exp5_main_results} come from both proposal cost and acceptance; robustness in $L$ prevents the serial-cost advantage from being erased at larger $\alpha$. Appendix Figure~\ref{fig:gamma_ablation_appendix} and Table~\ref{tab:dcd_single_request_sweep_200} give the $\gamma$ and throughput-grid details.

\section{Related Work}

\textbf{Speculative Decoding.}
Speculative decoding accelerates generation with draft-then-verify~\citep{leviathan2023fast,chen2023accelerating}. Recent work improves proposals through parallel heads, speculation trees and trajectories, self-speculation, distillation, and graph-structured acceptance~\citep{cai2024medusa,miao2024specinfer,fu2024break,ElhoushiSLHWL0A24,Zhang00S0CM24,ZhouLRMRKKA24,GongLWWWC0024}. The EAGLE series~\citep{li2024eagle2,li2024eagle,li2025eagle3} drafts at the feature level. A contrastive target raises a distinct question: should the drafter imitate $\pi_{CD}$, or should the amateur appear only in verification?

\textbf{Contrastive Decoding.}
Contrastive Decoding (CD)~\citep{li2023contrastive} amplifies the expert--amateur gap. Later work applies this idea to layer contrast, contextual understanding, noisy retrieval, vision-language hallucination mitigation, and tuning-by-proxy~\citep{chuang2023dola,DBLP:conf/naacl/ZhaoM0A24,DBLP:conf/emnlp/KimKPPCKYLK24,leng2024mitigating,wang2024mitigating,mitchell2024emulator,liu2024tuning}. CD still requires two model passes per token.

\textbf{Accelerating Contrastive Decoding.}
Existing CD acceleration methods mostly retain the amateur in proposal generation. SCD~\citep{yuan2024speculative} and CoS~\citep{fu2025speculative} use it as a drafter and retain its serial cost. Section~\ref{sec:why_not_cd} finds no consistent benefit from adding the contrastive signal to a constrained lightweight proposer, motivating DCD's expert-aligned path.

\section{Conclusion}
Contrastive Decoding (CD) improves generation but adds an expensive amateur pass. We ask whether speculative CD should use the contrastive signal during drafting or only verification. Matched Cross-$\alpha$ and Approximate Dual-Drafter diagnostics show that contrastive-aware lightweight drafting does not consistently outperform expert-aligned drafting: the correction is usually smaller than drafter error, which reconstruction can amplify. Decoupled Contrastive Decoding (DCD) therefore uses an expert-aligned lightweight proposer and reserves the amateur for unchanged CD verification, preserving vanilla CD's output distribution. Across the main 8B settings, EAGLE3-based DCD yields 1.65--1.95$\times$ average greedy speedups and cuts MMLU proposal-path latency by 5--12$\times$ relative to amateur-coupled proposal paths.

\clearpage

\section*{Limitations}
Our evaluation follows the public benchmark mix and fixed decoding budgets used in the main study. It does not exhaust every workload shape: very long contexts, multilingual or domain-specific prompts, and adaptive choices of $\alpha$, draft length, or the CD plausibility threshold may shift the best operating point.

\bibliography{references/core,references/sd,references/cd,references/eagle,references/dataset_model,references/from_cos,references/others,references/intro}

\begin{thebibliography}{31}
\providecommand{\natexlab}[1]{#1}

\bibitem[{Cai et~al.(2024)Cai, Li, Geng, Peng, Lee, Chen, and Dao}]{cai2024medusa}
Tianle Cai, Yuhong Li, Zhengyang Geng, Hongwu Peng, Jason~D Lee, Deming Chen, and Tri Dao. 2024.
\newblock {MEDUSA}: Simple {LLM} inference acceleration framework with multiple decoding heads.
\newblock In \emph{Proceedings of the 41st International Conference on Machine Learning}, pages 5209--5235.

\bibitem[{Chen et~al.(2023)Chen, Borgeaud, Irving, Lespiau, Sifre, and Jumper}]{chen2023accelerating}
Charlie Chen, Sebastian Borgeaud, Geoffrey Irving, Jean-Baptiste Lespiau, Laurent Sifre, and John Jumper. 2023.
\newblock Accelerating large language model decoding with speculative sampling.
\newblock \emph{arXiv preprint arXiv:2302.01318}.

\bibitem[{Chen et~al.(2021)Chen, Tworek, Jun, Yuan, de~Oliveira~Pinto, Kaplan, Edwards, Burda, Joseph, Brockman, Ray, Puri, Krueger, Petrov, Khlaaf, Sastry, Mishkin, Chan, Gray, Ryder, Pavlov, Power, Kaiser, Bavarian, Winter, Tillet, Such, Cummings, Plappert, Chantzis, Barnes, Herbert-Voss, Guss, Nichol, Paino, Tezak, Tang, Babuschkin, Balaji, Jain, Saunders, Hesse, Carr, Leike, Achiam, Misra, Morikawa, Radford, Knight, Brundage, Murati, Mayer, Welinder, McGrew, Amodei, McCandlish, Sutskever, and Zaremba}]{chen2021codex}
Mark Chen, Jerry Tworek, Heewoo Jun, Qiming Yuan, Henrique~Ponde de~Oliveira~Pinto, Jared Kaplan, Harri Edwards, Yuri Burda, Nicholas Joseph, Greg Brockman, Alex Ray, Raul Puri, Gretchen Krueger, Michael Petrov, Heidy Khlaaf, Girish Sastry, Pamela Mishkin, Brooke Chan, Scott Gray, and 39 others. 2021.
\newblock \href {https://arxiv.org/abs/2107.03374} {Evaluating large language models trained on code}.
\newblock \emph{Preprint}, arXiv:2107.03374.

\bibitem[{Chuang et~al.(2024)Chuang, Xie, Luo, Kim, Glass, and He}]{chuang2023dola}
Yung{-}Sung Chuang, Yujia Xie, Hongyin Luo, Yoon Kim, James~R. Glass, and Pengcheng He. 2024.
\newblock \href {https://openreview.net/forum?id=Th6NyL07na} {{DoLa}: Decoding by contrasting layers improves factuality in large language models}.
\newblock In \emph{The Twelfth International Conference on Learning Representations, {ICLR} 2024, Vienna, Austria, May 7-11, 2024}. OpenReview.net.

\bibitem[{Cobbe et~al.(2021)Cobbe, Kosaraju, Bavarian, Chen, Jun, Kaiser, Plappert, Tworek, Hilton, Nakano, Hesse, and Schulman}]{cobbe2021gsm8k}
Karl Cobbe, Vineet Kosaraju, Mohammad Bavarian, Mark Chen, Heewoo Jun, Lukasz Kaiser, Matthias Plappert, Jerry Tworek, Jacob Hilton, Reiichiro Nakano, Christopher Hesse, and John Schulman. 2021.
\newblock Training verifiers to solve math word problems.
\newblock \emph{arXiv preprint arXiv:2110.14168}.

\bibitem[{Elhoushi et~al.(2024)Elhoushi, Shrivastava, Liskovich, Hosmer, Wasti, Lai, Mahmoud, Acun, Agarwal, Roman, Aly, Chen, and Wu}]{ElhoushiSLHWL0A24}
Mostafa Elhoushi, Akshat Shrivastava, Diana Liskovich, Basil Hosmer, Bram Wasti, Liangzhen Lai, Anas Mahmoud, Bilge Acun, Saurabh Agarwal, Ahmed Roman, Ahmed~A Aly, Beidi Chen, and Carole{-}Jean Wu. 2024.
\newblock \href {https://doi.org/10.18653/V1/2024.ACL-LONG.681} {{LayerSkip}: Enabling early exit inference and self-speculative decoding}.
\newblock In \emph{Proceedings of the 62nd Annual Meeting of the Association for Computational Linguistics (Volume 1: Long Papers), {ACL} 2024, Bangkok, Thailand, August 11-16, 2024}, pages 12622--12642. Association for Computational Linguistics.

\bibitem[{Fu et~al.(2025)Fu, Jiang, Chen, Fan, Geng, and Yang}]{fu2025speculative}
Jiale Fu, Yuchu Jiang, Junkai Chen, Jiaming Fan, Xin Geng, and Xu~Yang. 2025.
\newblock Fast large language model collaborative decoding via speculation.
\newblock In \emph{Forty-second International Conference on Machine Learning}.

\bibitem[{Fu et~al.(2024)Fu, Bailis, Stoica, and Zhang}]{fu2024break}
Yichao Fu, Peter Bailis, Ion Stoica, and Hao Zhang. 2024.
\newblock \href {https://openreview.net/forum?id=eDjvSFOkXw} {Break the sequential dependency of {LLM} inference using lookahead decoding}.
\newblock In \emph{Forty-first International Conference on Machine Learning, {ICML} 2024, Vienna, Austria, July 21-27, 2024}. OpenReview.net.

\bibitem[{Gong et~al.(2024)Gong, Liu, Wang, Wu, Wang, Cai, Zhao, and Yan}]{GongLWWWC0024}
Zhuocheng Gong, Jiahao Liu, Ziyue Wang, Pengfei Wu, Jingang Wang, Xunliang Cai, Dongyan Zhao, and Rui Yan. 2024.
\newblock \href {https://doi.org/10.18653/V1/2024.FINDINGS-ACL.677} {Graph-structured speculative decoding}.
\newblock In \emph{Findings of the Association for Computational Linguistics, {ACL} 2024, Bangkok, Thailand and virtual meeting, August 11-16, 2024}, pages 11404--11415. Association for Computational Linguistics.

\bibitem[{Hendrycks et~al.(2021)Hendrycks, Burns, Basart, Zou, Mazeika, Song, and Steinhardt}]{hendryckstest2021}
Dan Hendrycks, Collin Burns, Steven Basart, Andy Zou, Mantas Mazeika, Dawn Song, and Jacob Steinhardt. 2021.
\newblock Measuring massive multitask language understanding.
\newblock \emph{Proceedings of the International Conference on Learning Representations (ICLR)}.

\bibitem[{Huang et~al.(2025)Huang, Yu, Ma, Zhong, Feng, Wang, Chen, Peng, Feng, Qin, and Liu}]{DBLP:journals/tois/HuangYMZFWCPFQL25}
Lei Huang, Weijiang Yu, Weitao Ma, Weihong Zhong, Zhangyin Feng, Haotian Wang, Qianglong Chen, Weihua Peng, Xiaocheng Feng, Bing Qin, and Ting Liu. 2025.
\newblock \href {https://doi.org/10.1145/3703155} {A survey on hallucination in large language models: Principles, taxonomy, challenges, and open questions}.
\newblock \emph{{ACM} Trans. Inf. Syst.}, 43(2):42:1--42:55.

\bibitem[{Kim et~al.(2024)Kim, Kim, Park, Park, Cho, Kim, Yoo, Lee, and Kim}]{DBLP:conf/emnlp/KimKPPCKYLK24}
Youna Kim, Hyuhng~Joon Kim, Cheonbok Park, Choonghyun Park, Hyunsoo Cho, Junyeob Kim, Kang~Min Yoo, Sang{-}goo Lee, and Taeuk Kim. 2024.
\newblock \href {https://doi.org/10.18653/V1/2024.FINDINGS-EMNLP.136} {Adaptive contrastive decoding in retrieval-augmented generation for handling noisy contexts}.
\newblock In \emph{Findings of the Association for Computational Linguistics: {EMNLP} 2024, Miami, Florida, USA, November 12-16, 2024}, pages 2421--2431. Association for Computational Linguistics.

\bibitem[{Leng et~al.(2024)Leng, Zhang, Chen, Li, Lu, Miao, and Bing}]{leng2024mitigating}
Sicong Leng, Hang Zhang, Guanzheng Chen, Xin Li, Shijian Lu, Chunyan Miao, and Lidong Bing. 2024.
\newblock Mitigating object hallucinations in large vision-language models through visual contrastive decoding.
\newblock In \emph{Proceedings of the IEEE/CVF Conference on Computer Vision and Pattern Recognition}, pages 13872--13882.

\bibitem[{Leviathan et~al.(2023)Leviathan, Kalman, and Matias}]{leviathan2023fast}
Yaniv Leviathan, Matan Kalman, and Yossi Matias. 2023.
\newblock Fast inference from transformers via speculative decoding.
\newblock In \emph{International Conference on Machine Learning}, pages 19274--19286. PMLR.

\bibitem[{Li et~al.(2023)Li, Holtzman, Fried, Liang, Eisner, Hashimoto, Zettlemoyer, and Lewis}]{li2023contrastive}
Xiang~Lisa Li, Ari Holtzman, Daniel Fried, Percy Liang, Jason Eisner, Tatsunori~B Hashimoto, Luke Zettlemoyer, and Mike Lewis. 2023.
\newblock Contrastive decoding: Open-ended text generation as optimization.
\newblock In \emph{Proceedings of the 61st Annual Meeting of the Association for Computational Linguistics (Volume 1: Long Papers)}, pages 12286--12312.

\bibitem[{Li et~al.(2024{\natexlab{a}})Li, Wei, Zhang, and Zhang}]{li2024eagle2}
Yuhui Li, Fangyun Wei, Chao Zhang, and Hongyang Zhang. 2024{\natexlab{a}}.
\newblock {EAGLE-2}: Faster inference of language models with dynamic draft trees.
\newblock In \emph{Empirical Methods in Natural Language Processing}.

\bibitem[{Li et~al.(2024{\natexlab{b}})Li, Wei, Zhang, and Zhang}]{li2024eagle}
Yuhui Li, Fangyun Wei, Chao Zhang, and Hongyang Zhang. 2024{\natexlab{b}}.
\newblock {EAGLE}: Speculative sampling requires rethinking feature uncertainty.
\newblock In \emph{International Conference on Machine Learning}.

\bibitem[{Li et~al.(2025)Li, Wei, Zhang, and Zhang}]{li2025eagle3}
Yuhui Li, Fangyun Wei, Chao Zhang, and Hongyang Zhang. 2025.
\newblock {EAGLE-3}: Scaling up inference acceleration of large language models via training-time test.
\newblock In \emph{Annual Conference on Neural Information Processing Systems}.

\bibitem[{Liu et~al.(2024)Liu, Han, Wang, Tsvetkov, Choi, and Smith}]{liu2024tuning}
Alisa Liu, Xiaochuang Han, Yizhong Wang, Yulia Tsvetkov, Yejin Choi, and Noah~A. Smith. 2024.
\newblock \href {https://doi.org/10.48550/ARXIV.2401.08565} {Tuning language models by proxy}.
\newblock \emph{CoRR}, abs/2401.08565.

\bibitem[{Miao et~al.(2024)Miao, Oliaro, Zhang, Cheng, Wang, Zhang, Wong, Zhu, Yang, Shi, Shi, Chen, Arfeen, Abhyankar, and Jia}]{miao2024specinfer}
Xupeng Miao, Gabriele Oliaro, Zhihao Zhang, Xinhao Cheng, Zeyu Wang, Zhengxin Zhang, Rae Ying~Yee Wong, Alan Zhu, Lijie Yang, Xiaoxiang Shi, Chunan Shi, Zhuoming Chen, Daiyaan Arfeen, Reyna Abhyankar, and Zhihao Jia. 2024.
\newblock {SpecInfer}: Accelerating large language model serving with tree-based speculative inference and verification.
\newblock In \emph{Proceedings of the 29th ACM International Conference on Architectural Support for Programming Languages and Operating Systems, Volume 3}, pages 932--949.

\bibitem[{Mitchell et~al.(2024)Mitchell, Rafailov, Sharma, Finn, and Manning}]{mitchell2024emulator}
Eric Mitchell, Rafael Rafailov, Archit Sharma, Chelsea Finn, and Christopher~D. Manning. 2024.
\newblock \href {https://openreview.net/forum?id=Eo7kv0sllr} {An emulator for fine-tuning large language models using small language models}.
\newblock In \emph{The Twelfth International Conference on Learning Representations, {ICLR} 2024, Vienna, Austria, May 7-11, 2024}. OpenReview.net.

\bibitem[{O'Brien and Lewis(2023)}]{obrien2023contrastive}
Sean O'Brien and Mike Lewis. 2023.
\newblock \href {https://doi.org/10.48550/ARXIV.2309.09117} {Contrastive decoding improves reasoning in large language models}.
\newblock \emph{CoRR}, abs/2309.09117.

\bibitem[{Plaat et~al.(2024)Plaat, Wong, Verberne, Broekens, van Stein, and B{\"{a}}ck}]{DBLP:journals/corr/abs-2407-11511}
Aske Plaat, Annie Wong, Suzan Verberne, Joost Broekens, Niki van Stein, and Thomas B{\"{a}}ck. 2024.
\newblock \href {https://doi.org/10.48550/ARXIV.2407.11511} {Reasoning with large language models, a survey}.
\newblock \emph{CoRR}, abs/2407.11511.

\bibitem[{See et~al.(2017)See, Liu, and Manning}]{see-etal-2017-get}
Abigail See, Peter~J. Liu, and Christopher~D. Manning. 2017.
\newblock \href {https://doi.org/10.18653/v1/P17-1099} {Get to the point: Summarization with pointer-generator networks}.
\newblock In \emph{Proceedings of the 55th Annual Meeting of the Association for Computational Linguistics (Volume 1: Long Papers)}, pages 1073--1083, Vancouver, Canada. Association for Computational Linguistics.

\bibitem[{Wang et~al.(2024)Wang, Pan, Ding, and Biemann}]{wang2024mitigating}
Xintong Wang, Jingheng Pan, Liang Ding, and Chris Biemann. 2024.
\newblock \href {https://doi.org/10.18653/V1/2024.FINDINGS-ACL.937} {Mitigating hallucinations in large vision-language models with instruction contrastive decoding}.
\newblock In \emph{Findings of the Association for Computational Linguistics, {ACL} 2024, Bangkok, Thailand and virtual meeting, August 11-16, 2024}, pages 15840--15853. Association for Computational Linguistics.

\bibitem[{Xia et~al.(2024)Xia, Yang, Dong, Wang, Li, Ge, Liu, Li, and Sui}]{xia2024unlocking}
Heming Xia, Zhe Yang, Qingxiu Dong, Peiyi Wang, Yongqi Li, Tao Ge, Tianyu Liu, Wenjie Li, and Zhifang Sui. 2024.
\newblock \href {https://doi.org/10.18653/v1/2024.findings-acl.456} {Unlocking efficiency in large language model inference: A comprehensive survey of speculative decoding}.
\newblock In \emph{Findings of the Association for Computational Linguistics: ACL 2024}, pages 7655--7671, Bangkok, Thailand and virtual meeting. Association for Computational Linguistics.

\bibitem[{Yuan et~al.(2024)Yuan, Lu, Huang, Yuan, and Zhou}]{yuan2024speculative}
Hongyi Yuan, Keming Lu, Fei Huang, Zheng Yuan, and Chang Zhou. 2024.
\newblock Speculative contrastive decoding.
\newblock In \emph{Proceedings of the 62nd Annual Meeting of the Association for Computational Linguistics (Volume 2: Short Papers)}, pages 56--64.

\bibitem[{Zhang et~al.(2024)Zhang, Wang, Li, Shou, Chen, Chen, and Mehrotra}]{Zhang00S0CM24}
Jun Zhang, Jue Wang, Huan Li, Lidan Shou, Ke~Chen, Gang Chen, and Sharad Mehrotra. 2024.
\newblock \href {https://doi.org/10.18653/V1/2024.ACL-LONG.607} {Draft{\&} verify: Lossless large language model acceleration via self-speculative decoding}.
\newblock In \emph{Proceedings of the 62nd Annual Meeting of the Association for Computational Linguistics (Volume 1: Long Papers), {ACL} 2024, Bangkok, Thailand, August 11-16, 2024}, pages 11263--11282. Association for Computational Linguistics.

\bibitem[{Zhao et~al.(2024)Zhao, Monti, Lehmann, and Assem}]{DBLP:conf/naacl/ZhaoM0A24}
Zheng Zhao, Emilio Monti, Jens Lehmann, and Haytham Assem. 2024.
\newblock \href {https://doi.org/10.18653/V1/2024.NAACL-LONG.237} {Enhancing contextual understanding in large language models through contrastive decoding}.
\newblock In \emph{Proceedings of the 2024 Conference of the North American Chapter of the Association for Computational Linguistics: Human Language Technologies (Volume 1: Long Papers), {NAACL} 2024, Mexico City, Mexico, June 16-21, 2024}, pages 4225--4237. Association for Computational Linguistics.

\bibitem[{Zheng et~al.(2024)Zheng, Yin, Xie, Sun, Huang, Yu, Cao, Kozyrakis, Stoica, Gonzalez, Barrett, and Sheng}]{zheng2024sglang}
Lianmin Zheng, Liangsheng Yin, Zhiqiang Xie, Chuyue Sun, Jeff Huang, Cody~Hao Yu, Shiyi Cao, Christos Kozyrakis, Ion Stoica, Joseph~E. Gonzalez, Clark~W. Barrett, and Ying Sheng. 2024.
\newblock {SGLang}: Efficient execution of structured language model programs.
\newblock \emph{Advances in Neural Information Processing Systems}, 37:62557--62583.

\bibitem[{Zhou et~al.(2024)Zhou, Lyu, Rawat, Menon, Rostamizadeh, Kumar, Kagy, and Agarwal}]{ZhouLRMRKKA24}
Yongchao Zhou, Kaifeng Lyu, Ankit~Singh Rawat, Aditya~Krishna Menon, Afshin Rostamizadeh, Sanjiv Kumar, Jean{-}Fran{\c{c}}ois Kagy, and Rishabh Agarwal. 2024.
\newblock \href {https://openreview.net/forum?id=rsY6J3ZaTF} {{DistillSpec}: Improving speculative decoding via knowledge distillation}.
\newblock In \emph{The Twelfth International Conference on Learning Representations, {ICLR} 2024, Vienna, Austria, May 7-11, 2024}. OpenReview.net.

\end{thebibliography}

\clearpage
\appendix

\section{Proposal-Side Diagnostics}
\label{sec:appendix_proposal_diagnostics}
\subsection{Analysis of Contrastive-Aligned Drafting Strategies}
\label{sec:cd_aligned_analysis}
This subsection reports supplementary measurements for the contrastive-aligned drafting results in Section~\ref{sec:why_not_cd}. It restates the matched-control setup, then reports per-position reconstruction checks and bucketed diagnostics over model groups, datasets, and contrastive strengths. The matched control remains the two-model-pair, two-dataset study in Section~\ref{sec:why_not_cd}; the later views are supplementary diagnostics.
\paragraph{Measurement Methodology.}
We use two related measurement setups. For the matched-control KL and reconstruction checks referenced in the main text, we run per-position forward passes on 200 GSM8K examples for the \textsc{Llama-3} and \textsc{Qwen3} model pairs under greedy decoding ($T=0$) with draft length $\gamma=5$. We record full-vocabulary logits from the expert model $\mathcal{M}_p$, the amateur model $\mathcal{M}_q$, and the expert-aligned drafter $E$ instantiated with EAGLE3. We compute $D_{KL}(\pi_p \| \pi_{CD})$ by constructing $\pi_{CD}$ from Eq.~\ref{eq:contrastive_target} at each position, and we compute $D_{KL}(\pi_e \| \pi_p)$ directly from drafter and expert logits.

For the supplementary bucketed regime analysis, we aggregate configurations across three model groups (\textsc{Llama-3}, \textsc{Llama-EFT}, and \textsc{Qwen3}), two datasets (GSM8K and MMLU), and four contrastive strengths ($\alpha \in \{0.1, 0.3, 0.5, 1.0\}$). We use the same shared evaluation convention as Section~\ref{sec:experiments}: 200 evaluation samples per dataset, greedy decoding ($T=0$), and draft length $\gamma=5$. Here $h$ denotes the current decoding history, and $x^\star=\arg\max_x \pi_{CD}(x\mid h)$ denotes the CD top-1 token at that position. The per-position contrastive-signal statistic is the expert--amateur absolute log-ratio at that token, $\left|\log \frac{\pi_p(x^\star\mid h)}{\pi_q(x^\star\mid h)}\right|$. Table~\ref{tab:signal_bucket_shares} is count-weighted: we sum the per-bucket position counts over all configurations and divide by the total count. The reported 81.1\% is therefore the sum of the first three buckets, $[0,0.25)$, $[0.25,0.5)$, and $[0.5,1)$.

\paragraph{Error Amplification in the Approximate Dual-Drafter.}
A simple error decomposition explains why Path~2, the Approximate Dual-Drafter, performs poorly. Let $\hat{\pi}_{CD}(x) \propto \pi_e(x)^{1+\alpha} \cdot \pi_f(x)^{-\alpha}$ be the combined distribution, where $\pi_f$ is the amateur-aligned EAGLE head approximating $\pi_q$. Define the log approximation errors $\epsilon_p(x) = \log(\pi_e(x)/\pi_p(x))$ and $\epsilon_q(x) = \log(\pi_f(x)/\pi_q(x))$.

\begin{lemma}[Reconstruction Error Decomposition]
\label{lem:reconstruction_error}
The reconstructed distribution $\hat{\pi}_{CD}$ is an exponential tilt of $\pi_{CD}$:
\[
\begin{aligned}
\hat{\pi}_{CD}(x) &= \frac{\pi_{CD}(x) \cdot e^{\Delta(x)}}{Z_{\Delta}}, \\
Z_{\Delta} &= \mathbb{E}_{\pi_{CD}}[e^{\Delta(X)}],
\end{aligned}
\]
where the amplified combined error is $\Delta(x) = (1{+}\alpha)\,\epsilon_p(x) - \alpha\,\epsilon_q(x)$.
\end{lemma}

\begin{proof}
Direct expansion:
\begin{align*}
\pi_e(x)^{1+\alpha} \cdot \pi_f(x)^{-\alpha}
&= \bigl[\pi_p(x)\,e^{\epsilon_p(x)}\bigr]^{1+\alpha} \\
&\quad \cdot \bigl[\pi_q(x)\,e^{\epsilon_q(x)}\bigr]^{-\alpha} \\
&= \pi_p(x)^{1+\alpha}\,\pi_q(x)^{-\alpha} \\
&\quad \cdot e^{(1+\alpha)\epsilon_p(x) - \alpha\,\epsilon_q(x)} \\
&= Z_{CD}\,\pi_{CD}(x) \cdot e^{\Delta(x)}.
\end{align*}
Normalizing yields the lemma.
\end{proof}

The lemma makes the amplification mechanism explicit: the expert EAGLE log-error $\epsilon_p$ is scaled by $(1{+}\alpha)$ and the amateur EAGLE error $\epsilon_q$ by $\alpha$. Even when the two heads have similar approximation quality ($D_{KL}(\pi_e \| \pi_p) \approx 2.6$ vs.\ $D_{KL}(\pi_f \| \pi_q) \approx 2.4$ on Llama-3), the combined error $\Delta$ grows before normalization. Consequently, $D_{KL}(\hat{\pi}_{CD} \| \pi_{CD})$ becomes sensitive to the error distribution.

\paragraph{Per-Position Verification.}
We also measure the fraction of autoregressive positions at which the Approximate Dual-Drafter reconstruction $\hat{\pi}_{CD}$ is closer to $\pi_{CD}$ than the expert-aligned draft $\pi_e$, that is, where $D_{KL}(\hat{\pi}_{CD} \| \pi_{CD}) < D_{KL}(\pi_e \| \pi_{CD})$. On Llama-3 (GSM8K), this holds for 39.7\% of positions at $\alpha=0.1$ and 33.1\% at $\alpha=0.5$; Qwen3 shows 39.2\% and 31.9\%, respectively. This position-level check is consistent with the underperformance of Approximate Dual-Drafter in the tested regime.

Table~\ref{tab:dual_drafting} reports the corresponding aggregate matched-control results over accepted length, draft-to-target KL, and Top-1 Match.

\begin{table}[t]
\centering
\renewcommand{\arraystretch}{1.10}
\setlength{\tabcolsep}{4pt}
\resizebox{\columnwidth}{!}{%
\begin{tabular}{clcccc}
\toprule
\multirow{2}{*}{\textbf{$\alpha$}} & \multirow{2}{*}{\textbf{Method}} & \multicolumn{2}{c}{\textbf{Mean Accepted Length} $L$} & \multirow{2}{*}{\makecell{\textbf{Avg.} \\ \textbf{$D_{KL}(\cdot\|\pi_{CD})$}}} & \multirow{2}{*}{\makecell{\textbf{Avg. Top-1} \\ \textbf{Match (\%)}}} \\
\cmidrule(lr){3-4}
 &  & \textbf{GSM8K} & \textbf{MMLU} &  &  \\
\midrule
\multicolumn{6}{c}{\textit{\textbf{Llama-3} (8B-Ins / 1B-Ins)}} \\
\midrule
\multirow{2}{*}{0.1} & Expert-aligned & \textbf{1.966} & \textbf{1.476} & \textbf{2.98} & \textbf{67.7} \\
 & Approx.\ Dual & 1.945 & 1.458 & 3.13 & 66.7 \\
\addlinespace[2pt]
\multirow{2}{*}{0.3} & Expert-aligned & \textbf{1.932} & \textbf{1.424} & \textbf{3.26} & \textbf{66.1} \\
 & Approx.\ Dual & 1.853 & 1.357 & 3.84 & 62.9 \\
\addlinespace[2pt]
\multirow{2}{*}{0.5} & Expert-aligned & \textbf{1.863} & \textbf{1.355} & \textbf{3.54} & \textbf{64.4} \\
 & Approx.\ Dual & 1.711 & 1.241 & 4.74 & 58.6 \\
\midrule
\multicolumn{6}{c}{\textit{\textbf{Qwen3} (8B / 1.7B)}} \\
\midrule
\multirow{2}{*}{0.1} & Expert-aligned & \textbf{1.736} & \textbf{1.122} & \textbf{6.45} & \textbf{60.6} \\
 & Approx.\ Dual & 1.717 & 1.104 & 6.69 & 59.7 \\
\addlinespace[2pt]
\multirow{2}{*}{0.3} & Expert-aligned & \textbf{1.656} & \textbf{1.027} & \textbf{6.55} & \textbf{58.1} \\
 & Approx.\ Dual & 1.593 & 0.973 & 7.39 & 55.2 \\
\addlinespace[2pt]
\multirow{2}{*}{0.5} & Expert-aligned & \textbf{1.544} & \textbf{0.904} & \textbf{6.69} & \textbf{54.7} \\
 & Approx.\ Dual & 1.420 & 0.820 & 8.31 & 49.3 \\
\bottomrule
\end{tabular}
}
\caption{Approximate dual-drafting vs.\ expert-aligned DCD$_{\textsc{eagle3}}$. GSM8K/MMLU columns report dataset-specific mean accepted length $L$. The last two columns average over GSM8K and MMLU: $D_{KL}(\cdot\|\pi_{CD})$ is the draft-to-target KL, and Top-1 Match is the argmax match rate with $\pi_{CD}$. The CD reconstruction route fails on all three diagnostics: it has lower $L$, higher draft-to-target KL, and lower Top-1 Match, with larger gaps as $\alpha$ increases.}
\label{tab:dual_drafting}
\vspace{-4pt}
\end{table}

\paragraph{Supplementary Aggregated Regime Analysis Beyond the Matched Control.}
Beyond the matched control above, we also aggregate configurations spanning three model groups (\textsc{Llama-3}, \textsc{Llama-EFT}, and \textsc{Qwen3}), two datasets (GSM8K and MMLU), and four contrastive strengths ($\alpha \in \{0.1, 0.3, 0.5, 1.0\}$). For each configuration, we record the Top-1 gap $\Delta$Top-1 $=$ Top-1(contrastive-aligned) $-$ Top-1(expert-aligned), where Top-1 denotes the fraction of positions whose argmax matches $\arg\max_x \pi_{CD}(x\mid h)$, across bucketed views of signal, $D_{KL}(\pi_e \| \pi_p)$, and $D_{KL}(\pi_f \| \pi_q)$. The main text now summarizes the signal distribution, the signal $\times$ $D_{KL}(\pi_e \| \pi_p)$ effect view, and the marginal KL$(\pi_e \| \pi_p)$ distribution; this appendix adds the complementary views from that broader diagnostic.

Table~\ref{tab:signal_bucket_shares} gives the detailed count-weighted signal-mass breakdown behind the main-text distribution panel. In total, 81.1\% of positions fall below signal strength 1.0, while 18.9\% lie at or above 1.0.

\begin{table}[t]
\centering
\small
\renewcommand{\arraystretch}{1.12}
\setlength{\tabcolsep}{4pt}
\begin{tabular*}{0.92\columnwidth}{@{\extracolsep{\fill}}lcc}
\toprule
\textbf{Signal bucket} & \textbf{Share} & \textbf{Cumulative} \\
\midrule
$[0,0.25)$ & 64.4\% & 64.4\% \\
$[0.25,0.5)$ & 8.4\% & 72.8\% \\
$[0.5,1)$ & 8.3\% & 81.1\% \\
$[1,2)$ & 7.0\% & 88.1\% \\
$[2,4)$ & 5.2\% & 93.3\% \\
$[4,8)$ & 3.6\% & 96.9\% \\
$\geq 8$ & 3.1\% & 100.0\% \\
\midrule
$< 1.0$ & 81.1\% & -- \\
$\geq 1.0$ & 18.9\% & -- \\
\bottomrule
\end{tabular*}
\caption{Count-weighted share of autoregressive positions by contrastive-signal bucket, aggregated over 24 configurations spanning three model groups, two datasets, and four contrastive strengths.}
\label{tab:signal_bucket_shares}
\vspace{-4pt}
\end{table}

Main-text Figure~\ref{fig:signal_kl_ep_heatmap} summarizes the count-weighted signal distribution, the marginal KL$(\pi_e \| \pi_p)$ distribution, and the signal $\times$ $D_{KL}(\pi_e \| \pi_p)$ mechanism view. The remaining appendix figures add complementary aggregated views.

Figure~\ref{fig:kl_f_q_heatmap} provides the complementary configuration-level view over KL$(\pi_f\,\|\,\pi_q)$ buckets. Along this axis, the values vary more by configuration and do not form a stable monotonic pattern.

\begin{figure*}[t]
\centering
\begin{tikzpicture}
\begin{axis}[
    x=0.65cm,
    y=0.65cm,
    xmin=-0.5, xmax=10.5,
    ymin=-0.5, ymax=23.5,
    y dir=reverse,
    xtick={0,1,2,3,4,5,6,7,8,9,10},
    ytick={0,1,2,3,4,5,6,7,8,9,10,11,12,13,14,15,16,17,18,19,20,21,22,23},
    xticklabels={{$[0,0.01)$},{$[0.01,0.05)$},{$[0.05,0.1)$},{$[0.1,0.2)$},{$[0.2,0.5)$},{$[0.5,1)$},{$[1,2)$},{$[2,4)$},{$[4,6)$},{$[6,8)$},{$[8,\infty)$}},
    yticklabels={$a=0.1$,$a=0.3$,$a=0.5$,$a=1.0$,$a=0.1$,$a=0.3$,$a=0.5$,$a=1.0$,$a=0.1$,$a=0.3$,$a=0.5$,$a=1.0$,$a=0.1$,$a=0.3$,$a=0.5$,$a=1.0$,$a=0.1$,$a=0.3$,$a=0.5$,$a=1.0$,$a=0.1$,$a=0.3$,$a=0.5$,$a=1.0$},
    xticklabel style={font=\scriptsize, rotate=26, anchor=east, xshift=10pt, yshift=-10pt},
    yticklabel style={font=\scriptsize, xshift=-5pt},
    tick style={draw=none},
    axis line style={draw=black!35, line width=0.35pt},
    xlabel={KL$(\pi_f\,\|\,\pi_q)$},
    xlabel style={font=\small},
    enlargelimits=false,
    scale only axis,
    axis on top,
    clip=false,
    point meta min=-0.30,
    point meta max=0.30,
    colormap={paperdiv}{
        rgb255(0cm)=(59,76,192);
        rgb255(0.18cm)=(98,130,234);
        rgb255(0.36cm)=(190,210,245);
        rgb255(0.50cm)=(250,240,215);
        rgb255(0.64cm)=(245,193,167);
        rgb255(0.82cm)=(220,120,98);
        rgb255(1cm)=(180,4,38)
    },
    colorbar,
    colorbar style={
        title={$\Delta$Top-1},
        title style={font=\scriptsize, yshift=1pt},
        ytick={-0.30,-0.15,0.00,0.15,0.30},
        yticklabels={{-0.3},{-0.15},{0},{0.15},{0.3}},
        tick label style={font=\scriptsize},
        width=0.25cm,
        major tick length=1.5pt,
        axis line style={draw=black!30, line width=0.3pt}
    }
]
\addplot[
    matrix plot*,
    mesh/cols=11,
    point meta=explicit,
    draw=black!8,
    line width=0.12pt,
] coordinates {
    (0,0) [-0.0080] (1,0) [0.0000] (2,0) [-0.0090] (3,0) [-0.0220] (4,0) [-0.0130] (5,0) [-0.0200] (6,0) [-0.0160] (7,0) [-0.0070] (8,0) [0.0050] (9,0) [-0.0060] (10,0) [-0.0020] (0,1) [-0.0110] (1,1) [-0.0120] (2,1) [-0.0470] (3,1) [-0.0400] (4,1) [-0.0420] (5,1) [-0.0390] (6,1) [-0.0330] (7,1) [-0.0330] (8,1) [-0.0120] (9,1) [-0.0320] (10,1) [-0.0110] (0,2) [-0.0280] (1,2) [-0.0390] (2,2) [-0.0640] (3,2) [-0.0940] (4,2) [-0.0860] (5,2) [-0.0870] (6,2) [-0.0540] (7,2) [-0.0530] (8,2) [-0.0390] (9,2) [-0.0480] (10,2) [-0.0180] (0,3) [-0.1540] (1,3) [-0.2010] (2,3) [-0.1980] (3,3) [-0.2400] (4,3) [-0.2150] (5,3) [-0.1730] (6,3) [-0.1390] (7,3) [-0.1170] (8,3) [-0.0960] (9,3) [-0.0920] (10,3) [-0.0500] (0,4) [-0.0040] (1,4) [-0.0030] (2,4) [-0.0070] (3,4) [-0.0150] (4,4) [-0.0160] (5,4) [-0.0070] (6,4) [-0.0100] (7,4) [-0.0060] (8,4) [-0.0050] (9,4) [-0.0010] (10,4) [-0.0090] (0,5) [-0.0060] (1,5) [-0.0270] (2,5) [-0.0350] (3,5) [-0.0420] (4,5) [-0.0420] (5,5) [-0.0310] (6,5) [-0.0260] (7,5) [-0.0200] (8,5) [-0.0240] (9,5) [-0.0070] (10,5) [-0.0200] (0,6) [-0.0370] (1,6) [-0.0440] (2,6) [-0.0590] (3,6) [-0.0710] (4,6) [-0.0720] (5,6) [-0.0630] (6,6) [-0.0510] (7,6) [-0.0450] (8,6) [-0.0200] (9,6) [-0.0460] (10,6) [-0.0300] (0,7) [-0.1390] (1,7) [-0.1750] (2,7) [-0.1950] (3,7) [-0.1900] (4,7) [-0.1840] (5,7) [-0.1480] (6,7) [-0.1140] (7,7) [-0.0830] (8,7) [-0.0730] (9,7) [-0.0460] (10,7) [-0.0680] (0,8) [-0.0080] (1,8) [-0.0100] (2,8) [-0.0150] (3,8) [-0.0190] (4,8) [-0.0110] (5,8) [-0.0060] (6,8) [-0.0170] (7,8) [-0.0160] (8,8) [-0.0100] (9,8) [-0.0030] (10,8) [0.0060] (0,9) [-0.0120] (1,9) [-0.0240] (2,9) [-0.0320] (3,9) [-0.0390] (4,9) [-0.0520] (5,9) [-0.0430] (6,9) [-0.0320] (7,9) [-0.0240] (8,9) [-0.0250] (9,9) [-0.0220] (10,9) [-0.0030] (0,10) [-0.0410] (1,10) [-0.0560] (2,10) [-0.0760] (3,10) [-0.0900] (4,10) [-0.0920] (5,10) [-0.0690] (6,10) [-0.0610] (7,10) [-0.0450] (8,10) [-0.0440] (9,10) [-0.0460] (10,10) [-0.0330] (0,11) [-0.2020] (1,11) [-0.1800] (2,11) [-0.1980] (3,11) [-0.2280] (4,11) [-0.2020] (5,11) [-0.1480] (6,11) [-0.1280] (7,11) [-0.1170] (8,11) [-0.1100] (9,11) [-0.0910] (10,11) [-0.0950] (0,12) [-0.0030] (1,12) [-0.0040] (2,12) [-0.0070] (3,12) [-0.0170] (4,12) [-0.0070] (5,12) [-0.0140] (6,12) [-0.0120] (7,12) [-0.0090] (8,12) [-0.0080] (9,12) [-0.0090] (10,12) [-0.0020] (0,13) [-0.0210] (1,13) [-0.0250] (2,13) [-0.0240] (3,13) [-0.0290] (4,13) [-0.0310] (5,13) [-0.0510] (6,13) [-0.0360] (7,13) [-0.0360] (8,13) [-0.0220] (9,13) [-0.0040] (10,13) [-0.0310] (0,14) [-0.0630] (1,14) [-0.0730] (2,14) [-0.0750] (3,14) [-0.0800] (4,14) [-0.0530] (5,14) [-0.0600] (6,14) [-0.0590] (7,14) [-0.0450] (8,14) [-0.0370] (9,14) [-0.0370] (10,14) [-0.0280] (0,15) [-0.1820] (1,15) [-0.2510] (2,15) [-0.1940] (3,15) [-0.1910] (4,15) [-0.1660] (5,15) [-0.1470] (6,15) [-0.1350] (7,15) [-0.1160] (8,15) [-0.0950] (9,15) [-0.0990] (10,15) [-0.0660] (0,16) [0.0000] (1,16) [-0.0080] (2,16) [-0.0160] (3,16) [-0.0110] (4,16) [-0.0080] (5,16) [-0.0170] (6,16) [-0.0240] (7,16) [-0.0160] (8,16) [-0.0120] (9,16) [-0.0110] (10,16) [-0.0030] (0,17) [-0.0260] (1,17) [-0.0360] (2,17) [-0.0340] (3,17) [-0.0320] (4,17) [-0.0390] (5,17) [-0.0410] (6,17) [-0.0380] (7,17) [-0.0660] (8,17) [-0.0440] (9,17) [-0.0290] (10,17) [-0.0160] (0,18) [-0.0850] (1,18) [-0.0950] (2,18) [-0.1080] (3,18) [-0.0690] (4,18) [-0.0570] (5,18) [-0.0920] (6,18) [-0.0940] (7,18) [-0.0840] (8,18) [-0.0950] (9,18) [-0.0520] (10,18) [-0.0310] (0,19) [-0.1500] (1,19) [-0.2190] (2,19) [-0.2420] (3,19) [-0.1870] (4,19) [-0.1960] (5,19) [-0.1350] (6,19) [-0.1210] (7,19) [-0.0830] (8,19) [-0.0670] (9,19) [-0.0590] (10,19) [-0.0320] (0,20) [-0.0090] (1,20) [0.0000] (2,20) [0.0000] (3,20) [-0.0030] (4,20) [-0.0100] (5,20) [-0.0140] (6,20) [-0.0140] (7,20) [-0.0080] (8,20) [-0.0110] (9,20) [-0.0150] (10,20) [-0.0040] (0,21) [0.0000] (1,21) [-0.0250] (2,21) [-0.0180] (3,21) [-0.0190] (4,21) [-0.0190] (5,21) [-0.0410] (6,21) [-0.0290] (7,21) [-0.0200] (8,21) [-0.0460] (9,21) [-0.0280] (10,21) [-0.0200] (0,22) [-0.0090] (1,22) [-0.0280] (2,22) [-0.0520] (3,22) [-0.0430] (4,22) [-0.0570] (5,22) [-0.0790] (6,22) [-0.0760] (7,22) [-0.0560] (8,22) [-0.0500] (9,22) [-0.0450] (10,22) [-0.0280] (0,23) [-0.1870] (1,23) [-0.1720] (2,23) [-0.1120] (3,23) [-0.1590] (4,23) [-0.1650] (5,23) [-0.1240] (6,23) [-0.1200] (7,23) [-0.1130] (8,23) [-0.0660] (9,23) [-0.0520] (10,23) [-0.0270]
};

\draw[black!60, line width=1pt] (axis cs:-0.5,3.5) -- (axis cs:10.5,3.5);
\draw[black!60, line width=1pt] (axis cs:-0.5,7.5) -- (axis cs:10.5,7.5);
\draw[black!60, line width=1pt] (axis cs:-0.5,11.5) -- (axis cs:10.5,11.5);
\draw[black!60, line width=1pt] (axis cs:-0.5,15.5) -- (axis cs:10.5,15.5);
\draw[black!60, line width=1pt] (axis cs:-0.5,19.5) -- (axis cs:10.5,19.5);

    \node[anchor=east, font=\small\bfseries, align=right] at (axis cs:-2.5,1.5) {Llama-3 / GSM8K};
    \node[anchor=east, font=\small\bfseries, align=right] at (axis cs:-2.5,5.5) {Llama-3 / MMLU};
    \node[anchor=east, font=\small\bfseries, align=right] at (axis cs:-2.5,9.5) {Llama-EFT / GSM8K};
    \node[anchor=east, font=\small\bfseries, align=right] at (axis cs:-2.5,13.5) {Llama-EFT / MMLU};
    \node[anchor=east, font=\small\bfseries, align=right] at (axis cs:-2.5,17.5) {Qwen3 / GSM8K};
    \node[anchor=east, font=\small\bfseries, align=right] at (axis cs:-2.5,21.5) {Qwen3 / MMLU};

\end{axis}
\end{tikzpicture}
\caption[KL$(\pi_f\,\|\,\pi_q)$ configuration heatmap]{Configuration-level $\Delta$Top-1 across KL$(\pi_f\,\|\,\pi_q)$ buckets. The association is heterogeneous across model groups and contrastive strengths: for $\alpha\leq0.5$, KL-axis patterns are weak or mixed, whereas at $\alpha=1.0$ the gap often becomes less negative at higher KL. No configuration-invariant positive-gain regime emerges.}
\label{fig:kl_f_q_heatmap}
\end{figure*}
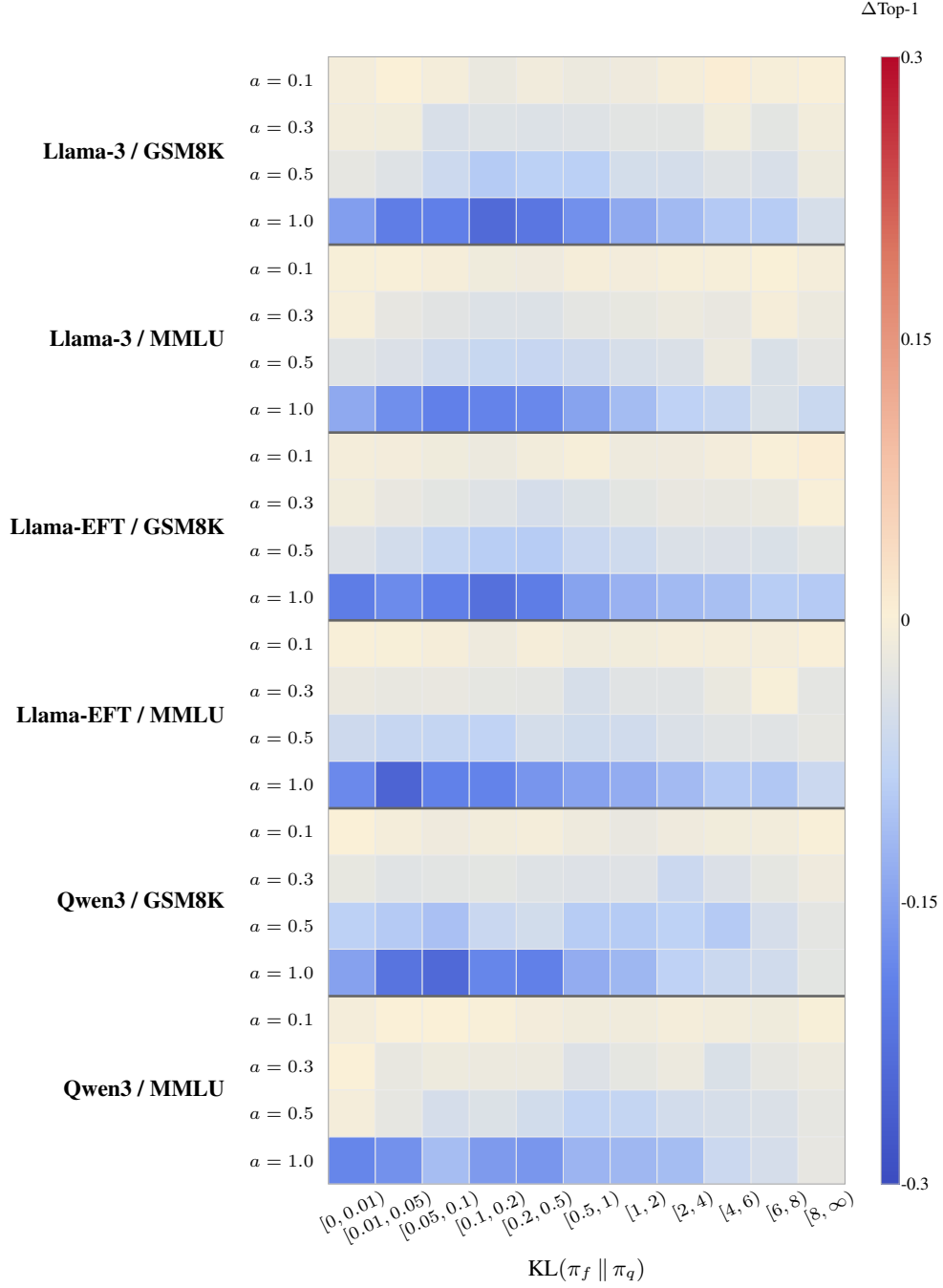

Figure~\ref{fig:exp1_regime_heatmap} provides a compact signal-centric aggregation over mutually exclusive signal buckets and shows that the low-signal failure pattern persists in most settings.

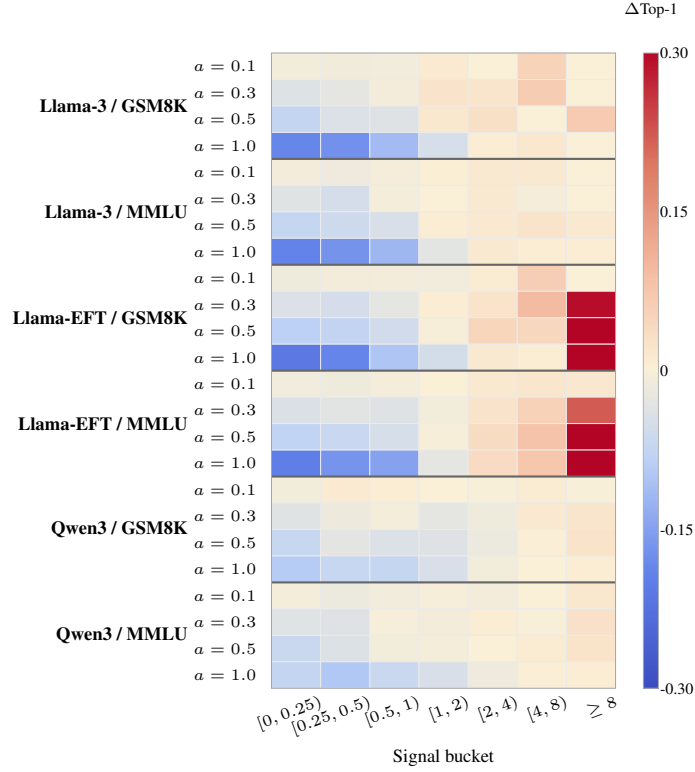
\begin{figure*}[t]
\centering
\begin{tikzpicture}
\begin{axis}[
    x=0.65cm,
    y=0.35cm,
    xmin=-0.5, xmax=6.5,
    ymin=-0.5, ymax=23.5,
    y dir=reverse,
    xtick={0,1,2,3,4,5,6},
    ytick={0,1,2,3,4,5,6,7,8,9,10,11,12,13,14,15,16,17,18,19,20,21,22,23},
    xticklabels={{$[0,0.25)$},{$[0.25,0.5)$},{$[0.5,1)$},{$[1,2)$},{$[2,4)$},{$[4,8)$},{$\geq 8$}},
    yticklabels={$a=0.1$,$a=0.3$,$a=0.5$,$a=1.0$,$a=0.1$,$a=0.3$,$a=0.5$,$a=1.0$,$a=0.1$,$a=0.3$,$a=0.5$,$a=1.0$,$a=0.1$,$a=0.3$,$a=0.5$,$a=1.0$,$a=0.1$,$a=0.3$,$a=0.5$,$a=1.0$,$a=0.1$,$a=0.3$,$a=0.5$,$a=1.0$},
    xticklabel style={font=\tiny, rotate=20, anchor=east, xshift=8pt, yshift=-8pt},
    yticklabel style={font=\tiny, xshift=-5pt},
    tick style={draw=none},
    axis line style={draw=black!35, line width=0.35pt},
    xlabel={Signal bucket},
    xlabel style={font=\scriptsize},
    enlargelimits=false,
    scale only axis,
    axis on top,
    clip=false,
    point meta min=-0.30,
    point meta max=0.30,
    colormap={paperdiv}{
        rgb255(0cm)=(59,76,192);
        rgb255(0.18cm)=(98,130,234);
        rgb255(0.36cm)=(190,210,245);
        rgb255(0.50cm)=(250,240,215);
        rgb255(0.64cm)=(245,193,167);
        rgb255(0.82cm)=(220,120,98);
        rgb255(1cm)=(180,4,38)
    },
    colorbar,
    colorbar style={
        title={$\Delta$Top-1},
        title style={font=\tiny, yshift=1pt},
        ytick={-0.30,-0.15,0.00,0.15,0.30},
        yticklabels={{-0.30},{-0.15},{0},{0.15},{0.30}},
        tick label style={font=\tiny},
        width=0.20cm,
        major tick length=1.2pt,
        axis line style={draw=black!30, line width=0.3pt}
    }
]
\addplot[
    matrix plot*,
    mesh/cols=7,
    point meta=explicit,
    draw=black!8,
    line width=0.12pt,
] coordinates {
    (0,0) [-0.0110] (1,0) [-0.0130] (2,0) [-0.0120] (3,0) [0.0110] (4,0) [0.0000] (5,0) [0.0520] (6,0) [0.0000] (0,1) [-0.0400] (1,1) [-0.0290] (2,1) [-0.0100] (3,1) [0.0230] (4,1) [0.0200] (5,1) [0.0640] (6,1) [0.0000] (0,2) [-0.0750] (1,2) [-0.0430] (2,2) [-0.0380] (3,2) [0.0160] (4,2) [0.0280] (5,2) [0.0000] (6,2) [0.0670] (0,3) [-0.1870] (1,3) [-0.1730] (2,3) [-0.1160] (3,3) [-0.0500] (4,3) [0.0060] (5,3) [0.0160] (6,3) [0.0000] (0,4) [-0.0100] (1,4) [-0.0160] (2,4) [-0.0070] (3,4) [0.0030] (4,4) [0.0120] (5,4) [0.0130] (6,4) [0.0000] (0,5) [-0.0360] (1,5) [-0.0520] (2,5) [-0.0090] (3,5) [0.0000] (4,5) [0.0120] (5,5) [-0.0070] (6,5) [0.0000] (0,6) [-0.0740] (1,6) [-0.0630] (2,6) [-0.0470] (3,6) [0.0050] (4,6) [0.0130] (5,6) [0.0210] (6,6) [0.0130] (0,7) [-0.1910] (1,7) [-0.1690] (2,7) [-0.1210] (3,7) [-0.0320] (4,7) [0.0120] (5,7) [0.0060] (6,7) [0.0050] (0,8) [-0.0150] (1,8) [-0.0100] (2,8) [-0.0100] (3,8) [-0.0120] (4,8) [0.0080] (5,8) [0.0600] (6,8) [0.0000] (0,9) [-0.0430] (1,9) [-0.0520] (2,9) [-0.0290] (3,9) [0.0060] (4,9) [0.0210] (5,9) [0.0930] (6,9) [0.2920] (0,10) [-0.0860] (1,10) [-0.0770] (2,10) [-0.0580] (3,10) [-0.0060] (4,10) [0.0470] (5,10) [0.0420] (6,10) [0.4870] (0,11) [-0.2090] (1,11) [-0.1880] (2,11) [-0.1020] (3,11) [-0.0540] (4,11) [0.0120] (5,11) [0.0050] (6,11) [0.7850] (0,12) [-0.0120] (1,12) [-0.0150] (2,12) [-0.0080] (3,12) [0.0000] (4,12) [0.0120] (5,12) [0.0180] (6,12) [0.0140] (0,13) [-0.0440] (1,13) [-0.0330] (2,13) [-0.0380] (3,13) [-0.0110] (4,13) [0.0230] (5,13) [0.0550] (6,13) [0.2190] (0,14) [-0.0790] (1,14) [-0.0690] (2,14) [-0.0500] (3,14) [-0.0060] (4,14) [0.0360] (5,14) [0.0800] (6,14) [0.7190] (0,15) [-0.2000] (1,15) [-0.1680] (2,15) [-0.1500] (3,15) [-0.0300] (4,15) [0.0370] (5,15) [0.0730] (6,15) [0.8200] (0,16) [-0.0110] (1,16) [0.0100] (2,16) [0.0050] (3,16) [0.0000] (4,16) [-0.0030] (5,16) [0.0070] (6,16) [-0.0040] (0,17) [-0.0360] (1,17) [-0.0180] (2,17) [-0.0080] (3,17) [-0.0300] (4,17) [-0.0150] (5,17) [0.0120] (6,17) [0.0200] (0,18) [-0.0710] (1,18) [-0.0310] (2,18) [-0.0430] (3,18) [-0.0390] (4,18) [-0.0200] (5,18) [0.0020] (6,18) [0.0220] (0,19) [-0.0920] (1,19) [-0.0720] (2,19) [-0.0740] (3,19) [-0.0480] (4,19) [-0.0110] (5,19) [-0.0010] (6,19) [0.0050] (0,20) [-0.0090] (1,20) [-0.0190] (2,20) [-0.0120] (3,20) [-0.0050] (4,20) [-0.0100] (5,20) [0.0000] (6,20) [0.0150] (0,21) [-0.0360] (1,21) [-0.0390] (2,21) [-0.0050] (3,21) [-0.0100] (4,21) [0.0050] (5,21) [-0.0030] (6,21) [0.0270] (0,22) [-0.0670] (1,22) [-0.0420] (2,22) [-0.0110] (3,22) [-0.0090] (4,22) [-0.0010] (5,22) [0.0080] (6,22) [0.0220] (0,23) [-0.0750] (1,23) [-0.0980] (2,23) [-0.0680] (3,23) [-0.0480] (4,23) [-0.0160] (5,23) [0.0010] (6,23) [0.0050]
};

\draw[black!60, line width=0.8pt] (axis cs:-0.5,3.5) -- (axis cs:6.5,3.5);
\draw[black!60, line width=0.8pt] (axis cs:-0.5,7.5) -- (axis cs:6.5,7.5);
\draw[black!60, line width=0.8pt] (axis cs:-0.5,11.5) -- (axis cs:6.5,11.5);
\draw[black!60, line width=0.8pt] (axis cs:-0.5,15.5) -- (axis cs:6.5,15.5);
\draw[black!60, line width=0.8pt] (axis cs:-0.5,19.5) -- (axis cs:6.5,19.5);

    \node[anchor=east, font=\scriptsize\bfseries, align=right] at (axis cs:-2,1.5) {Llama-3 / GSM8K};
    \node[anchor=east, font=\scriptsize\bfseries, align=right] at (axis cs:-2,5.5) {Llama-3 / MMLU};
    \node[anchor=east, font=\scriptsize\bfseries, align=right] at (axis cs:-2,9.5) {Llama-EFT / GSM8K};
    \node[anchor=east, font=\scriptsize\bfseries, align=right] at (axis cs:-2,13.5) {Llama-EFT / MMLU};
    \node[anchor=east, font=\scriptsize\bfseries, align=right] at (axis cs:-2,17.5) {Qwen3 / GSM8K};
    \node[anchor=east, font=\scriptsize\bfseries, align=right] at (axis cs:-2,21.5) {Qwen3 / MMLU};

\end{axis}
\end{tikzpicture}
    \caption[Contrastive-signal regime heatmap]{Compact regime heatmap for the expanded root-level contrastive-signal bucket analysis. Here contrastive signal denotes the per-position expert--amateur absolute log-ratio $\left|\log \frac{\pi_p(x\mid h)}{\pi_q(x\mid h)}\right|$. Rows are grouped by model family and dataset, with per-row tick labels indicating $\alpha$, and columns are mutually exclusive contrastive-signal buckets. Low-signal buckets are mostly negative, while several higher-signal buckets are positive. Warm colors favor contrastive-aligned drafting; cool colors favor expert-aligned drafting. To preserve visualization contrast for moderate variations, the color scale is truncated at $\pm 0.3$.}
\label{fig:exp1_regime_heatmap}
\end{figure*}

Figure~\ref{fig:exp1_alpha_trends} breaks this pattern down by model group and dataset. Larger $\alpha$ usually pushes the low-signal buckets further negative, while some higher-signal buckets improve when the contrastive signal is informative enough.

\begin{figure*}[t]
\centering
\definecolor{alphaone}{RGB}{158,202,225}
\definecolor{alphathree}{RGB}{107,174,214}
\definecolor{alphafive}{RGB}{128,125,186}
\definecolor{alphaonezero}{RGB}{165,15,21}

\begin{subfigure}[t]{0.32\textwidth}
\centering
\begin{tikzpicture}
\begin{axis}[
    width=0.80\linewidth,
    ylabel={$\Delta$Top-1},
    height=4.2cm,
    title={\textsc{Llama-3} | GSM8K},
    title style={font=\scriptsize\bfseries, yshift=-1ex},
    xlabel style={font=\scriptsize},
    ymin=-0.24, ymax=0.32,
    xtick={0,1,2,3,4,5},
    xticklabels={{$[0,0.25)$},{$[0.25,0.5)$},{$[0.5,1)$},{$[1,2)$},{$[2,4)$},{$[4,8)$}},
    xticklabel style={font=\tiny, rotate=45, anchor=east},
    ytick={-0.2,-0.1,0,0.1,0.2,0.3},
    yticklabel style={font=\scriptsize},
    grid=major,
    grid style={solid, gray!15},
    scale only axis,
]
\draw[black!70, solid, line width=0.8pt] (axis cs:-0.5,0) -- (axis cs:5.5,0);
\addplot[color=alphaone, mark=*, mark size=1.5pt, thick] coordinates { (0,-0.011) (1,-0.013) (2,-0.012) (3,0.011) (4,0.000) (5,0.052) };
\addplot[color=alphathree, mark=square*, mark size=1.5pt, thick] coordinates { (0,-0.040) (1,-0.029) (2,-0.010) (3,0.023) (4,0.020) (5,0.064) };
\addplot[color=alphafive, mark=triangle*, mark size=1.7pt, thick] coordinates { (0,-0.075) (1,-0.043) (2,-0.038) (3,0.016) (4,0.028) (5,0.000) };
\addplot[color=alphaonezero, mark=diamond*, mark size=1.8pt, thick] coordinates { (0,-0.187) (1,-0.173) (2,-0.116) (3,-0.050) (4,0.006) (5,0.016) };
\end{axis}
\end{tikzpicture}
\end{subfigure}
\hfill
\begin{subfigure}[t]{0.32\textwidth}
\centering
\begin{tikzpicture}
\begin{axis}[
    width=0.80\linewidth,
    height=4.2cm,
    title={\textsc{Llama-EFT} | GSM8K},
    title style={font=\scriptsize\bfseries, yshift=-1ex},
    xlabel style={font=\scriptsize},
    ymin=-0.24, ymax=0.32,
    xtick={0,1,2,3,4,5},
    xticklabels={{$[0,0.25)$},{$[0.25,0.5)$},{$[0.5,1)$},{$[1,2)$},{$[2,4)$},{$[4,8)$}},
    xticklabel style={font=\tiny, rotate=45, anchor=east},
    ytick={-0.2,-0.1,0,0.1,0.2,0.3},
    yticklabel style={font=\scriptsize},
    grid=major,
    grid style={solid, gray!15},
    scale only axis,
]
\draw[black!70, solid, line width=0.8pt] (axis cs:-0.5,0) -- (axis cs:5.5,0);
\addplot[color=alphaone, mark=*, mark size=1.5pt, thick] coordinates { (0,-0.015) (1,-0.010) (2,-0.010) (3,-0.012) (4,0.008) (5,0.060) };
\addplot[color=alphathree, mark=square*, mark size=1.5pt, thick] coordinates { (0,-0.043) (1,-0.052) (2,-0.029) (3,0.006) (4,0.021) (5,0.093) };
\addplot[color=alphafive, mark=triangle*, mark size=1.7pt, thick] coordinates { (0,-0.086) (1,-0.077) (2,-0.058) (3,-0.006) (4,0.047) (5,0.042) };
\addplot[color=alphaonezero, mark=diamond*, mark size=1.8pt, thick] coordinates { (0,-0.209) (1,-0.188) (2,-0.102) (3,-0.054) (4,0.012) (5,0.005) };
\end{axis}
\end{tikzpicture}
\end{subfigure}
\hfill
\begin{subfigure}[t]{0.32\textwidth}
\centering
\begin{tikzpicture}
\begin{axis}[
    width=0.80\linewidth,
    height=4.2cm,
    title={\textsc{Qwen3} | GSM8K},
    title style={font=\scriptsize\bfseries, yshift=-1ex},
    xlabel style={font=\scriptsize},
    ymin=-0.24, ymax=0.32,
    xtick={0,1,2,3,4,5},
    xticklabels={{$[0,0.25)$},{$[0.25,0.5)$},{$[0.5,1)$},{$[1,2)$},{$[2,4)$},{$[4,8)$}},
    xticklabel style={font=\tiny, rotate=45, anchor=east},
    ytick={-0.2,-0.1,0,0.1,0.2,0.3},
    yticklabel style={font=\scriptsize},
    grid=major,
    grid style={solid, gray!15},
    scale only axis,
]
\draw[black!70, solid, line width=0.8pt] (axis cs:-0.5,0) -- (axis cs:5.5,0);
\addplot[color=alphaone, mark=*, mark size=1.5pt, thick] coordinates { (0,-0.011) (1,0.010) (2,0.005) (3,0.000) (4,-0.003) (5,0.007) };
\addplot[color=alphathree, mark=square*, mark size=1.5pt, thick] coordinates { (0,-0.036) (1,-0.018) (2,-0.008) (3,-0.030) (4,-0.015) (5,0.012) };
\addplot[color=alphafive, mark=triangle*, mark size=1.7pt, thick] coordinates { (0,-0.071) (1,-0.031) (2,-0.043) (3,-0.039) (4,-0.020) (5,0.002) };
\addplot[color=alphaonezero, mark=diamond*, mark size=1.8pt, thick] coordinates { (0,-0.092) (1,-0.072) (2,-0.074) (3,-0.048) (4,-0.011) (5,-0.001) };
\end{axis}
\end{tikzpicture}
\end{subfigure}

\vspace{0.4em}

\begin{subfigure}[t]{0.32\textwidth}
\centering
\begin{tikzpicture}
\begin{axis}[
    width=0.80\linewidth,
    ylabel={$\Delta$Top-1},
    xlabel={Signal Bucket},
    height=4.2cm,
    title={\textsc{Llama-3} | MMLU},
    title style={font=\scriptsize\bfseries, yshift=-1ex},
    xlabel style={font=\scriptsize},
    ymin=-0.24, ymax=0.32,
    xtick={0,1,2,3,4,5},
    xticklabels={{$[0,0.25)$},{$[0.25,0.5)$},{$[0.5,1)$},{$[1,2)$},{$[2,4)$},{$[4,8)$}},
    xticklabel style={font=\tiny, rotate=45, anchor=east},
    ytick={-0.2,-0.1,0,0.1,0.2,0.3},
    yticklabel style={font=\scriptsize},
    grid=major,
    grid style={solid, gray!15},
    scale only axis,
]
\draw[black!70, solid, line width=0.8pt] (axis cs:-0.5,0) -- (axis cs:5.5,0);
\addplot[color=alphaone, mark=*, mark size=1.5pt, thick] coordinates { (0,-0.010) (1,-0.016) (2,-0.007) (3,0.003) (4,0.012) (5,0.013) };
\addplot[color=alphathree, mark=square*, mark size=1.5pt, thick] coordinates { (0,-0.036) (1,-0.052) (2,-0.009) (3,0.000) (4,0.012) (5,-0.007) };
\addplot[color=alphafive, mark=triangle*, mark size=1.7pt, thick] coordinates { (0,-0.074) (1,-0.063) (2,-0.047) (3,0.005) (4,0.013) (5,0.021) };
\addplot[color=alphaonezero, mark=diamond*, mark size=1.8pt, thick] coordinates { (0,-0.191) (1,-0.169) (2,-0.121) (3,-0.032) (4,0.012) (5,0.006) };
\end{axis}
\end{tikzpicture}
\end{subfigure}
\hfill
\begin{subfigure}[t]{0.32\textwidth}
\centering
\begin{tikzpicture}
\begin{axis}[
    width=0.80\linewidth,
    xlabel={Signal Bucket},
    height=4.2cm,
    title={\textsc{Llama-EFT} | MMLU},
    title style={font=\scriptsize\bfseries, yshift=-1ex},
    xlabel style={font=\scriptsize},
    ymin=-0.24, ymax=0.32,
    xtick={0,1,2,3,4,5},
    xticklabels={{$[0,0.25)$},{$[0.25,0.5)$},{$[0.5,1)$},{$[1,2)$},{$[2,4)$},{$[4,8)$}},
    xticklabel style={font=\tiny, rotate=45, anchor=east},
    ytick={-0.2,-0.1,0,0.1,0.2,0.3},
    yticklabel style={font=\scriptsize},
    grid=major,
    grid style={solid, gray!15},
    scale only axis,
]
\draw[black!70, solid, line width=0.8pt] (axis cs:-0.5,0) -- (axis cs:5.5,0);
\addplot[color=alphaone, mark=*, mark size=1.5pt, thick] coordinates { (0,-0.012) (1,-0.015) (2,-0.008) (3,0.000) (4,0.012) (5,0.018) };
\addplot[color=alphathree, mark=square*, mark size=1.5pt, thick] coordinates { (0,-0.044) (1,-0.033) (2,-0.038) (3,-0.011) (4,0.023) (5,0.055) };
\addplot[color=alphafive, mark=triangle*, mark size=1.7pt, thick] coordinates { (0,-0.079) (1,-0.069) (2,-0.050) (3,-0.006) (4,0.036) (5,0.080) };
\addplot[color=alphaonezero, mark=diamond*, mark size=1.8pt, thick] coordinates { (0,-0.200) (1,-0.168) (2,-0.150) (3,-0.030) (4,0.037) (5,0.073) };
\end{axis}
\end{tikzpicture}
\end{subfigure}
\hfill
\begin{subfigure}[t]{0.32\textwidth}
\centering
\begin{tikzpicture}
\begin{axis}[
    width=0.80\linewidth,
    xlabel={Signal Bucket},
    height=4.2cm,
    title={\textsc{Qwen3} | MMLU},
    title style={font=\scriptsize\bfseries, yshift=-1ex},
    xlabel style={font=\scriptsize},
    ymin=-0.24, ymax=0.32,
    xtick={0,1,2,3,4,5},
    xticklabels={{$[0,0.25)$},{$[0.25,0.5)$},{$[0.5,1)$},{$[1,2)$},{$[2,4)$},{$[4,8)$}},
    xticklabel style={font=\tiny, rotate=45, anchor=east},
    ytick={-0.2,-0.1,0,0.1,0.2,0.3},
    yticklabel style={font=\scriptsize},
    grid=major,
    grid style={solid, gray!15},
    scale only axis,
]
\draw[black!70, solid, line width=0.8pt] (axis cs:-0.5,0) -- (axis cs:5.5,0);
\addplot[color=alphaone, mark=*, mark size=1.5pt, thick] coordinates { (0,-0.009) (1,-0.019) (2,-0.012) (3,-0.005) (4,-0.010) (5,0.000) };
\addplot[color=alphathree, mark=square*, mark size=1.5pt, thick] coordinates { (0,-0.036) (1,-0.039) (2,-0.005) (3,-0.010) (4,0.005) (5,-0.003) };
\addplot[color=alphafive, mark=triangle*, mark size=1.7pt, thick] coordinates { (0,-0.067) (1,-0.042) (2,-0.011) (3,-0.009) (4,-0.001) (5,0.008) };
\addplot[color=alphaonezero, mark=diamond*, mark size=1.8pt, thick] coordinates { (0,-0.075) (1,-0.098) (2,-0.068) (3,-0.048) (4,-0.016) (5,0.001) };
\end{axis}
\end{tikzpicture}
\end{subfigure}

\vspace{0.3em}
\begin{tikzpicture}
\begin{axis}[
    hide axis,
    xmin=0, xmax=1,
    ymin=0, ymax=1,
    legend style={
        at={(0.5,0.5)},
        anchor=center,
        legend columns=4,
        draw=none,
        fill=none,
        font=\small,
        /tikz/every even column/.append style={column sep=0.7em}
    }
]
\addlegendimage{color=alphaone, mark=*, mark size=1.5pt, thick}
\addlegendentry{$\alpha=0.1$}
\addlegendimage{color=alphathree, mark=square*, mark size=1.5pt, thick}
\addlegendentry{$\alpha=0.3$}
\addlegendimage{color=alphafive, mark=triangle*, mark size=1.7pt, thick}
\addlegendentry{$\alpha=0.5$}
\addlegendimage{color=alphaonezero, mark=diamond*, mark size=1.8pt, thick}
\addlegendentry{$\alpha=1.0$}
\end{axis}
\end{tikzpicture}

\caption[$\alpha$-trends by contrastive-signal bucket]{Per-group $\Delta$Top-1 trends across contrastive-signal buckets. Here contrastive signal denotes the per-position expert--amateur absolute log-ratio $\left|\log \frac{\pi_p(x\mid h)}{\pi_q(x\mid h)}\right|$. Each subplot fixes a model group and dataset, and each line corresponds to one contrastive strength $\alpha$. Larger $\alpha$ usually worsens low-signal buckets but can amplify gains in higher-signal buckets, especially for the Llama-EFT settings. For clarity of the moderate-signal trends, the extreme outlier bucket ($\geq 8$) is excluded from the horizontal axis.}
\label{fig:exp1_alpha_trends}
\end{figure*}
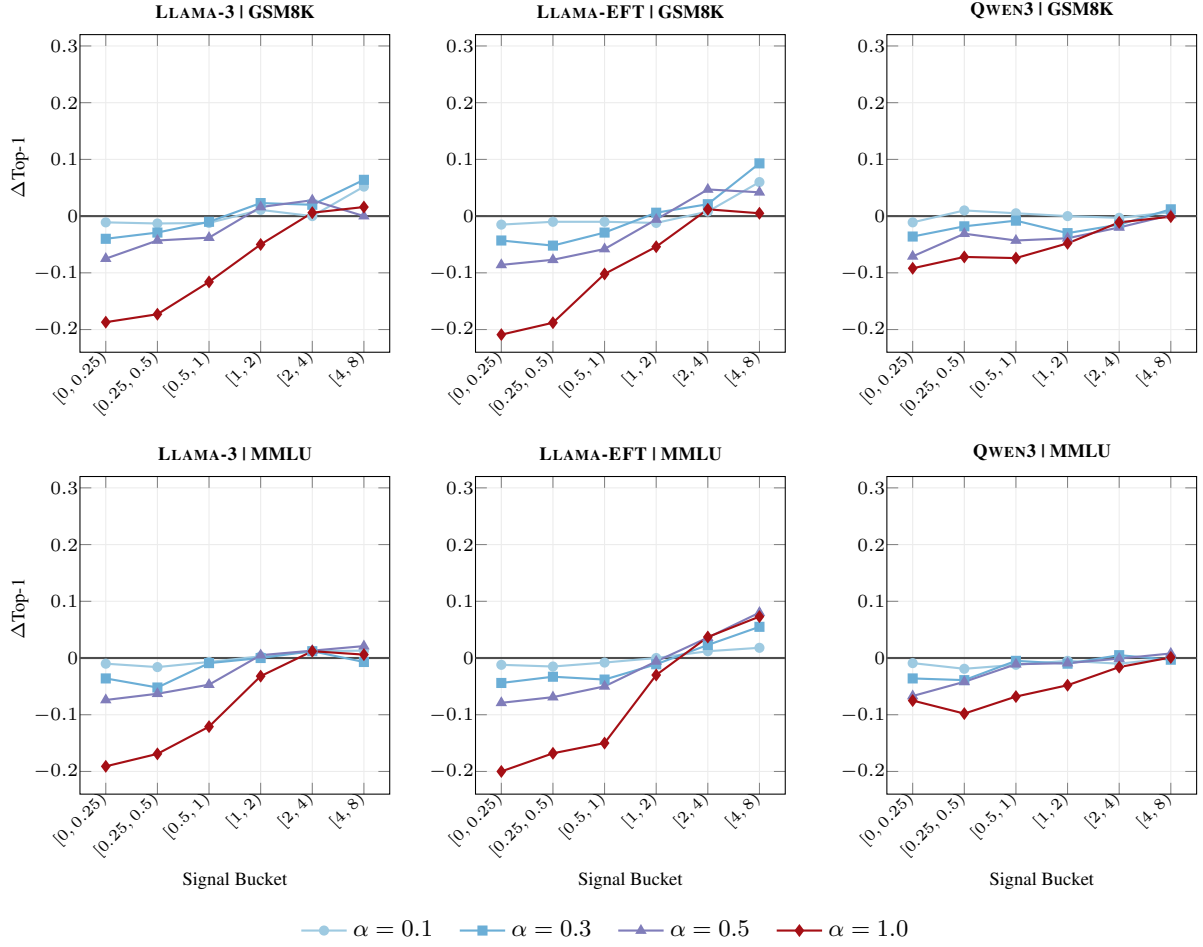

Across Figures~\ref{fig:signal_kl_ep_heatmap}, \ref{fig:kl_f_q_heatmap}, \ref{fig:exp1_regime_heatmap}, and \ref{fig:exp1_alpha_trends}, low-signal buckets remain mostly negative, the dependence on expert-side proposal error is strong, and the dependence on amateur-side reconstruction error is noisier.

\paragraph{Why Contrastive-Aligned Training Also Underperforms in Our Setting (Path~1).}
For Path~1, there is no inference-time reconstruction step, but the same scale mismatch remains: the lightweight EAGLE head must learn both the dominant expert distribution $\pi_p$ and the much smaller contrastive signal $\Delta \propto \pi_q^{-\alpha}$. This increases the modeling burden without reducing baseline proposal error. Within our CD-Aligned-EAGLE family, Table~\ref{tab:cross_alpha_full} therefore still favors expert-aligned training (train-$\alpha{=}0$) as the strongest default.

\subsection{Cross-alpha Experiment Details}
\label{sec:appendix_cross_alpha}
Table~\ref{tab:cross_alpha} in the main text summarizes the within-family controlled single-run ablation. Table~\ref{tab:cross_alpha_full} reports the full cross-$\alpha$ matrix for the Contrastive-Aligned Dual-Input Drafter in Section~\ref{sec:why_not_cd} (Path~1), where rows vary train-$\alpha$ and columns vary infer-$\alpha$. All numeric variants share the same dual-input architecture and ShareGPT training setup; only the target distribution differs. We also include the Official checkpoint from our main experiments, evaluated under the same inference protocol, as an external reference.

\begin{table*}[t]
\centering
\footnotesize
\renewcommand{\arraystretch}{1.15}
\setlength{\tabcolsep}{4pt}
\begin{tabular}{lcccc}
\toprule
\textbf{Train-$\alpha$} & \textbf{Infer-$\alpha{=}0.0$} & \textbf{Infer-$\alpha{=}0.1$} & \textbf{Infer-$\alpha{=}0.3$} & \textbf{Infer-$\alpha{=}0.5$} \\
\midrule
\multicolumn{5}{c}{\textit{\textbf{Llama-3} (8B-Ins / 1B-Ins)} — GSM8K} \\
\midrule
\textsc{Official checkpoint} & 1.555 & 1.537 & 1.546 & 1.476 \\
\rowcolor{yellow!15} 0.0 & 2.002 & 1.968 & 1.956 & 1.880 \\
0.1 & \textbf{2.008} & \textbf{1.988} & \textbf{1.977} & \textbf{1.892} \\
0.3 & 1.995 & 1.964 & 1.963 & 1.881 \\
0.5 & 1.936 & 1.915 & 1.896 & 1.833 \\
\midrule
\multicolumn{5}{c}{\textit{\textbf{Llama-3} (8B-Ins / 1B-Ins)} — MMLU} \\
\midrule
\textsc{Official checkpoint} & 1.523 & 1.428 & 1.344 & 1.214 \\
\rowcolor{yellow!15} 0.0 & 1.655 & \textbf{1.596} & \textbf{1.508} & \textbf{1.404} \\
0.1 & \textbf{1.663} & 1.576 & 1.503 & 1.370 \\
0.3 & 1.623 & 1.554 & 1.469 & 1.320 \\
0.5 & 1.618 & 1.552 & 1.466 & 1.328 \\
\midrule
\multicolumn{5}{c}{\textit{\textbf{Qwen3} (8B / 1.7B)} — GSM8K} \\
\midrule
\textsc{Official checkpoint} & \textbf{2.083} & \textbf{2.061} & \textbf{1.983} & \textbf{1.842} \\
\rowcolor{yellow!15} 0.0 & 1.732 & 1.691 & 1.602 & 1.477 \\
0.1 & 1.714 & 1.664 & 1.583 & 1.477 \\
0.3 & 1.715 & 1.678 & 1.575 & 1.474 \\
0.5 & 1.691 & 1.665 & 1.584 & 1.487 \\
\midrule
\multicolumn{5}{c}{\textit{\textbf{Qwen3} (8B / 1.7B)} — MMLU} \\
\midrule
\textsc{Official checkpoint} & \textbf{1.656} & \textbf{1.650} & \textbf{1.552} & \textbf{1.386} \\
\rowcolor{yellow!15} 0.0 & 1.438 & 1.426 & 1.317 & 1.189 \\
0.1 & 1.411 & 1.404 & 1.294 & 1.159 \\
0.3 & 1.418 & 1.411 & 1.317 & 1.176 \\
0.5 & 1.392 & 1.380 & 1.284 & 1.156 \\
\bottomrule
\end{tabular}
\caption{Full cross-$\alpha$ matrix for the Contrastive-Aligned Dual-Input Drafter. Numeric train-$\alpha$ rows report our CD-Aligned-EAGLE family: rows vary the training target, columns vary infer-$\alpha$, and cells report Mean Accepted Length (\textit{L}). The \textsc{Official checkpoint} row reports the publicly released model used in our main experiments, evaluated under the same inference protocol, as an external reference. Train-$\alpha{=}0.0$ is our expert-aligned baseline. \colorbox{yellow!15}{Highlighted rows} mark this baseline, and \textbf{bold} marks the best value in each infer-$\alpha$ column over all displayed rows.}
\label{tab:cross_alpha_full}
\vspace{-4pt}
\end{table*}

\subsection{Optimistic Offline Diagnostic with a Stronger Full-LM Proposer}
\label{sec:appendix_stronger_proposer}

This subsection extends the diagnosis in Section~\ref{sec:why_not_cd} to a stronger same-family proposer. It tests whether the mechanism in Figure~\ref{fig:signal_kl_ep_heatmap} disappears once the proposal side is replaced by a substantially stronger full LM.

\paragraph{Setup.}
We keep the expert and amateur fixed as \textsc{Llama-3.1-8B-Instruct} and \textsc{Llama-3.2-1B-Instruct}, and replace the lightweight proposal proxy with \textsc{Llama-3.2-3B-Instruct}. Let $\pi_r$ denote the 3B proposer distribution. For each shared response history $h$, we compare the true contrastive target
\[
\pi_{CD}(x\mid h) \propto \pi_p(x\mid h)^{1+\alpha} \pi_q(x\mid h)^{-\alpha}
\]
with the stronger-proposer approximation
\[
\hat{\pi}^{r}_{CD}(x\mid h) \propto \pi_r(x\mid h)^{1+\alpha} \pi_q(x\mid h)^{-\alpha}.
\]
We evaluate GSM8K and MMLU on 200 aligned examples each with $\alpha \in \{0.1, 0.5\}$. Reusing the true amateur logits and computing all statistics on shared response histories gives the proposer an optimistic offline diagnostic; online rollout performance is outside this check.

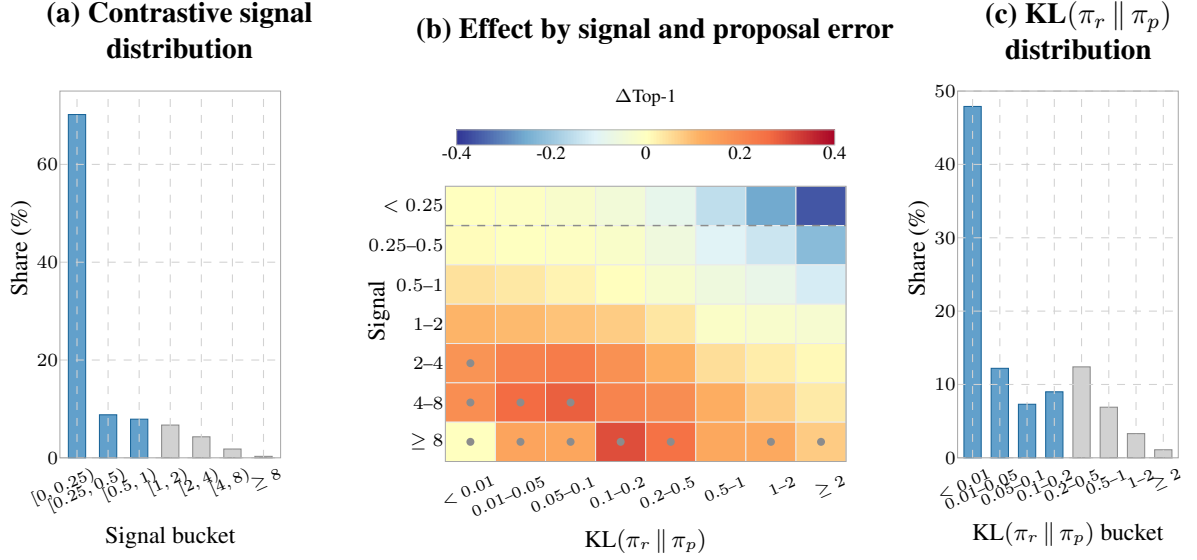
\begin{figure*}[t]
\centering
\hspace*{0cm}
\begin{subfigure}[t]{0.24\textwidth}
\centering
\parbox[t][2.7em][c]{\linewidth}{\centering
\makebox[0pt][c]{\hspace*{2.5em}\textbf{(a) Contrastive signal}}\\
\makebox[0pt][c]{\hspace*{2.5em}\textbf{distribution}}}\par\vspace{0.25em}
\begin{tikzpicture}
\begin{axis}[
    ybar,
    bar width=6.5pt,
    bar shift=0pt,
    width=0.76\linewidth,
    height=4.85cm,
    xmin=-0.55, xmax=6.55,
    ymin=0, ymax=75,
    xtick={0,1,2,3,4,5,6},
    xticklabels={{$[0,0.25)$},{$[0.25,0.5)$},{$[0.5,1)$},{$[1,2)$},{$[2,4)$},{$[4,8)$},{$\geq 8$}},
    ytick={0,20,40,60},
    xlabel={Signal bucket},
    ylabel={Share (\%)},
    xlabel style={font=\small, yshift=0pt},
    ylabel style={font=\small, xshift=0pt, yshift=0pt},
    xticklabel style={font=\tiny, rotate=24, anchor=east, xshift=6pt, yshift=-4pt},
    yticklabel style={font=\scriptsize},
    tick style={draw=none},
    axis line style={draw=black!35, line width=0.35pt},
    grid=major,
    grid style={dashed, gray!35},
    scale only axis,
    axis on top,
    clip=false,
]
\addplot[draw=dcdblue!85!black, fill=dcdblue!70] coordinates { (0,70.2) (1,8.8) (2,7.9) };
\addplot[draw=black!45, fill=black!18] coordinates { (3,6.7) (4,4.3) (5,1.8) (6,0.3) };
\end{axis}
\end{tikzpicture}
\end{subfigure}
\hspace{0.3cm}
\begin{subfigure}[t]{0.46\textwidth}
\centering
\parbox[t][2.7em][c]{\linewidth}{\centering
\makebox[0pt][c]{\hspace*{3.5em}\textbf{(b) Effect by signal and proposal error}}}\par\vspace{0.25em}
\begin{tikzpicture}
\begin{axis}[
    width=0.72\linewidth,
    height=3.64cm,
    xmin=-0.5, xmax=7.5,
    ymin=-0.5, ymax=6.5,
    y dir=reverse,
    xtick={0,1,2,3,4,5,6,7},
    ytick={0,1,2,3,4,5,6},
    xticklabels={{$<0.01$},{$0.01$--$0.05$},{$0.05$--$0.1$},{$0.1$--$0.2$},{$0.2$--$0.5$},{$0.5$--$1$},{$1$--$2$},{$\geq 2$}},
    yticklabels={{$<0.25$},{$0.25$--$0.5$},{$0.5$--$1$},{$1$--$2$},{$2$--$4$},{$4$--$8$},{$\geq 8$}},
    xticklabel style={font=\tiny, rotate=24, anchor=east, xshift=6pt, yshift=-10pt},
    yticklabel style={font=\scriptsize},
    tick style={draw=none},
    axis line style={draw=black!35, line width=0.35pt},
    xlabel={KL$(\pi_r \,\|\, \pi_p)$},
    ylabel={Signal},
    xlabel style={font=\small, yshift=0pt},
    ylabel style={font=\small, xshift=-5pt, yshift=-10pt},
    enlargelimits=false,
    scale only axis,
    axis on top,
    point meta min=-0.40,
    point meta max=0.40,
    colormap={paperdiv}{
        rgb255(0cm)=(49,54,149);
        rgb255(0.18cm)=(116,173,209);
        rgb255(0.36cm)=(224,243,248);
        rgb255(0.50cm)=(255,255,191);
        rgb255(0.64cm)=(253,174,97);
        rgb255(0.82cm)=(244,109,67);
        rgb255(1cm)=(165,0,38)
    },
    colorbar horizontal,
    colorbar style={
        at={(0.5,1.15)},
        anchor=south,
        title={$\Delta$Top-1},
        title style={font=\scriptsize, yshift=-1pt},
        xtick={-0.40,-0.20,0.00,0.20,0.40},
        xticklabels={{-0.4},{-0.2},{0},{0.2},{0.4}},
        ytick=\empty,
        tick label style={font=\scriptsize},
        width=5cm,
        height=0.2cm,
        major tick length=1.3pt,
        axis line style={draw=black!30, line width=0.3pt}
    }
]
\addplot[
    matrix plot*,
    mesh/cols=8,
    point meta=explicit,
    draw=black!8,
    line width=0.12pt,
] coordinates {
    (0,0) [-0.0002] (1,0) [-0.0082] (2,0) [-0.0241] (3,0) [-0.0487] (4,0) [-0.0864] (5,0) [-0.1562] (6,0) [-0.2600] (7,0) [-0.3607] (0,1) [0.0044] (1,1) [-0.0023] (2,1) [-0.0078] (3,1) [-0.0201] (4,1) [-0.0584] (5,1) [-0.1087] (6,1) [-0.1403] (7,1) [-0.2276] (0,2) [0.0412] (1,2) [0.0320] (2,2) [0.0159] (3,2) [0.0015] (4,2) [-0.0272] (5,2) [-0.0599] (6,2) [-0.0796] (7,2) [-0.1256] (0,3) [0.1031] (1,3) [0.0960] (2,3) [0.0832] (3,3) [0.0704] (4,3) [0.0345] (5,3) [-0.0164] (6,3) [-0.0261] (7,3) [-0.0335] (0,4) [0.1695] (1,4) [0.2093] (2,4) [0.2230] (3,4) [0.1769] (4,4) [0.1109] (5,4) [0.0458] (6,4) [0.0254] (7,4) [0.0086] (0,5) [0.1818] (1,5) [0.2576] (2,5) [0.2738] (3,5) [0.2051] (4,5) [0.1858] (5,5) [0.1131] (6,5) [0.0677] (7,5) [0.0287] (0,6) [0.0000] (1,6) [0.1333] (2,6) [0.1304] (3,6) [0.2963] (4,6) [0.2444] (5,6) [0.1250] (6,6) [0.1236] (7,6) [0.0714]
};
\fill[black!45] (axis cs:0,4) circle (1.5pt);
\fill[black!45] (axis cs:0,5) circle (1.5pt);
\fill[black!45] (axis cs:1,5) circle (1.5pt);
\fill[black!45] (axis cs:2,5) circle (1.5pt);
\fill[black!45] (axis cs:0,6) circle (1.5pt);
\fill[black!45] (axis cs:1,6) circle (1.5pt);
\fill[black!45] (axis cs:2,6) circle (1.5pt);
\fill[black!45] (axis cs:3,6) circle (1.5pt);
\fill[black!45] (axis cs:4,6) circle (1.5pt);
\fill[black!45] (axis cs:6,6) circle (1.5pt);
\fill[black!45] (axis cs:7,6) circle (1.5pt);
\draw[black!45, dashed, line width=0.55pt] (axis cs:-0.5,0.5) -- (axis cs:7.5,0.5);
\end{axis}
\end{tikzpicture}
\end{subfigure}
\hspace{0cm}
\begin{subfigure}[t]{0.24\textwidth}
\centering
\parbox[t][2.7em][c]{\linewidth}{\centering
\makebox[0pt][c]{\hspace*{2.5em}\textbf{(c) KL$(\pi_r \,\|\, \pi_p)$}}\\
\makebox[0pt][c]{\hspace*{2.5em}\textbf{distribution}}}\par\vspace{0.25em}
\begin{tikzpicture}
\begin{axis}[
    ybar,
    bar width=6.5pt,
    bar shift=0pt,
    width=0.76\linewidth,
    height=4.85cm,
    xmin=-0.55, xmax=7.55,
    ymin=0, ymax=50,
    xtick={0,1,2,3,4,5,6,7},
    xticklabels={{$<0.01$},{$0.01$--$0.05$},{$0.05$--$0.1$},{$0.1$--$0.2$},{$0.2$--$0.5$},{$0.5$--$1$},{$1$--$2$},{$\geq 2$}},
    ytick={0,10,20,30,40,50},
    xlabel={KL$(\pi_r \,\|\, \pi_p)$ bucket},
    ylabel={Share (\%)},
    xlabel style={font=\small, yshift=0pt},
    ylabel style={font=\small, xshift=0pt, yshift=0pt},
    xticklabel style={font=\tiny, rotate=24, anchor=east, xshift=6pt, yshift=-4pt},
    yticklabel style={font=\scriptsize},
    tick style={draw=none},
    axis line style={draw=black!35, line width=0.35pt},
    grid=major,
    grid style={dashed, gray!35},
    scale only axis,
    axis on top,
    clip=false,
]
\addplot[draw=dcdblue!85!black, fill=dcdblue!70] coordinates { (0,47.9) (1,12.2) (2,7.3) (3,9.0) };
\addplot[draw=black!45, fill=black!18] coordinates { (4,12.4) (5,6.9) (6,3.3) (7,1.1) };
\end{axis}
\end{tikzpicture}
\end{subfigure}

\caption{Full offline 3B proposer diagnostic. $\Delta$Top-1 is the CD-aware reconstruction's Top-1 agreement gain over the unmodified 3B proposer, measured against $\arg\max_x \pi_{CD}(x\mid h)$. Low-signal rows remain mostly negative; setting-level mean $\Delta$Top-1 values are slightly negative in all four settings (GSM8K ($\alpha=0.1$) -0.0028; GSM8K ($\alpha=0.5$) -0.0123; MMLU ($\alpha=0.1$) -0.0030; MMLU ($\alpha=0.5$) -0.0144).}
\label{fig:stronger_proposer_signal_kl_full}
\vspace{-4pt}
\end{figure*}

Figure~\ref{fig:stronger_proposer_signal_kl_full} reports the same three-part diagnostic as Figure~\ref{fig:signal_kl_ep_heatmap}. Under this optimistic proposer replacement, favorable gains remain concentrated in higher-signal, lower-error regions, while low-signal or high-proposer-error buckets remain unfavorable. The stronger proposer therefore enlarges the favorable pocket but does not overturn the conclusion that proposal-side approximation error limits contrastive-aligned drafting.

\section{Theory for Expert-Aligned Drafting in Contrastive Acceleration}
\label{sec:appendix_theory}

This appendix provides the robustness analysis summarized in Section~\ref{sec:theory}. It characterizes the expert-aligned lightweight-drafter regime used by DCD$_{\textsc{eagle3}}$ and connects the analysis to the observed relative degradation trends; the empirical comparisons remain the primary evidence.

\subsection{Distribution Alignment and Robustness Analysis}

We analyze how draft distributions align with the contrastive decoding target distribution $\pi_{CD}$ and why expert-aligned drafting is more robust to $\alpha$ under the stated alignment condition. This analysis supports the empirical finding that mean accepted length decays more gradually in the expert-aligned DCD$_{\textsc{eagle3}}$ instantiation as contrastive strength increases.

\paragraph{KL Divergence Decomposition.}

Let $\pi_p$, $\pi_q$, and $\pi_e$ denote the expert, amateur, and lightweight-drafter distributions, respectively. The target distribution is defined as $\pi_{CD}(x) = \frac{1}{Z_{CD}} \pi_p(x)^{1+\alpha} \pi_q(x)^{-\alpha}$.

\begin{lemma}[KL Divergence to Target]
\label{lem:kl_decomposition}
\begin{align*}
D_{KL}(\pi_p \| \pi_{CD})\! &=\! \log Z_{CD} - \alpha D_{KL}(\pi_p \| \pi_q), \\
D_{KL}(\pi_q \| \pi_{CD})\! &=\! \log Z_{CD} \\
&\quad + (1+\alpha) D_{KL}(\pi_q \| \pi_p), \\
D_{KL}(\pi_e \| \pi_{CD})\! &=\! \log Z_{CD} \\
&\quad + (1+\alpha) D_{KL}(\pi_e \| \pi_p) \\
&\quad - \alpha D_{KL}(\pi_e \| \pi_q).
\end{align*}
\end{lemma}

\begin{proof}
Substituting the definition of $\pi_{CD}$ into $D_{KL}(\cdot \| \pi_{CD})$ and expanding the logarithms gives the result.
\end{proof}

\paragraph{Static Analysis.}

We first examine the absolute distance between each draft distribution and the target $\pi_{CD}$ at a fixed $\alpha$.

From Lemma~\ref{lem:kl_decomposition}, the KL divergence gap between the amateur model and the drafter relative to the target is:
\begin{equation}
\label{eq:static_alignment}
\begin{split}
&D_{KL}(\pi_q \| \pi_{CD}) - D_{KL}(\pi_e \| \pi_{CD}) \\
&\quad = (1+\alpha) [D_{KL}(\pi_q \| \pi_p) - D_{KL}(\pi_e \| \pi_p)] \\
&\qquad + \alpha D_{KL}(\pi_e \| \pi_q).
\end{split}
\end{equation}
This identity follows directly from Lemma~\ref{lem:kl_decomposition}, where the $\log Z_{CD}$ terms cancel.

\noindent \textbf{Remark (Static Capacity Gap).}
In practice, if the absolute distance from the drafter to the target is larger than the corresponding distance from the amateur ($D_{KL}(\pi_e \| \pi_{CD}) > D_{KL}(\pi_q \| \pi_{CD})$), it implies:
\[
\begin{split}
D_{KL}(\pi_e \| \pi_p) > & D_{KL}(\pi_q \| \pi_p) \\
& + \frac{\alpha}{1+\alpha} D_{KL}(\pi_e \| \pi_q).
\end{split}
\]

This inequality quantifies the \textit{static capacity gap}. Although the lightweight drafter is trained to mimic the expert, its approximation error $D_{KL}(\pi_e \| \pi_p)$ can still exceed the amateur's baseline error by a margin large enough to offset the alignment term. This interpretation matches our acceptance-rate comparisons, in which the lightweight drafter can start from a lower acceptance baseline than the amateur model.

\paragraph{Dynamic Analysis: Robustness to Contrastive Penalty.}
Although $\pi_e$ may be farther from the target in absolute terms, its response to increasing $\alpha$ can still be more favorable under the alignment condition below.

To keep the analysis tied to a meaningful draft model, we use the following alignment condition.

\begin{definition}[Effective Draft Alignment]
\label{def:effective_alignment}
A draft distribution $\pi_e$ is effectively aligned if:
\[
\mathbb{E}_{x \sim \pi_e}\left[\log \frac{\pi_p(x)}{\pi_q(x)}\right] > \mathbb{E}_{x \sim \pi_q}\left[\log \frac{\pi_p(x)}{\pi_q(x)}\right].
\]
\end{definition}

This condition is equivalent to
\begin{equation}
\sum_x (\pi_e(x) - \pi_q(x)) \log \frac{\pi_p(x)}{\pi_q(x)} > 0.
\label{eq:alignment_condition}
\end{equation}

\noindent \textbf{Remark.}
The condition is practical. Because $\pi_e$ is trained to approximate $\pi_p$, it tends to move probability mass toward tokens on which the expert has a larger log-likelihood advantage over the amateur. Condition \eqref{eq:alignment_condition} only requires the drafter to separate expert-preferred tokens from the rest at least as well as the amateur does.

\begin{theorem}[Slope Comparison under Effective Draft Alignment]
\label{thm:robustness}
Assuming $\pi_e$ satisfies the Effective Draft Alignment condition, the distance from the amateur distribution $\pi_q$ to the target distribution $\pi_{CD}$ grows faster than the distance from the lightweight drafter $\pi_e$ as $\alpha$ increases:
\[
\frac{\partial}{\partial \alpha} D_{KL}(\pi_q \| \pi_{CD}) > \frac{\partial}{\partial \alpha} D_{KL}(\pi_e \| \pi_{CD}).
\]
\end{theorem}

\paragraph{Proof of Theorem \ref{thm:robustness}.}
\label{sec:proof_robustness}

Empirically, SCD often has a higher acceptance rate at $\alpha=0$, whereas expert-aligned drafting typically deteriorates more slowly as $\alpha$ increases. We explain this trend by tracking how the distance to the target changes with $\alpha$.

\begin{proof}
The target distribution is:
\[
\pi_{CD}(x) = \frac{1}{Z_{CD}} \pi_p(x)^{1+\alpha} \pi_q(x)^{-\alpha}.
\]
Substituting its log-form into the KL divergence formula $D_{KL}(\cdot \| \pi_{CD})$ and differentiating with respect to $\alpha$:

For the amateur model $\mathcal{M}_q$:
\begin{align*}
\text{Slope}_q &= \frac{\partial}{\partial \alpha} D_{KL}(\pi_q \| \pi_{CD}) \\
&=\! \frac{\partial}{\partial \alpha} \mathbb{E}_{x \sim \pi_q} [\log \pi_q(x) \\
&\quad - ((1+\alpha)\log \pi_p(x) \\
&\quad - \alpha \log \pi_q(x) - \log Z_{CD})] \\
&=\! \mathbb{E}_{x \sim \pi_q} [\log \pi_q(x) - \log \pi_p(x)] \\
&\quad + \frac{\partial \log Z_{CD}}{\partial \alpha}.
\end{align*}

For the drafter $\pi_e$:
\begin{align*}
\text{Slope}_e &= \frac{\partial}{\partial \alpha} D_{KL}(\pi_e \| \pi_{CD}) \\
&= \frac{\partial}{\partial \alpha} \mathbb{E}_{x \sim \pi_e} [\log \pi_e(x) \\
&\quad - ((1+\alpha)\log \pi_p(x) \\
&\quad - \alpha \log \pi_q(x) - \log Z_{CD})] \\
&= \! \mathbb{E}_{x \sim \pi_e}\! [\log \pi_q(x)\! -\! \log \pi_p(x)] \\
&\quad + \frac{\partial \log Z_{CD}}{\partial \alpha}.
\end{align*}

Considering the slope difference $\Delta_{\text{slope}}$, the normalization term $\frac{\partial \log Z_{CD}}{\partial \alpha}$ cancels out:
\begin{align}
\Delta_{\text{slope}} &= \text{Slope}_q - \text{Slope}_e \nonumber \\
&= \mathbb{E}_{x \sim \pi_q} [\log \pi_q(x) - \log \pi_p(x)] \nonumber \\
&\quad - \mathbb{E}_{x \sim \pi_e} [\log \pi_q(x) - \log \pi_p(x)] \nonumber \\
&= \sum_x \pi_q(x) \log \frac{\pi_q(x)}{\pi_p(x)} \nonumber \\
&\quad - \sum_x \pi_e(x) \log \frac{\pi_q(x)}{\pi_p(x)} \nonumber \\
&= \sum_x (\pi_q(x) - \pi_e(x)) \log \frac{\pi_q(x)}{\pi_p(x)} \nonumber \\
&= \sum_x (\pi_e(x)\! -\! \pi_q(x)) \log \frac{\pi_p(x)}{\pi_q(x)}.
\label{eq:final_slope_diff}
\end{align}

By \textbf{Definition \ref{def:effective_alignment}}, the expression in \eqref{eq:final_slope_diff} is strictly positive. Thus $\Delta_{\text{slope}} > 0$.

\end{proof}
\textbf{Conclusion and Discussion.}
As $\alpha$ increases, the amateur model $\mathcal{M}_q$ used in SCD diverges from the contrastive target faster than the lightweight drafter under Effective Draft Alignment. This explains the acceptance-side robustness of expert-aligned drafting in DCD. The ablations in Section~\ref{sec:ablation_studies} show the empirical counterpart: the relative mean accepted length of DCD declines more slowly than that of SCD as $\alpha$ increases.

\paragraph{Synthesis.}
This analysis clarifies the empirical pattern: although the lightweight drafter may begin farther from the target because of its capacity limits (Eq.~\ref{eq:static_alignment}), it has a more favorable dynamic slope under Theorem~\ref{thm:robustness}. This split between static gap and dynamic robustness explains the acceptance-side degradation trends observed in the experiments.

\subsection{Efficiency Advantage Condition}
\label{sec:proof_efficiency}

Let $\tau_m$ and $T_m$ denote the mean output tokens and serial time per iteration for method $m$. Under the throughput surrogate $S_m=\tau_m/T_m$, DCD is faster than an amateur-coupled baseline $b$ if and only if
\begin{equation}
\frac{\tau_{\text{DCD}}}{\tau_b} > \frac{T_{\text{DCD}}}{T_b}.
\label{eq:efficiency_condition_appendix}
\end{equation}
Thus, a lower accepted-token yield can be offset by a sufficiently large reduction in serial cost. Appendix~\ref{sec:theory_speed_calc} expands $T_m$ for DCD, SCD, and CoS and reports the measured decomposition used in Section~\ref{sec:trade_off}.

\subsection{Proof of Lossless Property}
\label{sec:proof_lossless}

DCD inherits the lossless guarantee of speculative decoding by construction. It changes only the \textit{choice} of proposal distribution while leaving the verification procedure unchanged.

\begin{theorem}[Lossless Property of DCD]
\label{thm:lossless}
DCD produces outputs distributed identically to autoregressive decoding from $\pi_{CD}$.
\end{theorem}

\begin{proof}
Speculative decoding is lossless for any target distribution $\pi_{\mathrm{target}}$ that is a valid probability distribution~\citep{leviathan2023fast,chen2023accelerating}. DCD instantiates $\pi_{\mathrm{target}} = \pi_{CD}$, where:
\[
\pi_{CD}(x) = \frac{\pi_p(x)^{1+\alpha} \pi_q(x)^{-\alpha}}{\sum_{x'} \pi_p(x')^{1+\alpha} \pi_q(x')^{-\alpha}}.
\]
Since softmax outputs satisfy $\pi_p(x), \pi_q(x) > 0$ for all $x \in \mathcal{V}$, the numerator is strictly positive and the normalization is well defined. Thus $\pi_{CD}$ is a valid distribution, and the lossless property follows immediately.
\end{proof}

\begin{remark}
The choice of proposer affects only the acceptance rate (efficiency), never output correctness, as long as verification uses the CD target distribution. This separation is fundamental to speculative decoding and carries over to DCD unchanged.
\end{remark}

\begin{remark}[Plausibility-constrained CD]
The same argument applies when vanilla CD uses the expert-plausibility constraint of \citet{li2023contrastive}. In that case, the target distribution is the masked-and-renormalized contrastive distribution over $\mathcal{P}_{\beta}(h)$, and speculative verification still samples from a valid normalized target distribution. DCD therefore remains a lossless acceleration of the corresponding truncated contrastive target; only the acceptance rate changes when the proposer emits tokens outside the mask.
\end{remark}

\section{Alignment Gap Estimation Methodology}
\label{sec:kl_details}

Table~\ref{tab:alignment_gap_summary} summarizes the empirical evidence for Definition~\ref{def:effective_alignment}. We measure the estimates in four greedy settings: \textsc{Qwen3}/GSM8K, \textsc{Qwen3}/MMLU, \textsc{Llama-3}/GSM8K, and \textsc{Llama-3}/MMLU. Following the shared evaluation protocol in Section~\ref{sec:experiments}, we use 200 evaluation samples per dataset, greedy decoding ($T=0$), and a base draft length of $\gamma=5$. This appendix describes how we estimate $\mathbb{E}_{\pi_e}[\log(\pi_p/\pi_q)]$ and $\mathbb{E}_{\pi_q}[\log(\pi_p/\pi_q)]$ from draft trees collected during DCD inference.

\begin{table}[t]
\centering
\small
\renewcommand{\arraystretch}{1.08}
\setlength{\tabcolsep}{4pt}
\begin{tabular*}{\columnwidth}{@{\extracolsep{\fill}}lccc}
\toprule
\textbf{Setting} & \textbf{LHS} & \textbf{RHS} & \textbf{Gap} \\
\midrule
\textsc{Qwen3} / GSM8K & 0.392 & $-$0.348 & \textbf{0.740} \\
\textsc{Qwen3} / MMLU & 0.181 & $-$0.571 & \textbf{0.753} \\
\textsc{Llama-3} / GSM8K & $-$0.305 & $-$0.332 & \textbf{0.028} \\
\textsc{Llama-3} / MMLU & $-$0.274 & $-$0.485 & \textbf{0.211} \\
\bottomrule
\end{tabular*}
\caption{Empirical alignment-gap summary in four greedy settings. Here LHS $= \mathbb{E}_{\pi_e}[\log(\pi_p/\pi_q)]$, RHS $= \mathbb{E}_{\pi_q}[\log(\pi_p/\pi_q)]$, and gap $= \mathrm{LHS} - \mathrm{RHS}$; positive gaps support Effective Draft Alignment (Definition~\ref{def:effective_alignment}).}
\label{tab:alignment_gap_summary}
\end{table}

\paragraph{Sampling.}
For each evaluation sample in the reported settings, we collect the EAGLE3 draft tree generated by greedy DCD$_{\textsc{eagle3}}$ inference with $\alpha=0.1$. At each node, we record full-vocabulary logits from the expert model $\mathcal{M}_p$, the amateur model $\mathcal{M}_q$, and the drafter $E$.

\paragraph{Estimator.}
Let $r(x) = \log \pi_p(x) - \log \pi_q(x)$. We then compute:
\begin{align*}
\mathbb{E}_{\pi_e}\left[\log\frac{\pi_p}{\pi_q}\right] &= \sum_{x \in \mathcal{V}} \pi_e(x) \cdot r(x), \\
\mathbb{E}_{\pi_q}\left[\log\frac{\pi_p}{\pi_q}\right] &= \sum_{x \in \mathcal{V}} \pi_q(x) \cdot r(x),
\end{align*}
These expectations are evaluated over the full vocabulary at each position and then averaged across all positions and samples.

\section{Throughput Analysis}
\label{sec:theory_speed_calc}

This section makes explicit the coarse throughput surrogate used in Section~\ref{sec:trade_off}. The formulas below use measured latency components to form analytic approximations to runtime throughput.

\subsection{DCD}
\label{sec:dcd_speed_calc}
DCD decouples proposal generation from the amateur model through an amateur-independent proposer $E$.

At each iteration, the proposer generates $\gamma$ draft tokens, and the expert and amateur then run forward passes on the same draft sequence for verification.

We approximate the serial time per iteration by
\begin{equation*}
T_{\text{DCD}} = \gamma \cdot t_e + t_p + t_q,
\end{equation*}
where $t_e$, $t_p$, and $t_q$ denote the latency components reported in Section~\ref{sec:trade_off} for the proposer, expert path, and amateur path, respectively.

Under the same coarse surrogate $S \approx (L+1)/T_{\text{iter}}$, the throughput approximation is
\begin{equation*}
\text{Speed}_{\text{DCD}} = \frac{L + 1}{\gamma \cdot t_e + t_p + t_q},
\end{equation*}
where $L$ denotes the average number of accepted tokens per iteration, excluding the additional token.

When the proposer is lightweight or retrieval-based, $t_e \ll t_q$, so the proposal overhead of DCD is much lower than that of SCD.

\subsection{SCD with Delayed Drafting}
\label{sec:scd_opt_speed_calc}

SCD~\citep{yuan2024speculative} with delayed drafting uses conditional drafting to reduce unnecessary draft generation.

In each iteration, the amateur model first generates $\gamma$ draft tokens. An additional draft token is generated only when all $\gamma$ draft tokens are accepted by the contrastive distribution, which occurs with probability $\rho^\gamma$.

The expected number of draft operations per iteration is therefore
\begin{equation*}
N_{\text{draft}} = \gamma + \rho^\gamma,
\end{equation*}
where $\rho$ denotes the single-step acceptance rate. The corresponding throughput approximation is
\begin{equation*}
\text{Speed}_{\text{SCD}} = \frac{L + 1}{(\gamma + \rho^\gamma) \cdot t_q + t_p},
\end{equation*}
where $L$ denotes the average number of accepted tokens per iteration, excluding the additional token.

In practice, we infer an effective acceptance rate $\rho$ from the observed $L$ value using the local relation $L \approx \sum_{i=1}^{\gamma} \rho^i$.

\subsection{CoS}
\label{sec:cos_speed_calc}
CoS~\citep{fu2025speculative} uses an alternating proposal mechanism to reduce draft steps.

If, in the previous iteration, all amateur-model draft tokens and the expert proposal token are accepted, the current iteration can skip one amateur-model draft step.

The probability of saving a draft step is
\begin{equation*}
P_{\text{save}} = \rho_p^{\gamma_p} \cdot \rho_q^{\gamma_q},
\end{equation*}
where $\rho_p$ denotes the per-token acceptance rate for the expert proposal ($\gamma_p$ is typically set to 1), and $\rho_q$ denotes the per-token acceptance rate for the amateur model, so that $\rho_q^{\gamma_q}$ is the probability that all $\gamma_q$ draft tokens are accepted.

Conversely, when all tokens in the current iteration are accepted, an additional amateur-model draft step is required with probability $\rho_q^{\gamma_q}$ to verify the expert proposal.

The expected number of draft steps per iteration is therefore
\begin{equation*}
N_{\text{draft}} = \gamma - \rho_p^{\gamma_p} \cdot \rho_q^{\gamma_q} + \rho_q^{\gamma_q}.
\end{equation*}
The corresponding throughput approximation is
\begin{equation*}
\text{Speed}_{\text{CoS}} = \frac{L + 1}{(\gamma - \rho_p^{\gamma_p} \cdot \rho_q^{\gamma_q} + \rho_q^{\gamma_q}) \cdot t_q + t_p },
\end{equation*}
where $L$ denotes the average number of accepted tokens per iteration, excluding the additional token.

In practice, we infer an effective amateur acceptance rate $\rho_q$ from the observed $L$ value using the same local relation. Because $\rho_p$ is not measured directly, we set it to 0.9 as a calibration constant.

\subsection{Multiplicative Throughput Decomposition}
\label{sec:throughput_decomposition_appendix}

The speed formulas above yield the same coarse multiplicative decomposition for the advantage of DCD over any amateur-coupled baseline $b$ evaluated at the same operating point:
\[
\frac{\mathrm{Speed}_{\mathrm{DCD}}}{\mathrm{Speed}_b}
\approx
\frac{L_{\mathrm{DCD}}+1}{L_b+1}
\cdot
\frac{T_b}{T_{\mathrm{DCD}}}.
\]
The first term captures acceptance and the second captures serial cost. This decomposition is useful because it separates token acceptance from generation latency. In our measured regime, the acceptance term is typically below 1, whereas the serial-cost term is well above 1 and dominates the product.

\begin{table}[t]
\centering
\small
\renewcommand{\arraystretch}{1.12}
\setlength{\tabcolsep}{4pt}
\begin{tabular*}{\columnwidth}{@{\extracolsep{\fill}}llcccc}
\toprule
\textbf{Model} & \textbf{Base} & \textbf{Acc.} & \textbf{Cost} & \textbf{Pred.} & \textbf{Meas.} \\
\midrule
\textsc{Llama-3} & SCD & 0.71 & 1.62 & 1.14x & 1.12x \\
 & CoS & 0.71 & 1.62 & 1.15x & 1.12x \\
\midrule
\textsc{Qwen3} & SCD & 0.71 & 2.06 & 1.46x & 1.41x \\
 & CoS & 0.71 & 2.01 & 1.43x & 1.38x \\
\midrule
\textsc{Llama-EFT} & SCD & 0.68 & 2.30 & 1.55x & 1.49x \\
 & CoS & 0.68 & 2.50 & 1.69x & 1.61x \\
\bottomrule
\end{tabular*}
\caption{Multiplicative decomposition of the speed advantage of DCD$_{\textsc{eagle3}}$ on MMLU ($T{=}0$, $\alpha{=}0.1$, $\gamma{=}5$). For an amateur-coupled baseline $b$, the predicted ratio is approximated by an acceptance factor $(L_{\mathrm{DCD}_{\textsc{eagle3}}}+1)/(L_b+1)$ times a serial-cost factor $T_b/T_{\mathrm{DCD}_{\textsc{eagle3}}}$. Acc. below 1 means DCD$_{\textsc{eagle3}}$ accepts fewer tokens per speculative round; Cost above 1 means DCD$_{\textsc{eagle3}}$ has a cheaper serial proposal path. Across all model families, the cost factor is the dominant term.}
\label{tab:throughput_decomposition}
\vspace{-4pt}
\end{table}

\section{Additional Results and Ablations}
\label{sec:appendix_additional_results}
\subsection{Main-Result Breakdowns}
\label{sec:appendix_exp5}

This subsection reports the full throughput tables underlying \Cref{tab:exp5_main_results}.

\Cref{tab:exp5_appendix_results} reports the greedy results for $\alpha \in \{0.1, 0.5\}$, including raw tokens/sec and reported standard deviations.

\Cref{tab:t1_exp5_appendix_results} reports the corresponding sampling results at $T=1$.

\begin{table*}[t]
\centering
\small
\renewcommand{\arraystretch}{1.15}
\setlength{\tabcolsep}{3pt}
\begin{tabular}{lccccc}
\toprule
\textbf{Method} & \textbf{HumanEval} & \textbf{GSM8K} & \textbf{MMLU} & \textbf{CNN/DM} & \textbf{Avg.} \\
\midrule
\multicolumn{6}{c}{\textit{Llama-3 (Llama-3.1-8B-Instruct / Llama-3.2-1B-Instruct)}} \\
\midrule
\multicolumn{6}{c}{$\alpha = 0.1$} \\
\cmidrule(lr){1-6}
CD & 91.3$\pm$4.3 (1.00$\times$) & 86.6$\pm$5.6 (1.00$\times$) & 93.1$\pm$2.7 (1.00$\times$) & 84.0$\pm$0.6 (1.00$\times$) & 88.7 (1.00$\times$) \\
SCD & 139.2$\pm$4.7 (1.52$\times$) & 132.1$\pm$2.4 (1.53$\times$) & 117.4$\pm$1.3 (1.26$\times$) & 117.4$\pm$2.5 (1.40$\times$) & 126.5 (1.43$\times$) \\
CoS & 156.5$\pm$16.9 (1.71$\times$) & 153.7$\pm$2.6 (1.77$\times$) & 119.9$\pm$4.7 (1.29$\times$) & 115.0$\pm$6.4 (1.37$\times$) & 136.3 (1.54$\times$) \\
DCD$_{\textsc{ngram}}$ & 135.0$\pm$1.3 (1.48$\times$) & 124.0$\pm$1.4 (1.43$\times$) & 120.5$\pm$0.1 (1.29$\times$) & 173.6$\pm$3.0 (2.07$\times$) & 138.3 (1.56$\times$) \\
\rowcolor{highlightcolor} \textbf{DCD$_{\textsc{eagle3}}$} & \textbf{189.2$\pm$2.5 (2.07$\times$)} & \textbf{159.7$\pm$1.6 (1.84$\times$)} & \textbf{149.0$\pm$7.2 (1.60$\times$)} & \textbf{151.4$\pm$2.1 (1.80$\times$)} & \textbf{162.3 (1.83$\times$)} \\
\midrule
\multicolumn{6}{c}{$\alpha = 0.5$} \\
\cmidrule(lr){1-6}
CD & 95.7$\pm$0.6 (1.00$\times$) & 91.7$\pm$3.7 (1.00$\times$) & 83.0$\pm$0.7 (1.00$\times$) & 90.0$\pm$3.5 (1.00$\times$) & 90.1 (1.00$\times$) \\
SCD & 152.2$\pm$0.7 (1.59$\times$) & 130.8$\pm$1.0 (1.43$\times$) & 100.1$\pm$0.6 (1.21$\times$) & 100.5$\pm$2.3 (1.12$\times$) & 120.9 (1.34$\times$) \\
CoS & 159.6$\pm$0.2 (1.67$\times$) & 142.8$\pm$0.9 (1.56$\times$) & 99.7$\pm$0.6 (1.20$\times$) & 101.6$\pm$3.4 (1.13$\times$) & 126.0 (1.40$\times$) \\
DCD$_{\textsc{ngram}}$ & 133.3$\pm$0.7 (1.39$\times$) & 122.7$\pm$0.2 (1.34$\times$) & 128.3$\pm$2.3 (1.55$\times$) & 188.8$\pm$1.4 (2.10$\times$) & 143.3 (1.59$\times$) \\
\rowcolor{highlightcolor} \textbf{DCD$_{\textsc{eagle3}}$} & \textbf{181.4$\pm$1.1 (1.89$\times$)} & \textbf{152.1$\pm$2.6 (1.66$\times$)} & \textbf{119.5$\pm$0.2 (1.44$\times$)} & \textbf{143.2$\pm$3.5 (1.59$\times$)} & \textbf{149.0 (1.65$\times$)} \\
\midrule
\multicolumn{6}{c}{\textit{Qwen3 (Qwen3-8B / Qwen3-1.7B)}} \\
\midrule
\multicolumn{6}{c}{$\alpha = 0.1$} \\
\cmidrule(lr){1-6}
CD & 59.9$\pm$5.7 (1.00$\times$) & 65.7$\pm$1.8 (1.00$\times$) & 65.7$\pm$0.4 (1.00$\times$) & 61.0$\pm$4.1 (1.00$\times$) & 63.1 (1.00$\times$) \\
SCD & 82.8$\pm$3.9 (1.38$\times$) & 88.5$\pm$0.9 (1.35$\times$) & 64.8$\pm$0.2 (0.99$\times$) & 64.0$\pm$3.6 (1.05$\times$) & 75.0 (1.19$\times$) \\
CoS & 87.1$\pm$2.1 (1.45$\times$) & 87.4$\pm$6.5 (1.33$\times$) & 76.6$\pm$0.4 (1.17$\times$) & 69.3$\pm$2.4 (1.14$\times$) & 80.1 (1.27$\times$) \\
DCD$_{\textsc{ngram}}$ & 87.1$\pm$0.8 (1.45$\times$) & 88.9$\pm$1.2 (1.35$\times$) & 77.1$\pm$1.4 (1.17$\times$) & 69.6$\pm$2.4 (1.14$\times$) & 80.7 (1.28$\times$) \\
\rowcolor{highlightcolor} \textbf{DCD$_{\textsc{eagle3}}$} & \textbf{137.8$\pm$0.8 (2.30$\times$)} & \textbf{131.9$\pm$4.1 (2.01$\times$)} & \textbf{108.8$\pm$4.8 (1.66$\times$)} & \textbf{113.0$\pm$1.2 (1.85$\times$)} & \textbf{122.9 (1.95$\times$)} \\
\midrule
\multicolumn{6}{c}{$\alpha = 0.5$} \\
\cmidrule(lr){1-6}
CD & 64.1$\pm$4.1 (1.00$\times$) & 66.9$\pm$0.2 (1.00$\times$) & 64.4$\pm$2.0 (1.00$\times$) & 63.2$\pm$3.3 (1.00$\times$) & 64.6 (1.00$\times$) \\
SCD & 78.8$\pm$0.7 (1.23$\times$) & 83.1$\pm$1.4 (1.24$\times$) & 56.8$\pm$6.0 (0.88$\times$) & 51.9$\pm$0.9 (0.82$\times$) & 67.6 (1.05$\times$) \\
CoS & 80.1$\pm$2.6 (1.25$\times$) & 87.7$\pm$0.9 (1.31$\times$) & 62.1$\pm$2.2 (0.97$\times$) & 49.7$\pm$2.9 (0.79$\times$) & 69.9 (1.08$\times$) \\
DCD$_{\textsc{ngram}}$ & 84.5$\pm$0.2 (1.32$\times$) & 85.3$\pm$0.5 (1.28$\times$) & 75.2$\pm$1.8 (1.17$\times$) & 63.4$\pm$1.2 (1.00$\times$) & 77.1 (1.19$\times$) \\
\rowcolor{highlightcolor} \textbf{DCD$_{\textsc{eagle3}}$} & \textbf{127.1$\pm$0.7 (1.98$\times$)} & \textbf{123.4$\pm$4.8 (1.85$\times$)} & \textbf{103.0$\pm$4.5 (1.60$\times$)} & \textbf{88.1$\pm$3.3 (1.39$\times$)} & \textbf{110.4 (1.71$\times$)} \\
\midrule
\multicolumn{6}{c}{\textit{Llama-EFT (Llama-3.1-8B-Instruct / Llama-3.1-8B)}} \\
\midrule
\multicolumn{6}{c}{$\alpha = 0.1$} \\
\cmidrule(lr){1-6}
CD & 75.9$\pm$1.5 (1.00$\times$) & 75.0$\pm$0.8 (1.00$\times$) & 73.1$\pm$3.6 (1.00$\times$) & 69.1$\pm$0.5 (1.00$\times$) & 73.3 (1.00$\times$) \\
SCD & 91.4$\pm$2.1 (1.20$\times$) & 87.1$\pm$1.0 (1.16$\times$) & 77.5$\pm$3.1 (1.06$\times$) & 71.4$\pm$1.1 (1.03$\times$) & 81.9 (1.12$\times$) \\
CoS & 98.0$\pm$3.5 (1.29$\times$) & 91.5$\pm$1.3 (1.22$\times$) & 73.9$\pm$1.3 (1.01$\times$) & 73.2$\pm$5.5 (1.06$\times$) & 84.1 (1.15$\times$) \\
DCD$_{\textsc{ngram}}$ & 108.3$\pm$0.5 (1.43$\times$) & 97.9$\pm$1.3 (1.31$\times$) & 96.2$\pm$0.4 (1.32$\times$) & 142.1$\pm$1.8 (2.06$\times$) & 111.1 (1.52$\times$) \\
\rowcolor{highlightcolor} \textbf{DCD$_{\textsc{eagle3}}$} & \textbf{154.5$\pm$5.2 (2.03$\times$)} & \textbf{130.1$\pm$0.4 (1.74$\times$)} & \textbf{113.3$\pm$1.7 (1.55$\times$)} & \textbf{120.6$\pm$8.0 (1.74$\times$)} & \textbf{129.6 (1.77$\times$)} \\
\midrule
\multicolumn{6}{c}{$\alpha = 0.5$} \\
\cmidrule(lr){1-6}
CD & 70.7$\pm$2.9 (1.00$\times$) & 72.3$\pm$2.8 (1.00$\times$) & 73.4$\pm$3.4 (1.00$\times$) & 71.6$\pm$3.6 (1.00$\times$) & 72.0 (1.00$\times$) \\
SCD & 84.4$\pm$3.7 (1.19$\times$) & 82.5$\pm$0.7 (1.14$\times$) & 70.1$\pm$1.9 (0.96$\times$) & 65.8$\pm$1.0 (0.92$\times$) & 75.7 (1.05$\times$) \\
CoS & 92.8$\pm$0.6 (1.31$\times$) & 83.2$\pm$1.1 (1.15$\times$) & 71.6$\pm$1.8 (0.98$\times$) & 69.3$\pm$1.6 (0.97$\times$) & 79.2 (1.10$\times$) \\
DCD$_{\textsc{ngram}}$ & 109.3$\pm$0.6 (1.55$\times$) & 97.3$\pm$0.9 (1.35$\times$) & 96.0$\pm$1.4 (1.31$\times$) & 151.1$\pm$0.9 (2.11$\times$) & 113.4 (1.58$\times$) \\
\rowcolor{highlightcolor} \textbf{DCD$_{\textsc{eagle3}}$} & \textbf{145.5$\pm$5.8 (2.06$\times$)} & \textbf{116.8$\pm$6.0 (1.62$\times$)} & \textbf{117.9$\pm$5.1 (1.61$\times$)} & \textbf{115.7$\pm$6.4 (1.61$\times$)} & \textbf{124.0 (1.72$\times$)} \\
\bottomrule
\end{tabular}
\caption{Per-dataset greedy breakdown for $\alpha \in \{0.1, 0.5\}$. Each cell reports throughput in tokens/sec with reported std and speedup over vanilla CD at the same $\alpha$. The DCD$_{\textsc{ngram}}$ rows use the matched greedy N-gram proposer setting.}
\label{tab:exp5_appendix_results}
\vspace{-4pt}
\end{table*}

\begin{table*}[t]
\centering
\small
\renewcommand{\arraystretch}{1.15}
\setlength{\tabcolsep}{3pt}
\begin{tabular}{lccccc}
\toprule
\textbf{Method} & \textbf{HumanEval} & \textbf{GSM8K} & \textbf{MMLU} & \textbf{CNN/DM} & \textbf{Avg.} \\
\midrule
\multicolumn{6}{c}{\textit{Llama-3 (Llama-3.1-8B-Instruct / Llama-3.2-1B-Instruct)}} \\
\midrule
\multicolumn{6}{c}{$\alpha = 0.1$} \\
\cmidrule(lr){1-6}
CD & 92.4$\pm$3.3 (1.00$\times$) & 90.5$\pm$6.6 (1.00$\times$) & 90.3$\pm$3.3 (1.00$\times$) & 88.2$\pm$2.6 (1.00$\times$) & 90.3 (1.00$\times$) \\
SCD & 120.7$\pm$1.6 (1.31$\times$) & 125.3$\pm$0.9 (1.39$\times$) & 104.3$\pm$4.9 (1.16$\times$) & 107.8$\pm$3.2 (1.22$\times$) & 114.6 (1.27$\times$) \\
CoS & 156.3$\pm$5.9 (1.69$\times$) & 139.3$\pm$2.4 (1.54$\times$) & 109.2$\pm$2.3 (1.21$\times$) & 108.7$\pm$5.2 (1.23$\times$) & 128.4 (1.42$\times$) \\
DCD$_{\textsc{ngram}}$ & 128.9$\pm$2.0 (1.39$\times$) & 114.4$\pm$1.5 (1.26$\times$) & 108.4$\pm$0.9 (1.20$\times$) & 142.1$\pm$0.9 (1.61$\times$) & 123.4 (1.37$\times$) \\
\rowcolor{highlightcolor} \textbf{DCD$_{\textsc{eagle3}}$} & \textbf{161.3$\pm$8.9 (1.75$\times$)} & \textbf{139.5$\pm$0.6 (1.54$\times$)} & \textbf{119.6$\pm$5.2 (1.32$\times$)} & \textbf{140.9$\pm$1.7 (1.60$\times$)} & \textbf{140.3 (1.55$\times$)} \\
\midrule
\multicolumn{6}{c}{$\alpha = 0.5$} \\
\cmidrule(lr){1-6}
CD & 93.3$\pm$3.3 (1.00$\times$) & 90.8$\pm$5.0 (1.00$\times$) & 86.0$\pm$5.0 (1.00$\times$) & 93.1$\pm$2.0 (1.00$\times$) & 90.8 (1.00$\times$) \\
SCD & 126.6$\pm$2.0 (1.36$\times$) & 116.8$\pm$2.6 (1.29$\times$) & 91.8$\pm$0.8 (1.07$\times$) & 93.5$\pm$5.1 (1.00$\times$) & 107.2 (1.18$\times$) \\
CoS & 134.7$\pm$9.3 (1.44$\times$) & 112.4$\pm$0.4 (1.24$\times$) & 85.0$\pm$3.1 (0.99$\times$) & 96.8$\pm$5.4 (1.04$\times$) & 107.2 (1.18$\times$) \\
DCD$_{\textsc{ngram}}$ & 126.7$\pm$1.3 (1.36$\times$) & 114.5$\pm$1.3 (1.26$\times$) & 104.9$\pm$1.8 (1.22$\times$) & 147.4$\pm$1.3 (1.58$\times$) & 123.4 (1.36$\times$) \\
\rowcolor{highlightcolor} \textbf{DCD$_{\textsc{eagle3}}$} & \textbf{183.1$\pm$2.4 (1.96$\times$)} & \textbf{148.1$\pm$4.4 (1.63$\times$)} & \textbf{132.6$\pm$1.9 (1.54$\times$)} & \textbf{165.2$\pm$0.1 (1.77$\times$)} & \textbf{157.2 (1.73$\times$)} \\
\midrule
\multicolumn{6}{c}{\textit{Qwen3 (Qwen3-8B / Qwen3-1.7B)}} \\
\midrule
\multicolumn{6}{c}{$\alpha = 0.1$} \\
\cmidrule(lr){1-6}
CD & 65.1$\pm$3.4 (1.00$\times$) & 63.9$\pm$0.2 (1.00$\times$) & 64.2$\pm$2.2 (1.00$\times$) & 61.5$\pm$2.1 (1.00$\times$) & 63.7 (1.00$\times$) \\
SCD & 76.2$\pm$7.6 (1.17$\times$) & 85.3$\pm$0.7 (1.34$\times$) & 67.2$\pm$3.0 (1.05$\times$) & 63.4$\pm$2.8 (1.03$\times$) & 73.0 (1.15$\times$) \\
CoS & 84.8$\pm$1.7 (1.30$\times$) & 89.8$\pm$1.5 (1.41$\times$) & 73.2$\pm$0.9 (1.14$\times$) & 66.8$\pm$1.5 (1.09$\times$) & 78.7 (1.24$\times$) \\
DCD$_{\textsc{ngram}}$ & 85.0$\pm$0.9 (1.31$\times$) & 87.6$\pm$1.2 (1.37$\times$) & 77.1$\pm$0.6 (1.20$\times$) & 69.3$\pm$0.6 (1.13$\times$) & 79.7 (1.25$\times$) \\
\rowcolor{highlightcolor} \textbf{DCD$_{\textsc{eagle3}}$} & \textbf{133.6$\pm$1.4 (2.05$\times$)} & \textbf{125.2$\pm$5.6 (1.96$\times$)} & \textbf{106.9$\pm$5.4 (1.67$\times$)} & \textbf{103.8$\pm$2.4 (1.69$\times$)} & \textbf{117.4 (1.84$\times$)} \\
\midrule
\multicolumn{6}{c}{$\alpha = 0.5$} \\
\cmidrule(lr){1-6}
CD & 64.6$\pm$2.0 (1.00$\times$) & 65.9$\pm$0.3 (1.00$\times$) & 61.5$\pm$2.5 (1.00$\times$) & 60.4$\pm$0.3 (1.00$\times$) & 63.1 (1.00$\times$) \\
SCD & 69.0$\pm$0.8 (1.07$\times$) & 79.7$\pm$1.5 (1.21$\times$) & 57.7$\pm$3.7 (0.94$\times$) & 49.3$\pm$1.5 (0.82$\times$) & 63.9 (1.01$\times$) \\
CoS & 74.3$\pm$2.9 (1.15$\times$) & 84.4$\pm$1.2 (1.28$\times$) & 61.0$\pm$3.3 (0.99$\times$) & 48.1$\pm$2.2 (0.80$\times$) & 67.0 (1.06$\times$) \\
DCD$_{\textsc{ngram}}$ & 80.4$\pm$0.7 (1.25$\times$) & 85.5$\pm$0.8 (1.30$\times$) & 72.9$\pm$1.4 (1.19$\times$) & 64.8$\pm$0.2 (1.07$\times$) & 75.9 (1.20$\times$) \\
\rowcolor{highlightcolor} \textbf{DCD$_{\textsc{eagle3}}$} & \textbf{127.0$\pm$0.2 (1.97$\times$)} & \textbf{118.0$\pm$10.4 (1.79$\times$)} & \textbf{103.3$\pm$6.5 (1.68$\times$)} & \textbf{92.4$\pm$3.1 (1.53$\times$)} & \textbf{110.2 (1.75$\times$)} \\
\midrule
\multicolumn{6}{c}{\textit{Llama-EFT (Llama-3.1-8B-Instruct / Llama-3.1-8B)}} \\
\midrule
\multicolumn{6}{c}{$\alpha = 0.1$} \\
\cmidrule(lr){1-6}
CD & 70.4$\pm$2.3 (1.00$\times$) & 72.8$\pm$2.3 (1.00$\times$) & 67.7$\pm$1.0 (1.00$\times$) & 67.7$\pm$3.9 (1.00$\times$) & 69.7 (1.00$\times$) \\
SCD & 79.2$\pm$5.6 (1.12$\times$) & 77.4$\pm$2.4 (1.06$\times$) & 69.9$\pm$2.0 (1.03$\times$) & 61.8$\pm$0.9 (0.91$\times$) & 72.1 (1.03$\times$) \\
CoS & 84.8$\pm$1.9 (1.20$\times$) & 78.9$\pm$4.7 (1.08$\times$) & 71.4$\pm$5.1 (1.06$\times$) & 69.3$\pm$4.9 (1.02$\times$) & 76.1 (1.09$\times$) \\
DCD$_{\textsc{ngram}}$ & 105.3$\pm$1.0 (1.49$\times$) & 91.4$\pm$0.6 (1.25$\times$) & 86.1$\pm$1.2 (1.27$\times$) & 114.8$\pm$2.4 (1.70$\times$) & 99.4 (1.43$\times$) \\
\rowcolor{highlightcolor} \textbf{DCD$_{\textsc{eagle3}}$} & \textbf{137.5$\pm$7.2 (1.95$\times$)} & \textbf{115.9$\pm$0.5 (1.59$\times$)} & \textbf{99.9$\pm$5.8 (1.48$\times$)} & \textbf{112.2$\pm$3.1 (1.66$\times$)} & \textbf{116.4 (1.67$\times$)} \\
\midrule
\multicolumn{6}{c}{$\alpha = 0.5$} \\
\cmidrule(lr){1-6}
CD & 66.0$\pm$0.9 (1.00$\times$) & 75.0$\pm$0.6 (1.00$\times$) & 72.8$\pm$1.0 (1.00$\times$) & 68.4$\pm$4.4 (1.00$\times$) & 70.6 (1.00$\times$) \\
SCD & 76.4$\pm$0.8 (1.16$\times$) & 65.1$\pm$0.9 (0.87$\times$) & 57.3$\pm$0.6 (0.79$\times$) & 63.6$\pm$0.6 (0.93$\times$) & 65.6 (0.93$\times$) \\
CoS & 78.6$\pm$1.9 (1.19$\times$) & 73.6$\pm$1.9 (0.98$\times$) & 62.4$\pm$2.6 (0.86$\times$) & 63.7$\pm$2.3 (0.93$\times$) & 69.6 (0.99$\times$) \\
DCD$_{\textsc{ngram}}$ & 104.5$\pm$1.2 (1.58$\times$) & 92.8$\pm$1.1 (1.24$\times$) & 87.2$\pm$0.9 (1.20$\times$) & 120.9$\pm$2.2 (1.77$\times$) & 101.4 (1.44$\times$) \\
\rowcolor{highlightcolor} \textbf{DCD$_{\textsc{eagle3}}$} & \textbf{142.9$\pm$8.9 (2.17$\times$)} & \textbf{129.8$\pm$2.5 (1.73$\times$)} & \textbf{117.5$\pm$2.1 (1.61$\times$)} & \textbf{131.0$\pm$8.7 (1.91$\times$)} & \textbf{130.3 (1.85$\times$)} \\
\bottomrule
\end{tabular}
\caption{Per-dataset sampling breakdown for $T=1$ and $\alpha \in \{0.1, 0.5\}$. Each per-dataset cell reports throughput in tokens/sec as mean $\pm$ standard deviation across three independent runs, with speedup over vanilla CD at the same $\alpha$. The Avg. column reports the unweighted mean of the per-dataset throughputs; its parenthesized multiplier is the ratio of this mean to the corresponding vanilla-CD mean. The DCD$_{\textsc{ngram}}$ rows use the matched $T=1$ N-gram proposer setting.}
\label{tab:t1_exp5_appendix_results}
\vspace{-4pt}
\end{table*}

\subsection{Full Accuracy Sweep Details}
\label{sec:appendix_performance_details}
This subsection documents the full evaluation configuration underlying Section~\ref{sec:performance_accuracy}.

\paragraph{Setup.}
We evaluate the Llama-3.1-8B-Instruct and Llama-3.2-1B-Instruct pair on GSM8K and MMLU with 1000 examples per setting, draft length $\gamma{=}5$, and $\alpha \in \{0.0, 0.1, \ldots, 1.0\}$. We report DCD$_{\textsc{eagle3}}$, SCD, CoS, vanilla CD, and the expert-only autoregressive baseline under both temperatures, with the matched DCD$_{\textsc{ngram}}$ reference overlaid where available. Error bars denote 95\% confidence intervals.

\Cref{fig:performance_full} presents the complete accuracy sweep for both greedy decoding ($T{=}0$) and sampling ($T{=}1$). The matched DCD$_{\textsc{ngram}}$ reference curve is overlaid on both temperature settings, alongside the main EAGLE3-based DCD trajectory.

\input{output/accuracy_alpha_ablation_appendix_figure.tex}

\subsection{Additional Hyperparameter Ablations}
\label{sec:appendix_hparam_ablation}

Figure~\ref{fig:gamma_ablation_appendix} reports the complementary draft-length sweep referenced in the main text. Table~\ref{tab:gamma_tuning_summary} summarizes the same study in the fixed-$\gamma$ versus best-$\gamma$ format discussed there.
Table~\ref{tab:dcd_single_request_sweep_200} reports the full DCD$_{\textsc{eagle3}}$ throughput grid at $\alpha{=}0.1$. Each row corresponds to one evaluated combination of speculative steps, draft tokens, and top-$k$, and the last two columns report the resulting throughput on GSM8K and MMLU.

\begin{figure*}[t]
\centering
\definecolor{dcdblue}{RGB}{31,119,180}
\definecolor{scdorange}{RGB}{255,127,14}
\definecolor{cosgreen}{RGB}{44,160,44}

\pgfplotsset{
    every axis/.append style={
        label style={font=\scriptsize},
        tick label style={font=\scriptsize, inner sep=0.5pt},
        title style={font=\small, at={(0.5,1.03)}, anchor=south},
        title style={font=\scriptsize, align=center, at={(0.5,1.01)}, anchor=south},
        ylabel style={xshift=10pt, yshift=-0.08cm},
        legend style={font=\scriptsize},
        grid style={dashed, gray!40},
        scale only axis,
        trim axis left,
        trim axis right,
    }
}
\begin{subfigure}[b]{0.22\textwidth}
\centering
\begin{tikzpicture}
\begin{axis}[
    width=0.84\textwidth,
    height=2.0cm,
    xlabel={Draft Tokens ($\gamma$)},
    ylabel={},
    xmin=0.5, xmax=7.5,
    ymin=0.8, ymax=2.2,
    xtick={1,2,3,4,5,6,7},
    ytick={1.0,1.5,2.0},
    grid=major,
    title={\shortstack{\textsc{Llama-3}\\\textsc{MMLU}}}
]
\draw[black, thick, dashed] (axis cs:0.5,1) -- (axis cs:7.5,1);
\addplot[color=dcdblue, mark=*, mark size=1.5pt, thick] coordinates {
    (1, 1.19) (2, 1.37) (3, 1.41) (4, 1.50) (5, 1.57) (6, 1.54) (7, 1.50)
};
\addplot[color=scdorange, mark=square*, mark size=1.5pt, thick] coordinates {
    (1, 1.20) (2, 1.27) (3, 1.32) (4, 1.27) (5, 1.13) (6, 1.11) (7, 1.09)
};
\addplot[color=cosgreen, mark=triangle*, mark size=1.5pt, thick] coordinates {
    (1, 1.36) (2, 1.40) (3, 1.29) (4, 1.30) (5, 1.24) (6, 1.15) (7, 1.08)
};
\end{axis}
\end{tikzpicture}
\end{subfigure}
\hfill
\begin{subfigure}[b]{0.22\textwidth}
\centering
\begin{tikzpicture}
\begin{axis}[
    width=0.84\textwidth,
    height=2.0cm,
    xlabel={Draft Tokens ($\gamma$)},
    ylabel={},
    xmin=0.5, xmax=7.5,
    ymin=0.8, ymax=2.2,
    xtick={1,2,3,4,5,6,7},
    ytick={1.0,1.5,2.0},
    yticklabels={,,},
    grid=major,
    title={\shortstack{\textsc{Qwen3}\\\textsc{MMLU}}}
]
\draw[black, thick, dashed] (axis cs:0.5,1) -- (axis cs:7.5,1);
\addplot[color=dcdblue, mark=*, mark size=1.5pt, thick] coordinates {
    (1, 1.32) (2, 1.60) (3, 1.66) (4, 1.59) (5, 1.64) (6, 1.55) (7, 1.53)
};
\addplot[color=scdorange, mark=square*, mark size=1.5pt, thick] coordinates {
    (1, 1.14) (2, 1.18) (3, 1.14) (4, 1.16) (5, 1.08) (6, 0.96) (7, 0.93)
};
\addplot[color=cosgreen, mark=triangle*, mark size=1.5pt, thick] coordinates {
    (1, 1.44) (2, 1.30) (3, 1.22) (4, 1.11) (5, 1.06) (6, 0.97) (7, 0.91)
};
\end{axis}
\end{tikzpicture}
\end{subfigure}
\hfill
\begin{subfigure}[b]{0.22\textwidth}
\centering
\begin{tikzpicture}
\begin{axis}[
    width=0.84\textwidth,
    height=2.0cm,
    xlabel={Draft Tokens ($\gamma$)},
    ylabel={},
    xmin=0.5, xmax=7.5,
    ymin=0.8, ymax=2.2,
    xtick={1,2,3,4,5,6,7},
    ytick={1.0,1.5,2.0},
    yticklabels={,,},
    grid=major,
    title={\shortstack{\textsc{Llama-3}\\\textsc{GSM8K}}}
]
\draw[black, thick, dashed] (axis cs:0.5,1) -- (axis cs:7.5,1);
\addplot[color=dcdblue, mark=*, mark size=1.5pt, thick] coordinates {
    (1, 1.31) (2, 1.54) (3, 1.60) (4, 1.64) (5, 1.67) (6, 1.64) (7, 1.55)
};
\addplot[color=scdorange, mark=square*, mark size=1.5pt, thick] coordinates {
    (1, 1.17) (2, 1.49) (3, 1.59) (4, 1.61) (5, 1.57) (6, 1.55) (7, 1.52)
};
\addplot[color=cosgreen, mark=triangle*, mark size=1.5pt, thick] coordinates {
    (1, 1.48) (2, 1.55) (3, 1.66) (4, 1.59) (5, 1.61) (6, 1.52) (7, 1.56)
};
\end{axis}
\end{tikzpicture}
\end{subfigure}
\hfill
\begin{subfigure}[b]{0.22\textwidth}
\centering
\begin{tikzpicture}
\begin{axis}[
    width=0.84\textwidth,
    height=2.0cm,
    xlabel={Draft Tokens ($\gamma$)},
    ylabel={},
    xmin=0.5, xmax=7.5,
    ymin=0.8, ymax=2.2,
    xtick={1,2,3,4,5,6,7},
    ytick={1.0,1.5,2.0},
    yticklabels={,,},
    grid=major,
    title={\shortstack{\textsc{Qwen3}\\\textsc{GSM8K}}}
]
\draw[black, thick, dashed] (axis cs:0.5,1) -- (axis cs:7.5,1);
\addplot[color=dcdblue, mark=*, mark size=1.5pt, thick] coordinates {
    (1, 1.46) (2, 1.81) (3, 1.92) (4, 2.04) (5, 1.99) (6, 2.04) (7, 2.05)
};
\addplot[color=scdorange, mark=square*, mark size=1.5pt, thick] coordinates {
    (1, 1.24) (2, 1.36) (3, 1.37) (4, 1.37) (5, 1.37) (6, 1.26) (7, 1.28)
};
\addplot[color=cosgreen, mark=triangle*, mark size=1.5pt, thick] coordinates {
    (1, 1.48) (2, 1.53) (3, 1.49) (4, 1.49) (5, 1.43) (6, 1.37) (7, 1.31)
};
\end{axis}
\end{tikzpicture}
\end{subfigure}

\vspace{-0.5em}
\begin{center}
\begin{tikzpicture}
\begin{axis}[
    hide axis,
    xmin=0, xmax=1,
    ymin=0, ymax=1,
    legend style={
        draw=none,
        fill=none,
        legend columns=-1,
        /tikz/every even column/.append style={column sep=0.5cm},
        font=\small
    },
    legend pos=north west
]
\addlegendimage{color=dcdblue, mark=*, mark size=2pt, thick}
\addlegendentry{DCD$_{\textsc{eagle3}}$}
\addlegendimage{color=scdorange, mark=square*, mark size=2pt, thick}
\addlegendentry{SCD}
\addlegendimage{color=cosgreen, mark=triangle*, mark size=2pt, thick}
\addlegendentry{CoS}
\end{axis}
\end{tikzpicture}
\end{center}
\vspace{-1.2em}
    \caption{Appendix-only $\gamma$ ablation. The y-axis reports speedup relative to vanilla CD. DCD$_{\textsc{eagle3}}$ stays efficient across a broader draft-length range, while the amateur-coupled baselines usually peak earlier as larger draft lengths make their heavier proposal paths less worthwhile.}
\label{fig:gamma_ablation_appendix}
\vspace{-8pt}
\end{figure*}
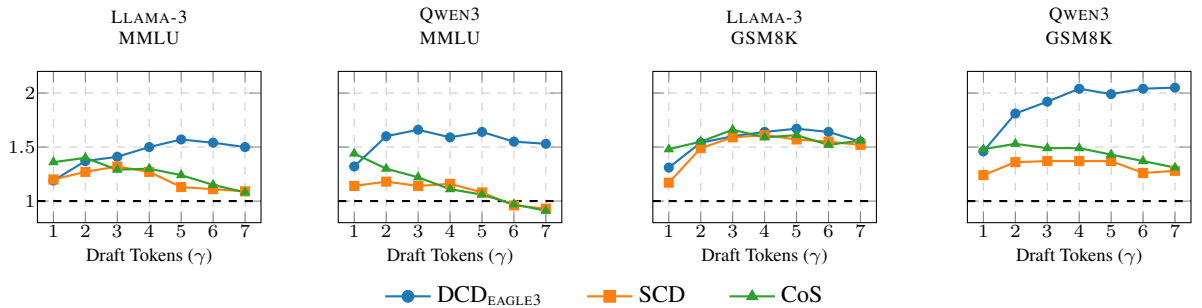
\begin{table}[t]
\centering
\scriptsize
\renewcommand{\arraystretch}{1.08}
\setlength{\tabcolsep}{2.5pt}
\begin{tabular*}{\columnwidth}{@{\extracolsep{\fill}}llccc}
\toprule
\textbf{Model} & \textbf{Method} & \textbf{Fixed $\gamma{=}5$} & \textbf{Best $\gamma$} & \textbf{Best speedup} \\
\midrule
\textsc{Llama-3} & SCD & 1.13$\times$ & 3 & 1.32$\times$ \\
\textsc{Llama-3} & CoS & 1.24$\times$ & 2 & 1.40$\times$ \\
\rowcolor{highlightcolor} \textsc{Llama-3} & \textbf{DCD$_{\textsc{eagle3}}$} & \textbf{1.57$\times$} & \textbf{5} & \textbf{1.57$\times$} \\
\midrule
\textsc{Qwen3} & SCD & 1.08$\times$ & 2 & 1.18$\times$ \\
\textsc{Qwen3} & CoS & 1.06$\times$ & 1 & 1.44$\times$ \\
\rowcolor{highlightcolor} \textsc{Qwen3} & \textbf{DCD$_{\textsc{eagle3}}$} & \textbf{1.64$\times$} & \textbf{3} & \textbf{1.66$\times$} \\
\midrule
\textsc{Llama-EFT} & SCD & 1.01$\times$ & 3 & 1.18$\times$ \\
\textsc{Llama-EFT} & CoS & 1.09$\times$ & 1 & 1.50$\times$ \\
\rowcolor{highlightcolor} \textsc{Llama-EFT} & \textbf{DCD$_{\textsc{eagle3}}$} & \textbf{1.60$\times$} & \textbf{6} & \textbf{1.70$\times$} \\
\bottomrule
\end{tabular*}
\caption{MMLU draft-length tuning check at $T{=}0$, $\alpha{=}0.1$. Speedups are relative to vanilla CD for the same model pair. Fixed $\gamma{=}5$ is the main operating point; Best $\gamma$ selects the highest-throughput setting from $\gamma \in \{1,\ldots,7\}$ using the existing chain-drafting implementation.}
\label{tab:gamma_tuning_summary}
\vspace{-4pt}
\end{table}

\begin{table}[t]
\centering
\scriptsize
\sisetup{mode=text, detect-weight=true, detect-family=true}
\renewcommand{\arraystretch}{1.08}
\setlength{\tabcolsep}{3.6pt}
\begin{tabular*}{\columnwidth}{@{\extracolsep{\fill}}ccS[table-format=3.2]S[table-format=3.2]}
\toprule
\multicolumn{2}{c}{\textbf{Configuration}} & \multicolumn{2}{c}{\textbf{Throughput (tok/s)}} \\
\cmidrule(lr){1-2}\cmidrule(lr){3-4}
\textbf{Draft Tokens} & \textbf{$k$} & \multicolumn{1}{c}{\textbf{GSM8K}} & \multicolumn{1}{c}{\textbf{MMLU}} \\
\midrule
\rowcolor{black!4}
\multicolumn{4}{l}{\textbf{Steps = 3}} \\
4 & 1 & 157.06 & 136.84 \\
6 & 3 & 163.94 & 145.23 \\
6 & 5 & 160.91 & 149.53 \\
10 & 3 & 170.73 & 154.89 \\
10 & 5 & 176.24 & 160.24 \\
20 & 3 & 178.13 & 163.77 \\
\rowcolor{highlightcolor} \textbf{20} & \textbf{5} & \bfseries 181.14 & \bfseries 168.51 \\
\addlinespace[2pt]
\rowcolor{black!4}
\multicolumn{4}{l}{\textbf{Steps = 5}} \\
6 & 1 & 158.44 & 134.78 \\
6 & 3 & 160.45 & 142.24 \\
6 & 5 & 158.72 & 147.62 \\
10 & 3 & 178.15 & 156.10 \\
10 & 5 & 169.37 & 157.24 \\
20 & 3 & 188.02 & 165.43 \\
\rowcolor{highlightcolor} \textbf{20} & \textbf{5} & \bfseries 191.84 & \bfseries 170.47 \\
\addlinespace[2pt]
\rowcolor{black!4}
\multicolumn{4}{l}{\textbf{Steps = 8}} \\
9 & 1 & 144.19 & 126.62 \\
9 & 3 & 152.58 & 137.00 \\
9 & 5 & 154.11 & 138.35 \\
10 & 3 & 156.26 & 139.18 \\
10 & 5 & 157.35 & 139.96 \\
20 & 3 & 169.32 & 148.55 \\
\rowcolor{highlightcolor} \textbf{20} & \textbf{5} & \bfseries 173.45 & \bfseries 155.76 \\
\bottomrule
\end{tabular*}
\caption{DCD$_{\textsc{eagle3}}$ throughput grid at $T{=}0$ and $\alpha{=}0.1$. Blocks fix speculative steps; rows vary draft tokens and top-$k$. Bold marks the block-wise best throughput for each dataset, and shading denotes the configuration that is best on both GSM8K and MMLU within a block.}
\label{tab:dcd_single_request_sweep_200}
\vspace{-4pt}
\end{table}

\section{Extended System Evaluation}
\label{sec:appendix_extended_system}
\subsection{Multi-Request Speedups on GSM8K}
\label{sec:appendix_multi_request}

Table~\ref{tab:multi_request_1000} reports multi-request serving speedups on GSM8K. The DCD$_{\textsc{eagle3}}$, SCD, and matched DCD$_{\textsc{ngram}}$ rows use the same serving configuration: five speculative steps with chain-style greedy drafting and six draft tokens per speculative round. Each entry reports throughput normalized by vanilla CD at the same request count.

\begin{table*}[t]
\centering
\small
\renewcommand{\arraystretch}{1.08}
\setlength{\tabcolsep}{6pt}
\begin{tabular*}{0.82\textwidth}{@{\extracolsep{\fill}}lcccccccc}
\toprule
\textbf{Method} & 2 & 4 & 8 & 16 & 24 & 32 & 48 & 56 \\
\midrule
SCD & 1.29 & 1.33 & 1.32 & 1.28 & 1.16 & 1.44 & 1.21 & 0.87 \\
DCD$_{\textsc{ngram}}$ & 1.37 & 1.38 & 1.19 & 0.91 & 1.04 & 0.99 & 1.11 & 0.72 \\
DCD$_{\textsc{eagle3}}$ & 1.56 & 1.51 & 1.29 & 1.48 & 1.36 & 1.52 & 1.32 & 1.37 \\
\bottomrule
\end{tabular*}
\caption{Configuration-matched multi-request serving speedups on the 1000-question GSM8K sweep. Columns denote the number of concurrent requests at 2, 4, 8, 16, 24, 32, 48, and 56. DCD$_{\textsc{eagle3}}$, SCD, and DCD$_{\textsc{ngram}}$ all use 5 speculative steps with chain-style greedy drafting and 6 draft tokens per speculative round. Each entry reports the throughput of the corresponding method divided by the throughput of vanilla CD at the same request count; therefore, the vanilla CD baseline is 1.0 in every column.}
\label{tab:multi_request_1000}
\vspace{-4pt}
\end{table*}

\subsection{Runtime KV-Cache Occupancy}
\label{sec:appendix_memory_kv}

We examine runtime KV-cache occupancy in the representative \textsc{Llama-3} 8B setting from the main paper, where \textsc{Llama-3.1-8B-Instruct} is the expert, \textsc{Llama-3.2-1B-Instruct} is the amateur, and DCD$_{\textsc{eagle3}}$ uses the same EAGLE3 drafter. AR serves only the expert, vanilla CD and SCD keep both the expert and amateur active, and DCD$_{\textsc{eagle3}}$ changes only the proposal path through a lightweight feature-level drafter. In both the single-request and shared-batching settings, vanilla CD, SCD, and DCD$_{\textsc{eagle3}}$ show nearly identical steady-state KV-cache occupancy because the persistent verification-time KV state remains dominated by the expert and amateur decoders.

\subsection{Spec-Bench Evaluation}
\label{sec:appendix_specbench}

To test whether the same speed pattern extends beyond the four main datasets, we also report results on Spec-Bench~\citep{xia2024unlocking}, a standardized benchmark for speculative decoding with six subtasks: multi-turn conversation (MT-Bench), translation, summarization, question answering, mathematical reasoning, and retrieval-augmented generation. Each subtask contains 80 prompts, giving 480 prompts in total. Table~\ref{tab:specbench} reports per-subtask throughput and average speedup under this setting.

\begin{table*}[t]
\centering
\small
\renewcommand{\arraystretch}{1.15}
\setlength{\tabcolsep}{4pt}
\begin{tabular}{lccccccc}
\toprule
\textbf{Method} & \textbf{MT-Bench} & \textbf{Translation} & \textbf{Summ.} & \textbf{QA} & \textbf{Math} & \textbf{RAG} & \textbf{Avg. (Speedup)} \\
\midrule
\multicolumn{8}{c}{\textit{\textbf{Llama-3} (8B-Ins / 1B-Ins)}} \\
\midrule
CD & 95.1 & 92.7 & 95.1 & 94.2 & 95.9 & 88.2 & 93.8 (1.00$\times$) \\
SCD & 132.6 & 105.7 & 114.7 & 122.6 & 151.8 & 119.3 & 125.6 (1.34$\times$) \\
CoS & 138.0 & 107.2 & 116.9 & \textbf{125.6} & 159.0 & 123.0 & 129.7 (1.38$\times$) \\
DCD$_{\textsc{ngram}}$ & 118.3 & 104.5 & 92.9 & 107.4 & 123.8 & 96.8 & 108.9 (1.16$\times$) \\
\rowcolor{yellow!8} \textbf{DCD$_{\textsc{eagle3}}$} & \textbf{156.3} & \textbf{109.4} & \textbf{146.4} & 121.7 & \textbf{166.3} & \textbf{142.3} & \textbf{142.6} (1.52$\times$) \\
\midrule
\multicolumn{8}{c}{\textit{\textbf{Qwen3} (8B / 1.7B)}} \\
\midrule
CD & 65.9 & 62.4 & 65.8 & 66.8 & 66.9 & 65.7 & 65.6 (1.00$\times$) \\
SCD & 76.2 & 75.5 & 66.2 & 66.0 & 97.2 & 69.8 & 75.3 (1.15$\times$) \\
CoS & 77.9 & 76.6 & 66.1 & 65.7 & 101.8 & 70.0 & 76.6 (1.17$\times$) \\
DCD$_{\textsc{ngram}}$ & 77.6 & 58.3 & 63.1 & 71.7 & 92.6 & 72.5 & 73.4 (1.12$\times$) \\
\rowcolor{yellow!8} \textbf{DCD$_{\textsc{eagle3}}$} & \textbf{128.0} & \textbf{99.9} & \textbf{110.3} & \textbf{114.6} & \textbf{152.1} & \textbf{124.7} & \textbf{122.5} (1.87$\times$) \\
\bottomrule
\end{tabular}
\caption{Spec-Bench evaluation results (tokens/sec). Spec-Bench~\citep{xia2024unlocking} is a standardized benchmark for speculative decoding with six sub-tasks. All methods use $\alpha{=}0.1$, $\gamma{=}5$, and greedy decoding. Speedup is relative to vanilla CD. The matched DCD$_{\textsc{ngram}}$ reference uses the same setting as an appendix row, while DCD$_{\textsc{eagle3}}$ has the highest average throughput for both model families and leads on most sub-tasks, but not every individual column.}
\label{tab:specbench}
\vspace{-4pt}
\end{table*}

\subsection{70B-Class Results}
\label{sec:appendix_large_model}

We report a greedy-only 70B extension with \texttt{Llama-3.3-70B-Instruct} as the expert and \texttt{Llama-3.1-8B-Instruct} as the amateur. The sweep uses the same four datasets and $\alpha \in \{0.0,0.1,0.5\}$; these results are single-run diagnostics for checking whether the main trend extends to a larger expert.

Table~\ref{tab:large_model_results} reports the per-dataset throughput breakdown for the EAGLE3 and matched N-gram proposer settings. Table~\ref{tab:large_model_alpha_robustness} reports the corresponding accepted-length robustness, normalized to each method at $\alpha=0.0$.

\begin{table*}[t]
\centering
\small
\renewcommand{\arraystretch}{1.15}
\setlength{\tabcolsep}{4pt}
\begin{tabular}{lc|ccccc}
\toprule
\textbf{Method} & \textbf{$\alpha$} & \textbf{HumanEval} & \textbf{GSM8K} & \textbf{MMLU} & \textbf{CNN/DM} & \textbf{Avg.} \\
\midrule
\multicolumn{7}{c}{\textit{Llama-3.3-70B-Instruct / Llama-3.1-8B-Instruct, greedy decoding}} \\
\midrule
\multicolumn{7}{c}{$\alpha = 0.0$} \\
\cmidrule(lr){1-7}
CD & 0.0 & 30.4 (1.00$\times$) & 30.2 (1.00$\times$) & 30.2 (1.00$\times$) & 29.8 (1.00$\times$) & 30.2 (1.00$\times$) \\
SCD & 0.0 & 61.6 (2.02$\times$) & 59.3 (1.96$\times$) & 49.5 (1.64$\times$) & 50.8 (1.70$\times$) & 55.3 (1.83$\times$) \\
CoS & 0.0 & 66.0 (2.17$\times$) & 62.6 (2.07$\times$) & \textbf{50.3 (1.67$\times$)} & 53.0 (1.78$\times$) & 58.0 (1.92$\times$) \\
DCD$_{\textsc{ngram}}$ & 0.0 & 46.8 (1.54$\times$) & 42.0 (1.39$\times$) & 38.0 (1.26$\times$) & 50.2 (1.69$\times$) & 44.3 (1.47$\times$) \\
\rowcolor{highlightcolor} \textbf{DCD$_{\textsc{eagle3}}$} & 0.0 & \textbf{79.6 (2.62$\times$)} & \textbf{62.7 (2.07$\times$)} & 47.4 (1.57$\times$) & \textbf{55.5 (1.86$\times$)} & \textbf{61.3 (2.03$\times$)} \\
\midrule
\multicolumn{7}{c}{$\alpha = 0.1$} \\
\cmidrule(lr){1-7}
CD & 0.1 & 30.2 (1.00$\times$) & 30.1 (1.00$\times$) & 30.3 (1.00$\times$) & 29.8 (1.00$\times$) & 30.1 (1.00$\times$) \\
SCD & 0.1 & 59.3 (1.96$\times$) & 58.2 (1.93$\times$) & 48.7 (1.61$\times$) & 49.8 (1.67$\times$) & 54.0 (1.79$\times$) \\
CoS & 0.1 & 62.3 (2.06$\times$) & \textbf{62.2 (2.07$\times$)} & \textbf{50.4 (1.66$\times$)} & 52.2 (1.75$\times$) & 56.8 (1.89$\times$) \\
DCD$_{\textsc{ngram}}$ & 0.1 & 46.5 (1.54$\times$) & 40.5 (1.35$\times$) & 38.0 (1.26$\times$) & 50.3 (1.69$\times$) & 43.8 (1.46$\times$) \\
\rowcolor{highlightcolor} \textbf{DCD$_{\textsc{eagle3}}$} & 0.1 & \textbf{78.1 (2.58$\times$)} & 62.0 (2.06$\times$) & 47.5 (1.57$\times$) & \textbf{55.3 (1.85$\times$)} & \textbf{60.7 (2.02$\times$)} \\
\midrule
\multicolumn{7}{c}{$\alpha = 0.5$} \\
\cmidrule(lr){1-7}
CD & 0.5 & 30.2 (1.00$\times$) & 30.3 (1.00$\times$) & 30.5 (1.00$\times$) & 29.8 (1.00$\times$) & 30.2 (1.00$\times$) \\
SCD & 0.5 & 52.2 (1.73$\times$) & 57.2 (1.89$\times$) & 45.6 (1.50$\times$) & 49.0 (1.64$\times$) & 51.0 (1.69$\times$) \\
CoS & 0.5 & 54.2 (1.79$\times$) & 60.0 (1.98$\times$) & \textbf{47.4 (1.56$\times$)} & 50.0 (1.68$\times$) & 52.9 (1.75$\times$) \\
DCD$_{\textsc{ngram}}$ & 0.5 & 46.1 (1.52$\times$) & 41.2 (1.36$\times$) & 37.3 (1.22$\times$) & 49.5 (1.66$\times$) & 43.5 (1.44$\times$) \\
\rowcolor{highlightcolor} \textbf{DCD$_{\textsc{eagle3}}$} & 0.5 & \textbf{73.6 (2.44$\times$)} & \textbf{60.8 (2.01$\times$)} & 46.6 (1.53$\times$) & \textbf{55.1 (1.85$\times$)} & \textbf{59.1 (1.96$\times$)} \\
\bottomrule
\end{tabular}
\caption{70B-class extension under greedy decoding. Values report throughput (tokens/sec); speedup in parentheses is relative to vanilla CD at the same $\alpha$. All results are single-run diagnostics using the Llama-3.3-70B-Instruct expert, the Llama-3.1-8B-Instruct amateur, and either the corresponding EAGLE3 70B drafter or the matched N-gram proposer setting.}
\label{tab:large_model_results}
\vspace{-4pt}
\end{table*}

\begin{table}[t]
\centering
\renewcommand{\arraystretch}{1.15}
\setlength{\tabcolsep}{4.5pt}
\small
\resizebox{\columnwidth}{!}{%
\begin{tabular}{lccc}
\toprule
\textbf{Method} & \textbf{Rel.@$0.1$} & \textbf{Rel.@$0.5$} & \textbf{Drop to $0.5$} \\
\midrule
SCD & 97.1\% & 89.3\% & -10.7\% \\
CoS & 97.1\% & 89.3\% & -10.7\% \\
DCD$_{\textsc{ngram}}$ & 99.2\% & 97.0\% & -3.0\% \\
\rowcolor{highlightcolor} \textbf{DCD$_{\textsc{eagle3}}$} & \textbf{98.5\%} & \textbf{93.7\%} & \textbf{-6.3\%} \\
\midrule
\multicolumn{4}{l}{\footnotesize Extra drop vs. DCD$_{\textsc{eagle3}}$: SCD = +4.4 pts, CoS = +4.4 pts.} \\
\bottomrule
\end{tabular}
}
\caption{Relative mean accepted length as contrastive strength increases in the 70B-class extension. Each ratio is normalized to the same method at $\alpha{=}0.0$ and averaged across the four datasets under greedy decoding, including the matched N-gram proposer setting. Drop is the relative decrease from $\alpha{=}0.0$ to $\alpha{=}0.5$.}
\label{tab:large_model_alpha_robustness}
\vspace{-4pt}
\end{table}

\section{Model Details}
\label{sec:appendix_models}
\Cref{tab:model_details} lists the off-the-shelf models and EAGLE3 proposer checkpoints used in the deployment experiments. The controlled proposal-side diagnostics that train dual-input or auxiliary EAGLE heads are described separately in Section~\ref{sec:why_not_cd} and Appendix~\ref{sec:appendix_cross_alpha}.

\begin{table}[htbp]
\centering
\small
\begin{threeparttable}
\begin{tabularx}{\columnwidth}{ll>{\raggedright\arraybackslash}X}
\toprule
\textbf{Configuration} & \textbf{Component} & \textbf{Model} \\
\midrule
\multirow{3}{*}{\textsc{Llama-3}} 
    & Expert & Llama-3.1-8B-Instruct \\
    \cmidrule(lr){2-3}
    & Amateur & Llama-3.2-1B-Instruct \\
    \cmidrule(lr){2-3}
    & Drafter & EAGLE3-LLaMA3.1-Instruct-8B\tnote{a} \\
\midrule
\multirow{3}{*}{\textsc{Qwen3}} 
    & Expert & Qwen3-8B \\
    \cmidrule(lr){2-3}
    & Amateur & Qwen3-1.7B \\
    \cmidrule(lr){2-3}
    & Drafter & Qwen3-8B-EAGLE3\tnote{b} \\
\midrule
\multirow{3}{*}{\textsc{Llama-EFT}} 
    & Expert & Llama-3.1-8B-Instruct \\
    \cmidrule(lr){2-3}
    & Amateur & Llama-3.1-8B \\
    \cmidrule(lr){2-3}
    & Drafter & EAGLE3-LLaMA3.1-Instruct-8B\tnote{a} \\
\midrule
\multirow{3}{*}{\makecell[l]{Llama-3\\70B Extension}}
    & Expert & Llama-3.3-70B-Instruct \\
    \cmidrule(lr){2-3}
    & Amateur & Llama-3.1-8B-Instruct \\
    \cmidrule(lr){2-3}
    & Drafter & EAGLE3-LLaMA3.3-Instruct-70B\tnote{c} \\
\bottomrule
\end{tabularx}
\begin{tablenotes}
\footnotesize
\item[a] \texttt{yuhuili/EAGLE3-LLaMA3.1-Instruct-8B}
\item[b] \texttt{Tengyunw/qwen3\_8b\_eagle3}
\item[c] \texttt{yuhuili/EAGLE3-LLaMA3.3-Instruct-70B}
\end{tablenotes}
\end{threeparttable}
\caption{Off-the-shelf model configurations used in the deployment experiments.}
\label{tab:model_details}
\end{table}

\section{Implementation Details}
\label{sec:implementation_details}

\paragraph{Experimental Environment.}
Unless noted otherwise, benchmarks run on a single NVIDIA H200 141GB GPU under the SGLang~\citep{zheng2024sglang} inference framework.

\paragraph{Artifacts and Intended Use.}
We use publicly released benchmark datasets, model checkpoints, EAGLE3 proposer checkpoints, and the SGLang inference framework under their respective licenses and access terms. These artifacts are used only for research benchmarking of decoding efficiency and task-level consistency, and we do not redistribute third-party datasets or model weights.

\paragraph{Inference Configuration.}
To measure latency cleanly, the main single-request deployment and diagnostic experiments use single-card serial execution. Table~\ref{tab:inference_config} summarizes the key single-request runtime parameters. Runtime KV-cache occupancy under this setup is examined in Appendix~\ref{sec:appendix_memory_kv}; multi-request serving uses the dedicated concurrent-request settings in Appendix~\ref{sec:appendix_multi_request}.
For EAGLE3, \texttt{speculative-eagle-topk}$=1$ implements chain-style greedy drafting.
For DCD$_{\textsc{ngram}}$, the proposer uses a draft-token budget of 5, an N-gram match-window range from 1 to 12, BFS breadth 1 for chain-style proposal selection, and branch length 18.

\begin{table}[ht]
\centering
\small
\begin{tabularx}{0.92\linewidth}{@{}>{\raggedright\arraybackslash}p{0.5\linewidth}>{\raggedright\arraybackslash}X@{}}
\toprule
\textbf{Parameter} & \textbf{Value} \\
\midrule
\makecell[tl]{\texttt{max\_running\_}\\\texttt{requests}} & 1 \\
\midrule
\makecell[tl]{\texttt{disable\_cuda\_}\\\texttt{graph}} & True \\
\midrule
\texttt{dtype} & bfloat16 \\
\midrule
\makecell[tl]{\texttt{context\_}\\\texttt{length}} & 4096 \\
\midrule
\makecell[tl]{\scriptsize\texttt{speculative-}\\\scriptsize\texttt{num-steps}} & 5 \\
\midrule
\makecell[tl]{\scriptsize\texttt{speculative-}\\\scriptsize\texttt{eagle-topk}} & 1 \\
\bottomrule
\end{tabularx}
\caption{Main single-request inference configuration.}
\label{tab:inference_config}
\end{table}

\paragraph{Dataset and Prompt Configuration.}
For the main deployment and task-accuracy evaluations, we use the official chat templates for all models and evaluate two settings:
\begin{itemize}
    \item \textbf{Main Experiments (Speed Analysis):} We use 200 samples per dataset with a maximum generation length of 256 tokens.
    \item \textbf{Performance Tests (Accuracy Verification):} We use 1000 samples with a maximum generation length of 1024 tokens and task-specific zero-shot templates with structured output formats, for example, ``Final Answer:'', to support automated evaluation.
\end{itemize}

\end{document}